\documentclass{article}

\usepackage{style/iclr2027/iclr2027_conference,times}
\usepackage{amsmath,amssymb,amsthm,mathtools,bm}
\usepackage{booktabs,array,multirow}
\usepackage{graphicx}
\usepackage{placeins}
\usepackage{float}
\usepackage{xcolor}
\usepackage{microtype}
\usepackage{hyperref}
\usepackage{url}

\def\pdfauthorname{Xiaoyang Li; Runni Zhou; Xinghao Yan}
\newcommand{\publicauthorblock}{Xiaoyang Li\thanks{Equal contribution. Xiaoyang Li and Runni Zhou are co-first authors.}\quad Runni Zhou\footnotemark[1]\quad \& Xinghao Yan\thanks{ORCID: Xiaoyang Li, \href{https://orcid.org/0009-0002-4863-5761}{0009-0002-4863-5761}; Runni Zhou, \href{https://orcid.org/0009-0006-8365-4852}{0009-0006-8365-4852}; Xinghao Yan, \href{https://orcid.org/0009-0002-4437-957X}{0009-0002-4437-957X}.}\\
College of Medicine and Biological Information Engineering\\
Northeastern University, Shenyang, China\\
\texttt{\{20246389,20246378\}@stu.neu.edu.cn}\\
\texttt{yanxh@mails.neu.edu.cn}}

\definecolor{linkblue}{HTML}{174A7E}
\hypersetup{colorlinks=true,linkcolor=linkblue,citecolor=linkblue,urlcolor=linkblue,
  pdftitle={Branch Geometry and Finite-Radius Sensitivity of Hard-ReLU Training},
  pdfauthor={\pdfauthorname},pdfsubject={Finite-radius objective responses and discrete branch geometry}}

\newtheorem{theorem}{Theorem}
\newtheorem{proposition}[theorem]{Proposition}
\newtheorem{lemma}[theorem]{Lemma}
\newtheorem{corollary}[theorem]{Corollary}

\theoremstyle{definition}

\newtheorem{remark}[theorem]{Remark}

\newcommand{\AD}{\mathrm{AD}}
\newcommand{\SC}{\mathrm{SC}}
\newcommand{\reg}{\mathrm{reg}}

\newcommand{\tr}{\operatorname{tr}}
\newcommand{\dd}{\mathrm{d}}

\title{Branch Geometry and Finite-Radius\\
Sensitivity of Hard-ReLU Training}
\author{\publicauthorblock}
\iclrfinalcopy

\begin{document}
\maketitle
\fancyhead{}
\lhead{Preprint}

\begin{abstract}
Two nearby initializations can activate the same ReLU unit one gradient
descent (GD) step apart. When their separation is proportional to the GD
step size $\eta$, that one step can determine which initialization gives
the lower outer loss. We derive the integer crossing steps for
piecewise-$C^2$ dynamics with finitely many separated transverse events.
Euler error shifts the crossing clock, and rounding an earlier
crossing to a GD step can shift later crossings. With a common
positive margin from the rounding boundaries, the expansion
predicts the trained state with uniform $O(\eta^2)$ error.
In a strongly convex ReLU risk, the branch prediction converges on
three fixed directions. The best single-vector predictor has positive
limiting error. The recursion
also gives a sufficient directional radius for preserving the training
masks. In contractive affine regions, verifying those masks and
bounding the finite remainder can certify candidate ordering. Scalar
and coupled examples test the two rounding effects. Small-network
experiments measure response accuracy, candidate selection, and how
often multiple entries of the activation mask change in one GD update.

\end{abstract}

\section{Introduction}
Suppose we can start a ReLU network from either of two nearby
initializations, run the same number of gradient descent (GD) steps,
and choose the initialization with lower validation loss. Can a
derivative of the training run predict that choice? The question also
arises when selecting a finite change to training weights or a continuous
hyperparameter
\citep{maclaurin2015reversible,franceschi2018bilevel,finn2017maml,ren2018reweight}.
We compare the losses at the two actual proposals, with the training
horizon fixed and the proposal radius specified.

In the scalar example in Section~\ref{sec:setup}, a ReLU unit activates
at step 72 from the base initialization. Perturbing the input in one
direction delays activation until step 73, while the opposite perturbation
leaves it at step 72. The delayed proposal gives lower loss. The derivative of the finite run
and the derivative of its continuous-time model both favor the other
proposal. When each preactivation stays away from zero at the evaluated
steps, automatic differentiation (AD) gives the exact local derivative
of the trained state. It describes a neighborhood where every training
example follows the same sequence of active units.

A continuous training path crosses a ReLU boundary at a real-valued time.
GD with step size $\eta$ changes masks only at integer steps. A one-step delay
uses the incoming update field once more. The resulting difference
is of order $\eta$ when the two fields differ.
When the proposal radius is $\epsilon=\zeta\eta$ for fixed $\zeta>0$,
this contribution is of the same order as the proposal itself.
The continuous crossing time gives a starting estimate. The state error accumulated by Euler before the crossing shifts that
estimate. Rounding one crossing then changes the state used to predict
the next. Figure~\ref{fig:atomic-overview}
shows these effects.

Classical saltation and event-time sensitivity describe moving
continuous crossings. Finite perturbation analysis can track changes to
a discrete program, and prior work documents sensitivity errors from
time discretization
\citep{kong2024saltation,wardi2017pa,cassandras2010ipa,stewart2010optimal,nurkanovic2024finite}.
Our result concerns the joint limit in which fixed-step GD is refined
and each proposal shrinks as $\eta z$. For a fixed sequence of separated
transverse crossings, Theorem~\ref{thm:coupled-response} gives the actual
integer crossing steps and a uniform $O(\eta^2)$ endpoint expansion
away from the rounding boundaries. Theorem~\ref{thm:main} gives the local and continuous-flow
derivative limits for comparison.

Corollary~\ref{cor:directional-radius} gives a sufficient directional radius
for preserving all training masks. Proposition~\ref{prop:linear-barrier}
gives a strongly convex risk where even the best single vector has positive
limiting prediction error across three directions at the critical radius.
Theorem~\ref{thm:contractive-main} certifies finite loss orderings in
contractive affine regions. Section~\ref{sec:experiments} measures prediction
error, candidate selection, and coverage. The main theorem varies
initialization. Appendices~\ref{app:parameter-adjoint}
and~\ref{app:resolution-synthesis} address continuous outer parameters.

\begin{figure}[t]
\centering
\includegraphics[width=\textwidth]{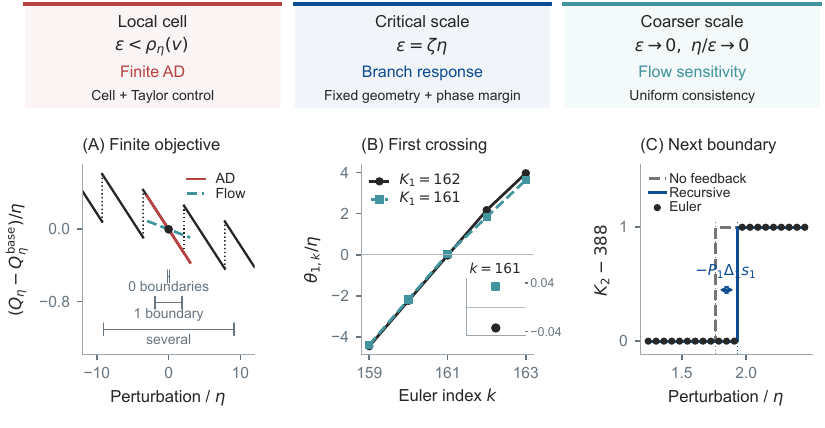}
\caption{\textbf{A finite proposal can change the activation step and the loss ordering.}
The top row summarizes sufficient conditions for using local AD, the branch
response, or flow sensitivity.
\textbf{(A)} Scalar GD at $N=1000$. The vertical axis shows
$(Q_\eta-Q_\eta^{\rm base})/\eta$, with $Q_\eta^{\rm base}=Q_\eta(x_0)$.
The brackets span zero, one, or several training-mask boundaries.
Both slope lines pass through the same base objective, and the flow slope
differentiates the continuous objective.
\textbf{(B)} Two nearby starts first cross the surface at different GD steps.
\textbf{(C)} Carrying the first rounding decision forward moves the
second crossing boundary. Points show independent Euler reruns.
The windows in (B)--(C) are set by the crossing geometry.}
\label{fig:atomic-overview}
\end{figure}

\section{From activation steps to a loss comparison}
\label{sec:setup}
\paragraph{A one-step delay changes the preferred proposal.}
Run GD from $x_0-\epsilon$ and $x_0+\epsilon$ and evaluate the loss of
each trained model. In this example the negative proposal keeps the
unit inactive for one more update and yields the lower loss.
We report the positive proposal's loss minus the negative proposal's
loss, divided by $2\epsilon$. A positive value therefore favors
$x_0-\epsilon$.
\begin{center}
\begin{tabular}{@{}lrrrr@{}}
\toprule
 & GD reruns & Local AD & Flow & Branch \\
\midrule
Normalized loss difference & $.750679$ & $-.121603$ & $-.0307043$ & $.751714$ \\
\bottomrule
\end{tabular}
\end{center}
Local AD uses the derivative of the base run. Flow uses the derivative
of the continuous training objective. Both favor $x_0+\epsilon$.
Branch predicts a discrete activation step for each start and favors
$x_0-\epsilon$, in agreement with the GD reruns.

To reproduce this E1 row, use the scalar risk in
Section~\ref{sec:scales} with $a=4,c=-3$, $N=204$, $\eta=1/N$,
$x_0=4(1-e^{\sqrt2/4})$, $\epsilon=.3/N$ and $q(y)=(y-1)^2/2$.
The GD update is $(1-\eta)x_k+4\eta$ before activation and
$(1-2\eta)x_k+\eta$ afterwards. For $y<0$, the first strictly positive
iterate has index $K(y)=\lceil\log((4-y)/4)/[-\log(1-\eta)]\rceil$
away from an exact hit. The negative, base, and positive starts give
$73,72,72$. Omitting the accumulated Euler error from Branch gives
$73,73,73$.

\paragraph{Activation events and crossing indices.}
An event is an isolated sign change of the preactivation
$h_{ir}(\theta(t))$ of unit $r$ on training example $i$.
It flips that entry of the activation mask.
We orient the switching function, or \emph{guard}, as $g_j=h_{ir}$
for activation and $g_j=-h_{ir}$ for deactivation. Both then cross
from negative to positive. The first positive index is the first step $k$ with
$g_j(\theta_k)>0$ near that crossing. It marks the first iterate on the
outgoing side, which is the inactive side for deactivation.

\paragraph{Local derivatives of the training run.}
For a fixed horizon $T$ and $N$ steps, set $\eta=T/N$ and write
$\theta_{k+1}=\theta_k+\eta f_{\alpha_k}(\theta_k)$, where
$\alpha_k$ is the mask at step $k$ and $f_\alpha=-\nabla L_\alpha$ is
the update field on region $\alpha$. The trained state is
$\Theta_{\eta,T}(\lambda)=\theta_N$, where $\lambda$ is the initialization
in the main theorem. A scalar objective $q$ can use the trained state
and the input, so $Q_\eta(\lambda)=q(\Theta_{\eta,T}(\lambda),\lambda)$.
The \emph{activation word} $(\alpha_0,\ldots,\alpha_{N-1})$ records
every executed mask. At a \emph{nonresonant} input, all executed
preactivations have strict signs. The word is locally fixed, so
\begin{equation}
 J_{k+1}=(I+\eta Df_{\alpha_k}(\theta_k))J_k,\qquad J_0=I,
 \qquad J_{\eta,T}^{\AD}=D\Theta_{\eta,T}.
 \label{eq:branchwise-ad}
\end{equation}
At an exact landing the selected AD product may exist even when the
classical finite-program derivative does not.

Let $\varphi_t$ be the continuous training flow and
$Q_0(\lambda)=q(\varphi_T(\lambda),\lambda)$.
We write $g_\eta=\nabla Q_\eta$ for the finite-run gradient,
$g_{\reg}$ for its regional limit, and $g_0=\nabla Q_0$ for the
flow gradient. The latter two use the derivatives of $q$ at the flow
endpoint, while $g_\eta$ uses the finite endpoint. We use $g_\eta$ where
$q$ is differentiable at its finite arguments, and use a $C^1$
neighborhood of the flow arguments when taking gradient limits.

\paragraph{Two reference limits.}
Write $f_j^-$ and $f_j^+$ for the fields immediately before and after
event $j$, $\Delta_j=f_j^+-f_j^-$ for their jump, and $n_j$ for
the oriented guard normal. Between crossings $\Phi_i$ transports a
small parameter change through a single region. At a crossing the
saltation matrix $\Xi_j$ accounts for the change in crossing time.
The reference path has a fixed, finite event sequence.
Each event has a neighborhood containing no other crossing (an
\emph{isolating tube}); all other preactivations stay away from zero.
Regional fields have bounded $C^1$ extensions and interfaces are $C^2$.
Both adjacent fields cross each oriented interface with normal speed at
least $\nu>0$, so crossings are transverse and in the same direction.
At an exact GD landing, the state convention selects an adjacent field
with that positive speed. Appendix~\ref{app:assumptions} gives the full conditions.

\begin{theorem}[Regional limit of finite-program derivatives]
\label{thm:main}
For this fixed itinerary, the GD interpolation has uniform state error
$O(\eta)$, and the selected derivative products converge to the regional
limit $J_{\reg}$. Writing $\Phi_i$ for ordinary variational propagation
between events,
\begin{equation}
 \begin{aligned}
 J_{\reg}&=\Phi_m\cdots\Phi_0,\qquad
 \Xi_j=I+\frac{\Delta_j n_j^\top}{n_j^\top f_j^-},\\
 D\varphi_T&=\Phi_m\Xi_m\Phi_{m-1}\Xi_{m-1}\cdots\Phi_1\Xi_1\Phi_0.
 \end{aligned}
 \label{eq:main-flow-derivative}
\end{equation}
Here $\Delta_j=f_j^+-f_j^-$ and $n_j$ is the oriented normal; for $m=0$, both products are $\Phi_0$.
On nonresonant meshes the selected product is the exact finite-program
derivative. Its regional-limit error is $o(1)$, and $O(\eta)$ when
the regional Jacobians are locally Lipschitz.
\end{theorem}
Thus GD states can converge to the flow while their AD derivatives
converge to a different object: $J_{\reg}$ omits crossing-time shifts,
whereas $D\varphi_T$ includes them through $\Xi_j$.
This mechanism is classical \citep{stewart2010optimal};
Appendix~\ref{app:general-proofs} proves it under the stated geometry.

The normalized loss difference used throughout the paper is
\begin{equation}
 D_{\eta,\epsilon}(v)=
 \frac{Q_\eta(\lambda+\epsilon v)-Q_\eta(\lambda-\epsilon v)}{2\epsilon},
 \qquad \|v\|=1.
 \label{eq:main-scale-response}
\end{equation}
The negative proposal has lower loss when $D_{\eta,\epsilon}(v)>0$.
The critical theorem treats initialization perturbations,
$\lambda=\theta_0$, with $q=q(\theta_T)$. Whether two endpoint Jacobians
differ, whether that difference survives projection onto the objective,
and whether it reverses a finite loss ordering are separate questions.

\section{Which predictor fits the proposal radius?}
\label{sec:resolution}
Fix an initialization $\lambda$, direction $v$ with $\|v\|=1$,
and terminal objective $q$. Write $x_T=\varphi_T(\lambda)$.
The following conditions justify three predictions for the same
normalized loss difference $D_{\eta,\epsilon}(v)$.

\paragraph{Local AD when the masks stay fixed.}
Let $\rho_\eta(v)$ be the supremum of radii for which the entire open
interval preserves all training and terminal-validation branches.
For $0<\epsilon<\rho_\eta(v)$, if the third directional derivative
of $Q_\eta$ is bounded by $M_{3,\eta}$ on the interval, the local AD
prediction $g_\eta^\top v$ has error at most $M_{3,\eta}\epsilon^2/6$.
Convergence along a refining sequence requires this remainder to vanish.
Near an exact landing, $\rho_\eta(v)/\eta$ can be arbitrarily small.

\paragraph{Branch response when the radius is proportional to $\eta$.}
Set $\epsilon=\zeta\eta$, with $\zeta>0$ fixed. The fractional event times
$p_N$ specify where the continuous crossings fall on the GD grid.
Section~\ref{sec:coupled-main} uses them to compute $e_{p_N}(z)$,
the first-order endpoint displacement for input change $\eta z$.
If the geometry of Theorem~\ref{thm:coupled-response} holds,
$q\in C^2$ near $x_T$, and all three inputs in $K=\{0,-\zeta v,\zeta v\}$
stay a fixed positive distance from the recursive rounding boundaries,
then on sufficiently fine grids
\[
 D_{\eta,\zeta\eta}
 =\underbrace{\frac{\nabla q(x_T)^\top
 [e_{p_N}(\zeta v)-e_{p_N}(-\zeta v)]}{2\zeta}}_{\widehat D_{\eta,\zeta}(v)}
 +O_K(\eta/\zeta).
\]

Apply the recursion to each proposed displacement. The two proposals
may have different activation words.

\paragraph{Flow when the vanishing radius is large relative to $\eta$.}
Suppose $\sup_U|Q_\eta-Q_0|\le C_0\eta$ and $Q_0\in C^1(U)$
on a fixed neighborhood containing the proposal interval.
Then $g_0^\top v$ has error at most
$C_0\eta/\epsilon+\omega_{\nabla Q_0}(\epsilon)$, with
$\omega_{\nabla Q_0}$ the local modulus of continuity of the flow
gradient. This error vanishes if $\epsilon\to0$ and $\eta/\epsilon\to0$.
Appendix~\ref{app:resolution-synthesis} gives the proofs and the
local $C^2$ remainder.

\paragraph{Why a proposal-dependent response is needed.}
In the strongly convex ReLU risk of Proposition~\ref{prop:linear-barrier},
one common critical radius and three fixed unit directions satisfy
\[
 \lim_{\substack{N\to\infty\\N\text{ odd}}}
 \inf_h\max_{v\in\mathcal V}|D_{\eta,\zeta\eta}(v)-h^\top v|
 =C_*>0,\qquad C_*\simeq .01676.
\]
Even the best vector chosen separately on each mesh has this error
floor. The branch response has $O(\eta)$ error.

\paragraph{Choosing a radius that preserves the training masks.}
For a given direction $v$, Corollary~\ref{cor:directional-radius}
compares each corrected clock's distance $d_j$ to the nearest integer
with its directional shift per unit of $\epsilon/\eta$.
The shift is $n_j^\top R_jv/u_j$, where $R_jv$ transports $v$ to event
$j$ and $u_j$ is the incoming normal speed. The smallest ratio gives
$Z_p(v)=\min_j d_ju_j/|n_j^\top R_jv|$, omitting zero denominators.
Along an admissible phase limit, for each fixed $0<r<Z_p(v)$,
the radius $\epsilon=\eta r$ preserves the training word on sufficiently
fine grids.
At finite $\eta$, reserve enough of each margin for the clock error and
check the terminal validation mask separately. In the scalar test,
using $.9$ of the leading-order radius preserves $9{,}799/9{,}801$
training sequences. The exact scalar clock correction preserves all $9{,}801$
(Appendix~\ref{app:finite-radius-check}).

\section{Predicting the GD crossing steps}
\label{sec:coupled-main}
For a proposal $x_0+\eta z$, the calculation has three parts.
Propagate the proposal and the accumulated Euler error up to the
next event. Use the resulting shift to predict the first GD step on
the outgoing side. Then carry the effect of that step forward to the
remaining events and the endpoint.

Fix $m$ events at $0<t_1<\cdots<t_m<T$, separated from each other
and the endpoints by $\Delta>0$. On segment $i$, let $f_i$ have a
bounded $C^2$ extension and let $x(t)$ be the base path. In each
isolating event tube exactly one $C^2$ surface $g_j=0$ can vanish;
all others have a positive margin. Away from these tubes all surfaces
have a positive margin on a fixed spatial neighborhood of the base path.
Orient each $g_j$ so that both adjacent speeds
$\nabla g_j^\top f_{j-1}$ and $\nabla g_j^\top f_j$ exceed $\nu>0$.
These conditions fix the order of the reference crossings.

Set $t_0=0,t_{m+1}=T$, $x_i=x(t_i)$ and $L_i=t_{i+1}-t_i$.
Let $P_i$ transport a perturbation along the smooth segment.
To compute the accumulated first-order Euler error on that segment,
solve the forced variational equation
\begin{equation}
 \dot w=Df_i(x(t))w-\tfrac12Df_i(x(t))f_i(x(t)),\quad
 w(t_i)=0,\quad w_i=w(t_{i+1}).
 \label{eq:main-smooth-bias}
\end{equation}
Define the guard normal $n_j=\nabla g_j(x_j)$, incoming normal speed
$u_j=n_j^\top f_{j-1}(x_j)$, and field jump
$\Delta_j=f_j(x_j)-f_{j-1}(x_j)$.
Let $p_j$ be the fractional part of $t_j/\eta$. It locates the reference
crossing within one GD step. For a scaled initial displacement $z$,
set $b_0=z$ and recurse:
\begin{align}
 a_j&=P_{j-1}b_{j-1}+w_{j-1}, &
 s_j&=\left\lceil p_j-\frac{n_j^\top a_j}{u_j}\right\rceil-p_j,
 \nonumber\\
 b_j&=a_j-\Delta_js_j,\quad 1\le j\le m,&
 e_p(z)&=P_mb_m+w_m.
 \label{eq:main-coupled-recursion}
\end{align}
The normal shift $n_j^\top a_j$ divided by the incoming speed $u_j$
converts the state error into a time shift, in units of GD steps.
The ceiling selects the outgoing step. If $s_j$ is positive, the run
uses the incoming field longer, producing the correction $-\Delta_js_j$.
The \emph{recursive phase margin} is the distance from the corrected
ceiling argument to the nearest integer, minimized over the events.
A common positive margin keeps the higher-order error from moving
any candidate across a rounding boundary.

\begin{theorem}[Finite-event critical-scale branch response]
\label{thm:coupled-response}
\label{thm:finite-smooth-response}
Under the stated geometry, let $K\subset\mathbb R^d$ be compact,
$0\in K$, and $\mu>0$. For every sufficiently fine canonical grid
$\eta=T/N$ whose phase vector $p_N=(\{t_j/\eta\})_{j=1}^m$ has
recursive phase margin at least $\mu$ uniformly on $K$, every Euler path
from $x_0+\eta z$, $z\in K$, has exactly the base event sequence.
Its first positive indices are
\[
 k_j(z)=\left\lceil t_j/\eta-n_j^\top a_j(p_N,z)/u_j\right\rceil.
\]
All executed signs are strict, and uniformly on $K$,
\begin{equation}
 \Theta_{\eta,T}(x_0+\eta z)
 =x_{m+1}+\eta e_{p_N}(z)+O_K(\eta^2).
 \label{eq:main-coupled-limit}
\end{equation}
For $q\in C^2$ near the endpoint, the rescaled objective change equals
$\nabla q(x_{m+1})^\top[e_{p_N}(z)-e_{p_N}(0)]+O_K(\eta)$.
On phase-convergent subsequences this becomes the corresponding $e_p$
law with error $O_K(\eta+\operatorname{dist}(p_N,p))$.
Constants depend on fixed $m,T,K$, field and surface derivative bounds,
transversality, event separation and tube margins. The mesh threshold
also depends on $\mu$; the endpoint remainder constant need not.
\end{theorem}

\paragraph{From endpoint error to response.}
The scaled errors $a_j$ and $b_j$ refer to time $t_j$ on the incoming
and outgoing smooth extensions, respectively. Their difference
accounts for the step offset $s_j=k_j-t_j/\eta$.
Applying the endpoint expansion to $z=\pm\zeta v$ gives the critical
response in Section~\ref{sec:resolution}.

The strict signs in Theorem~\ref{thm:coupled-response} verify the
predicted crossing indices. When a proposal changes an index,
its $O(\eta)$ endpoint contribution can survive normalization by
$2\zeta\eta$ and change which candidate has lower loss.

\paragraph{Why coupling matters.} With $P_{j-1:r}=P_{j-1}\cdots P_r$,
the $j$th scaled boundary contains
\begin{equation}
 p_j-\frac{n_j^\top A_j^{\rm free}(z)}{u_j}
 +\sum_{r<j}\frac{n_j^\top P_{j-1:r}\Delta_r}{u_j}s_r\in\mathbb Z,
 \label{eq:main-downstream}
\end{equation}
where $A_j^{\rm free}$ contains initial displacement and smooth bias.
Changing one upstream index can therefore change several later integers.
The no-upstream-feedback ablation deletes these threshold terms while
retaining the transported endpoint contributions.

\paragraph{Why the rounding is stable.} Appendix~\ref{app:finite-smooth-response}
derives the smooth bias from the one-step truncation error, identifies
each first positive index by two strict signs and normal monotonicity,
and propagates the outgoing error $a_j-\Delta_js_j$.
Surface curvature enters the $O(\eta^2)$ remainder under $C^2$
regularity. The argument works for deterministic phases and general
regional fields under the stated smoothness, separation, and margin
conditions. Section~\ref{sec:certificate} gives a finite bound under stronger
affine assumptions.

\paragraph{A concrete coupled training risk.}
For $L(\theta)=\tfrac12\|A\theta-b\|^2+
\tfrac12\sum_{j=1}^2[\operatorname{ReLU}(n_j^\top\theta)-c_j]^2$,
$A^\top A\succ0$ and $c_j\le0$ give global strong convexity.
Regional fields are affine, with $M_S=-A^\top A-\sum_{j\in S}n_jn_j^\top$,
and \eqref{eq:main-smooth-bias} reduces to
$w_i=-L_iP_iM_if_i(x_i)/2$. The fixed example
$A^\top A=\left(\begin{smallmatrix}1&.4\\.4&1.3\end{smallmatrix}\right)$,
$A^\top b=(2,2)$, $n_j=e_j$, $c=(-.4,-.6)$,
$\theta_0=(-.4,-1)$ and $T=1$ has two provably separated events.
Its regional commutator has norm $.4$, and its downstream coefficient
$n_2^\top P_1\Delta_1/u_2$ is strictly positive.
Appendix~\ref{app:coupled-response} gives the geometry and exact finite-cell jumps.

\section{An exact scalar response classification}
\label{sec:scales}
Let $L_{a,c}(x)=\tfrac12(x-a)^2+
\tfrac12(\operatorname{ReLU}(x)-c)^2$, with $a>0,c>-a$ and
$x\in(a(1-e^T),0)$. For $\eta=T/N<1/2$, set
$t=\log((a-x)/a)$, $P=(a+c)/2$,
$A=1-\eta$, $B_\eta=1-2\eta$ and $z_k=a(1-A^{-k})$.
Here $K$ is the first positive GD index. Each interval
$x\in(z_K,z_{K-1})$, $1\le K<N$, preserves that activation step.
The endpoint on the cell and its right-minus-left jump at an interior
boundary $1\le k<N$ are
\begin{equation}
 F_\eta(y)=P+[a+(y-a)A^K-P]B_\eta^{N-K},\qquad
 [F_\eta]_{z_k}=\eta c B_\eta^{N-k-1}.
 \label{eq:main-cell-jump}
\end{equation}
The training-cell radius is $\rho_\eta=\min(x-z_K,z_{K-1}-x)$; a smooth
terminal $q$ adds no branch constraint.
Writing $B=a-x$, $E=e^{-2(T-t)}$, $J=e^{-t}E$ and $F_0=P(1-E)$,
the critical response at phase $p=\lim\{t/[-\log(1-\eta)]\}$ is
\begin{equation}
 R_\zeta(p)=q'(F_0)\left[J+\frac{cE}{2\zeta}H(p,\zeta)\right],
 \quad H(p,\zeta)=\#\{j\in\mathbb Z:B|p-j|<\zeta\}.
 \label{eq:main-critical-phase}
\end{equation}
The count $H$ is the number of discrete activation boundaries between
the two proposals. Each crossed boundary contributes a jump to the response.
The phase law excludes $\zeta=B|p-j|$ for every integer $j$.
The corrected scalar phase includes the Euler bias in
\eqref{eq:main-smooth-bias}. When $\rho_\eta/\eta\to0$, preserving
the branch can require perturbations much smaller than the GD step.

For $\ell=2\zeta/B$, the count is $\lfloor\ell\rfloor$ or
$\lceil\ell\rceil$, with uniform-phase masses $1-r,r$, where
$r=\ell-\lfloor\ell\rfloor$. Consequently
\begin{equation}
 \mathbb E_pR_\zeta=Q_0'(x),\qquad
 \operatorname{Var}_pR_\zeta=
 [q'(F_0)cE/(2\zeta)]^2r(1-r).
 \label{eq:main-phase-classification}
\end{equation}
Positive integer $\ell$ cancels phase dependence on the allowed phases.
For irrational $t/T$, an equidistribution proof with uniform error control
away from boundaries yields canonical-mesh response frequencies.
Corollary~\ref{cor:phase-classification} in Appendix~\ref{app:phase-classification}
treats zero levels and exact-hit exceptions. These moments average
over mesh phase for the fixed scalar training problem.

\section{A finite comparison certificate}
\label{sec:certificate}
At a fixed mesh, a loss comparison is certified if both predicted
mask sequences pass verification and the predicted loss gap exceeds
the error bound. For contractive affine fields, the next theorem gives
the bound given the verified sequences. A failed sequence leaves the
comparison unresolved. The index-order screen used in the nonlinear
experiments checks only the proposed crossing indices.

\begin{theorem}[Contractive finite response]
\label{thm:contractive-main}
Let $m$ be fixed and $f_i(x)=M_ix+d_i$, with $M_i=M_i^\top\prec0$.
Choose supplied ordered knots $0=t_0<\cdots<t_{m+1}=T$, let
$L_i=t_{i+1}-t_i$, $P_i=e^{L_iM_i}$, and obtain scheduled states $x_i$
by the corresponding exact regional flows from $x_0$. For a
candidate word require $0=k_0<\cdots<k_{m+1}=N$ and
$\eta\|M_i\|_2<1$. Put $h_i=-M_i^{-1}d_i$,
$\delta_i=k_{i+1}-k_i-L_i/\eta$,
$A_i=P_i(\delta_iM_i-L_iM_i^2/2)$, and
$B_i=P_i[L_iM_i^3/3-\delta_iM_i^2/2+
(\delta_iM_i-L_iM_i^2/2)^2/2]$.
Starting from $e_0=z,v_0=0$, propagate
\[
 e_{i+1}=P_ie_i+A_i(x_i-h_i),\qquad
 v_{i+1}=P_iv_i+A_ie_i+B_i(x_i-h_i).
\]
If the entire candidate word passes its strict sign test, the endpoint
$\widetilde x=x_{m+1}+\eta e_{m+1}+\eta^2v_{m+1}$ has error at most
the explicit spectral remainder $R_{m+1}$ in
Eq.~\eqref{eq:finite-event-second-jet} in Appendix~\ref{app:contractive-certificate}.
For bounded displacements, offsets and regional data, with
$\eta\|M_i\|\le r_*<1$, the bound is uniformly $O(\eta^3)$.

Let $\widetilde x_\pm$ and $R_\pm$ correspond to $z=\pm\zeta v$.
If $q$ is $C^2$ on an open neighborhood of both endpoint-error balls
and $\|\nabla^2q\|\le M_q$ on the balls, then
\begin{align}
 \widehat D^{(2)}&=\frac{q(\widetilde x_+)-q(\widetilde x_-)}{2\zeta\eta},
 \nonumber\\
 |D_{\eta,\zeta\eta}-\widehat D^{(2)}|
 &\le\frac{\sum_{\sigma=\pm}
 [\|\nabla q(\widetilde x_\sigma)\|R_\sigma+M_qR_\sigma^2/2]}
 {2\zeta\eta}=:\mathcal E^{(2)}.
 \label{eq:main-certificate}
\end{align}
A strict margin $|\widehat D^{(2)}|>\mathcal E^{(2)}$ certifies the
candidate ordering and a gap of at least
$2\zeta\eta(|\widehat D^{(2)}|-\mathcal E^{(2)})$.
Failed words, violated domain conditions and insufficient margins are unresolved.
\end{theorem}

At fixed positive $\zeta$, the response bound is $O(\eta^2)$.
Mixed terms such as $A_ie_i$ carry the effect of earlier rounding
into the later regional updates.
\paragraph{Scope and cost.} Appendix~\ref{app:contractive-certificate}
proves the spectral tails and contractive defect recurrence. For affine switching surfaces in the plane,
each surface along a block is a constant plus two geometric sequences.
Its forward difference has at most one zero. Checking at most four integer
indices at the block endpoints or around an interior extremum verifies all signs. For $S$ surfaces, two candidates
need at most $8(m+1)S$ scalar sign evaluations, plus regional algebra.
The spectral remainder bound holds in any dimension; the constant-size word test is planar.
The experiments evaluate the real-arithmetic bound numerically, with
independent high-precision checks and disclosed base references.
Exact affine word evaluation is also inexpensive in this setting.

\section{Experiments on loss prediction and candidate choice}
\label{sec:experiments}
GD reruns supply the ground-truth loss difference in
\eqref{eq:main-scale-response}. E1 isolates the effect of one activation
step, E2 tests how it propagates to a second crossing, and E3 compares
predictors across radii. E4 asks whether a smaller prediction error
changes the choice among nine training-weight candidates.
Appendix~\ref{app:response-experiments} gives the protocols and complete
populations, including uncovered cases.

\paragraph{E1. One activation step can reverse the loss ordering.}
For the scalar task $a=4,c=-3,T=1,t=\sqrt2/4$,
$q(y)=(y-1)^2/2$ and $\epsilon=0.3/N$, the limiting responses are
$-0.122817$ and $0.751714$. Both infinitesimal signs are negative.
The finite loss ordering opposes both derivative predictions on
$1{,}034/9{,}801$ meshes ($10.5499\%$), close to the asymptotic frequency
$10.5328\%$ (Figure~\ref{fig:response-core}A).
The full rule predicts the three crossing indices on all 9,801 meshes.
Removing the smooth Euler bias reduces this count to 7,035
(Appendix~\ref{app:rebuttal-checks}).
The 20 closest critical endpoints agree with independent 60-digit
Euler checks within $2.27\times10^{-12}$.
Integer phase lengths $\ell=1,2,3$ approach the flow response.

\paragraph{E2. The first crossing can move the second.}
The coupled risk in Section~\ref{sec:coupled-main} uses $v=e_1$ and
$q(\theta)=(\theta_1-1)^2/2$, with zero coupling as a control.
Across $N=1000,\ldots,5000$ and three fixed ratios, the full and ablated
recursions predict different signed changes in the second crossing index
on 224 of 12,003 coupled rows. On that subset, the full rule has lower
error on 223 rows, with mean error $5.94\times10^{-5}$ versus
$8.17\times10^{-3}$ and zero sign errors versus eight.
The remaining row has an invalid predicted activation sequence.
The second-order test certifies $23{,}997/24{,}006$ dense comparisons
across both couplings, with zero incorrect decisions. Nine words fail verification.
Of 110 separately fixed coarse comparisons, 94 are certified and 16
unresolved. Two coupled certified reversals have actual objective gaps
$7.08\times10^{-5}$ and $2.01\times10^{-5}$, respectively
$0.097\%$ and $0.027\%$ of base loss. A third coupled reversal remains
unresolved.

\begin{figure}[t]
\centering
\includegraphics[width=\textwidth]{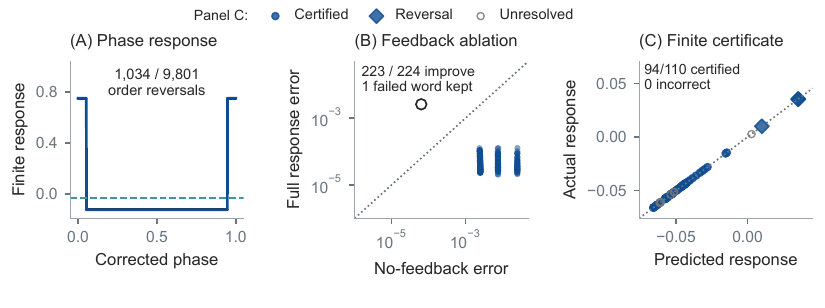}
\caption{\textbf{Scalar phase predictions, downstream rounding, and finite certificates.}
\textbf{(A)} All 9,801 scalar responses against the critical phase law and
flow prediction.
\textbf{(B)} The diagonal marks equal error. The hollow marker shows the
failed word among the 224 rows where the full and ablated predictions
disagree on the second-index change.
\textbf{(C)} The 101 valid coarse words against GD reruns. Numerically
evaluating the real-arithmetic bound certifies 94 comparisons, including
three reversals across the two couplings. Seven valid words and nine
failed words are unresolved.}
\label{fig:response-core}
\end{figure}

\paragraph{E3. Predictor accuracy depends on radius.}
We train a width-4 ReLU network with 17 parameters by full-batch
float64 GD on two moons (64 training and 64 validation samples).
We use $T=1$, $N=500,2000$, four radii from $10^{-5}$ to $10^{-2}$,
and initialization direction $e_1$. Nineteen of 20 seeds have a usable
reference, giving 38 seed--mesh pairs per radius and 152 of 160 planned
comparisons. We compare four predictors of the validation half-MSE
response (Figure~\ref{fig:resolution-network}). Local AD uses the
finite run, Regional uses its regional derivative limit, and Flow
includes the continuous crossing-time shifts. Branch predicts the
integer crossings of each candidate using the base reference path.

At the smallest radius, AD has median error $9.56\times10^{-12}$
and the smallest error among the four predictors on $25/38$ rows.
The branch-response predictor has the smallest error on $25/38$ and $30/38$
at $10^{-4}$ and $10^{-3}$. At $10^{-2}$, flow improves substantially
over AD: median errors $1.87\times10^{-4}$ versus $1.72\times10^{-3}$.
The branch-response predictor is best on $22/38$ rows and flow on $11/38$.
These comparisons describe the four tested radii.

\paragraph{When the predicted activation word fails.}
Of 152 reference rows, 81 match both complete training words, nine have
strictly ordered proposed events but an incorrect word, and 62 have
unordered proposed events. Branch's median/maximum errors in the three strata are
$9.03\times10^{-6}/7.77\times10^{-4}$,
$8.16\times10^{-5}/5.27\times10^{-4}$, and
$7.17\times10^{-5}/4.87\times10^{-2}$.
The strata differ in radius and event count. Two of the 81 matched
rows change the terminal validation mask. The recursive phase margin
has an area under the ROC curve of $0.647$ for word agreement.
All four methods give the correct sign on this direction.
Appendix~\ref{app:resolution-network} reports the tails, objective gaps,
and separate gradient-ranking and outer-update diagnostics.

\begin{figure}[t]
\centering
\includegraphics[width=\textwidth]{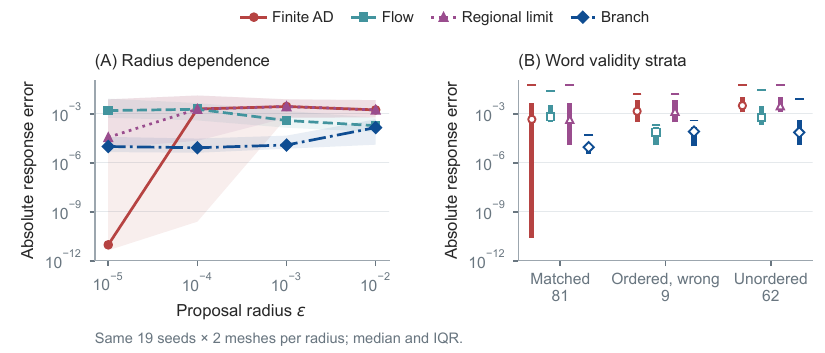}
\caption{\textbf{Prediction errors across radii and activation-word outcomes.}
\textbf{(A)} The same 19 seeds and two meshes at each radius
(152 reference rows). Lines show medians and bands show interquartile
ranges (IQRs).
\textbf{(B)} Both words match the GD reruns in 81 rows. The proposed
indices are strictly ordered but a word differs in nine rows, and strict
order fails in 62. Hollow markers show medians, vertical bars show IQRs,
and caps show the 95th percentiles. These are descriptive summaries of
these runs. Eight of the 160 planned comparisons lack a reference.
See Appendix~\ref{app:resolution-network} for phase-margin diagnostics.}
\label{fig:resolution-network}
\end{figure}

\paragraph{E4. Training-weight choices and coverage.}
We vary four group weights, retrain each of nine candidates, and
compare held-out losses. Across 20 seeds at each of widths 4, 8, and 16,
two meshes and two radii give 240 planned decisions, of which 228 have
a reference. The order screen requires strictly increasing, interior
crossing indices for every candidate. It admits 95 decisions (50, 36,
and 9 at widths 4, 8, and 16). The table compares mean absolute error
(MAE) in candidate loss changes and agreement with the best of nine.

\begin{center}
\begin{tabular}{@{}lrrr@{}}
\toprule
 & \multicolumn{2}{c}{All reference decisions} & \multicolumn{1}{c}{Passed order screen}\\
\cmidrule(lr){2-3}\cmidrule(l){4-4}
Predictor & Change MAE & Best choice & Best choice\\
\midrule
AD     & $3.81\times10^{-6}$ & 225/228 & 94/95\\
Flow   & $1.68\times10^{-6}$ & 225/228 & 94/95\\
Branch & $3.82\times10^{-7}$ & 227/228 & 95/95\\
\bottomrule
\end{tabular}
\end{center}

On the admitted set, Branch changes one choice, with a held-out loss
reduction of $5.32\times10^{-7}$. Mean regret for either baseline is
$5.60\times10^{-9}$, where regret is the selected loss minus the best
candidate loss. Ten admitted decisions have an incorrect activation
sequence for at least one candidate. The other 85 match all nine
sequences. These are empirical predictions after an index-order screen.
Certification also requires complete sequence verification and a
finite remainder bound. The nonlinear audit evaluates no such bound,
including on the 85 matched decisions
(Appendix~\ref{app:rebuttal-checks}).

\paragraph{Coverage when masks change together.}
Joining the weight decisions to the base GD replay, 108 of 228 have
a base update with multiple mask changes. The order screen admits
0 of these 108, although Branch matches all nine sequences in 63.
Among the other 120 decisions, 95 pass the screen and 103 match all
nine sequences (Table~\ref{tab:crowding-coverage}).

\paragraph{Refining the same candidate task.}
At width 16, refining from $N=2000$ to $N=32000$ raises screen
coverage from 9/40 to 31/40 planned decisions and complete agreement
from 25/36 to 35/36 references on the same seed--radius-ratio pairs.
At the finest mesh, central-response MAE is $6.65\times10^{-4}$ for
AD, $4.05\times10^{-4}$ for flow, and $3.06\times10^{-7}$ for Branch.
Each 40-decision plan includes four uncovered cases
(Appendix~\ref{app:paired-refinement}). The radii shrink as the mesh
is refined.

\section{Related work}
\label{sec:related}
\paragraph{Differentiating finite training.}
Hyperparameter optimization differentiates validation performance through
a finite training run \citep{maclaurin2015reversible}; bilevel formulations
make this finite inner objective explicit \citep{franceschi2018bilevel}.
Initialization meta-learning updates parameters through one or more
adaptation steps \citep{finn2017maml}, while example reweighting uses a
validation derivative through a one-step training surrogate
\citep{ren2018reweight}. These applications motivate a finite outer change,
and the need to quantify its prediction error at a specified radius.
\citet[Sec.~4, Fig.~10]{maclaurin2015reversible} already observed that
hypergradients can be uninformative about a broader objective trend in
unstable training. We quantify a different, transverse-switching mechanism
that also occurs in our strongly convex examples. We study the fixed-horizon,
full-batch Euler setting of this broader class of finite training problems.

\paragraph{Event and finite perturbation analysis.}
Saltation, endogenous event sensitivity and ordered flow selections are
classical \citep{hiskens2000trajectory,cassandras2010ipa,burden2016event,
kong2024saltation}. Finite perturbation analysis already reconstructs
changed sample paths and propagates downstream branch decisions
\citep{ho1983finitepa,ho1991perturbation,ho1992concepts,
cassandras1985eventdomain,wardi2017pa}. Stewart--Anitescu demonstrate
discretized sensitivity failure and artificial optimization structure
\citep{stewart2010optimal}; FESD recovers flow sensitivity using
event-aligned discretization \citep{nurkanovic2024finite}.
Finite perturbation analysis (FPA) can encode our fixed-grid program.
Theorem~\ref{thm:coupled-response} studies a family of such programs
with input displacement $\eta z$. It derives the integer crossing
steps on the original GD grid, including the smooth Euler bias and
the influence of earlier rounding, and controls the endpoint error
uniformly away from the resulting boundaries.
Theorem~\ref{thm:contractive-main} bounds finite comparisons in
contractive affine regions.

\paragraph{Derivative targets and limiting objects.}
Fixed-program AD correctness \citep{lee2020correctness,lee2023machine,
bolte2020mathematical,park2024what}, contractive fixed-point
differentiation \citep{bolte2022iterative,grazzi2024nonsmooth}, and
training-flow approximations \citep{elkabetz2021continuous} concern
distinct objects. Architectural Euler limits may share the numerical
structure \citep{xu2023correcting,gao2025global,brandle2026continuous}.
Our radius-dependent
response complements these derivative targets. Its three sufficient
regimes combine classical local/consistency bounds with the critical-scale
theorem. Nonlinear directional derivatives of event-selected flows
are also classical \citep{burden2016event}. Our linear-prediction obstruction
concerns finite responses around differentiable points of changing
fixed-grid programs, retaining their integer phases.
Appendix~\ref{app:prior-art-comparison} records exact source locators and
the remaining primary-text access gaps.

\section{Discussion and limitations}
The grid must resolve the event sequence. In the base GD replay, multiple
mask changes occur in one update in $8/20$, $11/20$, and $19/20$ runs at
widths 4, 8, and 16 with $N=500$. At $N=2000$, the counts are $3/20$,
$7/20$, and $12/20$. Appendix~\ref{app:event-crowding} compares these counts
with continuous event spacing on the 57 validated references.
Extending the expansion to growing width or longer horizons would
require control of event counts, separation, normal speed, derivative
bounds, and rounding margins. Grazing, sliding, and simultaneous
continuous crossings require a different local analysis.
Appendix~\ref{app:statistical-interface} identifies the covariance terms
for a statistical extension. Appendix~\ref{app:zero-margin} treats exact
rounding boundaries.

The controlled reweighting task has lower prediction error and a small
observed selection gain, with one corrected choice among 95 screened
decisions. The experiments use prescribed radii and fixed candidate sets. On actual
outer updates, logging the radius and directional shift relative to the
nearest rounding boundary would measure how often critical crossings
arise. Paired reruns could then measure how often they change the chosen
candidate and by how much. The finite certificate applies to contractive
affine regions, and broader optimizer gains and reference-computation
cost remain open empirical questions.

\label{scientificmainend}

\section*{Reproducibility statement}
Appendices~\ref{app:assumptions}--\ref{app:finite-event-certificate}
state the assumptions and prove the main results;
Appendix~\ref{app:resolution-synthesis} proves the resolution synthesis.
Appendices~\ref{app:response-experiments}--\ref{app:derivative-coverage}
specify experiments, retained failures and independent-check coverage.
Appendix~\ref{app:finite-proposal-consequences} proves the additional
radius and linear-prediction consequences, and
Appendix~\ref{app:paired-refinement} states the paired refinement
protocol and complete finite-grid results.
Numerical summaries and vector figures were generated from preserved arrays
using the documented analysis procedures. Analytic curves follow the
stated formulas.

\section*{Statement on the use of AI assistants}
Generative AI tools assisted with literature search; formulating, checking,
drafting, and editing mathematical claims and proofs; experimental design and
implementation; result interpretation; source-code editing; figure
preparation; and reproducibility audits.  The tools did not serve as sources
of empirical measurements.  Every reported value was checked against
preserved machine-readable output, selected finite-program and flow derivatives were independently cross-checked
by centered finite differences with coverage detailed in Appendix~\ref{app:repro},
and event-aware forward and reverse
calculations were cross-checked in the analytic examples. The human authors retain responsibility for
the final manuscript, including every claim, proof, citation, numerical result,
and artifact.

\label{mainend}

\bibliography{references}

\begin{thebibliography}{28}
\providecommand{\natexlab}[1]{#1}
\providecommand{\url}[1]{\texttt{#1}}
\expandafter\ifx\csname urlstyle\endcsname\relax
  \providecommand{\doi}[1]{doi: #1}\else
  \providecommand{\doi}{doi: \begingroup \urlstyle{rm}\Url}\fi

\bibitem[Ambrosio et~al.(2000)Ambrosio, Fusco, and Pallara]{ambrosio2000bv}
Luigi Ambrosio, Nicola Fusco, and Diego Pallara.
\newblock \emph{Functions of Bounded Variation and Free Discontinuity
  Problems}.
\newblock Oxford Mathematical Monographs. Oxford University Press, Oxford,
  2000.
\newblock \doi{10.1093/oso/9780198502456.001.0001}.
\newblock URL \url{https://academic.oup.com/book/53762}.

\bibitem[Bolte \& Pauwels(2020)Bolte and Pauwels]{bolte2020mathematical}
J{\'e}r{\^o}me Bolte and Edouard Pauwels.
\newblock A mathematical model for automatic differentiation in machine
  learning.
\newblock In \emph{Advances in Neural Information Processing Systems},
  volume~33, pp.\  10809--10819. Curran Associates, Inc., 2020.
\newblock URL \url{https://arxiv.org/abs/2006.02080}.

\bibitem[Bolte et~al.(2022)Bolte, Pauwels, and Vaiter]{bolte2022iterative}
J{\'e}r{\^o}me Bolte, Edouard Pauwels, and Samuel Vaiter.
\newblock Automatic differentiation of nonsmooth iterative algorithms.
\newblock In \emph{Advances in Neural Information Processing Systems},
  volume~35, pp.\  26404--26417, 2022.
\newblock \doi{10.52202/068431-1915}.
\newblock URL
  \url{https://proceedings.neurips.cc/paper_files/paper/2022/hash/a9077da44185792cb63599cc9e0357bc-Abstract-Conference.html}.

\bibitem[Br{\"a}ndle et~al.(2026)Br{\"a}ndle, Eisenmann, G{\"o}tz, and
  Durstewitz]{brandle2026continuous}
Alena Br{\"a}ndle, Lukas Eisenmann, Florian G{\"o}tz, and Daniel Durstewitz.
\newblock Continuous-time piecewise-linear recurrent neural networks, 2026.
\newblock URL \url{https://arxiv.org/abs/2602.15649}.
\newblock arXiv preprint, version 2, August 6, 2026.

\bibitem[Burden et~al.(2016)Burden, Sastry, Koditschek, and
  Revzen]{burden2016event}
Samuel~A. Burden, S.~Shankar Sastry, Daniel~E. Koditschek, and Shai Revzen.
\newblock Event-selected vector field discontinuities yield
  piecewise-differentiable flows.
\newblock \emph{SIAM Journal on Applied Dynamical Systems}, 15\penalty0
  (2):\penalty0 1227--1267, 2016.
\newblock \doi{10.1137/15M1016588}.
\newblock URL \url{https://epubs.siam.org/doi/10.1137/15M1016588}.

\bibitem[Cassandras \& Ho(1985)Cassandras and Ho]{cassandras1985eventdomain}
Christos~G. Cassandras and Yu-Chi Ho.
\newblock An event domain formalism for sample path perturbation analysis of
  discrete event dynamic systems.
\newblock \emph{IEEE Transactions on Automatic Control}, 30\penalty0
  (12):\penalty0 1217--1221, 1985.
\newblock \doi{10.1109/TAC.1985.1103875}.

\bibitem[Cassandras et~al.(2010)Cassandras, Wardi, Panayiotou, and
  Yao]{cassandras2010ipa}
Christos~G. Cassandras, Yorai Wardi, Christos~G. Panayiotou, and Chenyang Yao.
\newblock Perturbation analysis and optimization of stochastic hybrid systems.
\newblock \emph{European Journal of Control}, 16\penalty0 (6):\penalty0
  642--661, 2010.
\newblock \doi{10.3166/ejc.16.642-661}.
\newblock URL \url{https://people.bu.edu/cgc/Published/EJC_12_10.pdf}.

\bibitem[Demengel(1984)]{demengel1984hessien}
Fran{\c{c}}oise Demengel.
\newblock Fonctions {\`a} hessien born{\'e}.
\newblock \emph{Annales de l'Institut Fourier}, 34\penalty0 (2):\penalty0
  155--190, 1984.
\newblock \doi{10.5802/aif.969}.
\newblock URL \url{https://www.numdam.org/item/AIF_1984__34_2_155_0/}.

\bibitem[Elkabetz \& Cohen(2021)Elkabetz and Cohen]{elkabetz2021continuous}
Omer Elkabetz and Nadav Cohen.
\newblock Continuous vs. discrete optimization of deep neural networks.
\newblock In \emph{Advances in Neural Information Processing Systems},
  volume~34, pp.\  4947--4960, 2021.
\newblock URL
  \url{https://papers.nips.cc/paper/2021/hash/274ad4786c3abca69fa097b85867d9a4-Abstract.html}.

\bibitem[Finn et~al.(2017)Finn, Abbeel, and Levine]{finn2017maml}
Chelsea Finn, Pieter Abbeel, and Sergey Levine.
\newblock Model-agnostic meta-learning for fast adaptation of deep networks.
\newblock In \emph{Proceedings of the 34th International Conference on Machine
  Learning}, volume~70 of \emph{Proceedings of Machine Learning Research}, pp.\
   1126--1135. PMLR, 2017.
\newblock URL \url{https://proceedings.mlr.press/v70/finn17a.html}.

\bibitem[Franceschi et~al.(2018)Franceschi, Frasconi, Salzo, Grazzi, and
  Pontil]{franceschi2018bilevel}
Luca Franceschi, Paolo Frasconi, Saverio Salzo, Riccardo Grazzi, and
  Massimiliano Pontil.
\newblock Bilevel programming for hyperparameter optimization and
  meta-learning.
\newblock In \emph{Proceedings of the 35th International Conference on Machine
  Learning}, volume~80 of \emph{Proceedings of Machine Learning Research}, pp.\
   1568--1577. PMLR, 2018.
\newblock URL \url{https://proceedings.mlr.press/v80/franceschi18a.html}.

\bibitem[Gao et~al.(2025)Gao, Sun, Liu, and Gao]{gao2025global}
Tianxiang Gao, Siyuan Sun, Hailiang Liu, and Hongyang Gao.
\newblock Global convergence in neural {ODE}s: Impact of activation functions.
\newblock In \emph{The Thirteenth International Conference on Learning
  Representations}, 2025.
\newblock URL
  \url{https://proceedings.iclr.cc/paper_files/paper/2025/hash/2fd37205c429a30cc1bb0afdbbf96857-Abstract-Conference.html}.

\bibitem[Grazzi et~al.(2024)Grazzi, Pontil, and Salzo]{grazzi2024nonsmooth}
Riccardo Grazzi, Massimiliano Pontil, and Saverio Salzo.
\newblock Nonsmooth implicit differentiation: Deterministic and stochastic
  convergence rates.
\newblock In \emph{Proceedings of the 41st International Conference on Machine
  Learning}, volume 235 of \emph{Proceedings of Machine Learning Research},
  pp.\  16250--16274. PMLR, 2024.
\newblock URL \url{https://proceedings.mlr.press/v235/grazzi24a.html}.

\bibitem[Hiskens \& Pai(2000)Hiskens and Pai]{hiskens2000trajectory}
Ian~A. Hiskens and M.~A. Pai.
\newblock Trajectory sensitivity analysis of hybrid systems.
\newblock \emph{IEEE Transactions on Circuits and Systems I: Fundamental Theory
  and Applications}, 47\penalty0 (2):\penalty0 204--220, 2000.
\newblock \doi{10.1109/81.828574}.
\newblock URL \url{https://ieeexplore.ieee.org/document/828574}.

\bibitem[Ho(1992)]{ho1992concepts}
Yu-Chi Ho.
\newblock Perturbation analysis: Concepts and algorithms.
\newblock In \emph{Proceedings of the 1992 Winter Simulation Conference}, pp.\
  231--240, 1992.
\newblock \doi{10.1145/167293.167340}.
\newblock URL \url{https://www.informs-sim.org/wsc92papers/1992_0024.pdf}.

\bibitem[Ho \& Cao(1991)Ho and Cao]{ho1991perturbation}
Yu-Chi Ho and Xi-Ren Cao.
\newblock \emph{Perturbation Analysis of Discrete Event Dynamic Systems},
  volume 145 of \emph{The Kluwer International Series in Engineering and
  Computer Science}.
\newblock Kluwer Academic Publishers, 1991.
\newblock \doi{10.1007/978-1-4615-4024-3}.

\bibitem[Ho et~al.(1983)Ho, Cao, and Cassandras]{ho1983finitepa}
Yu-Chi Ho, Xi-Ren Cao, and Christos~G. Cassandras.
\newblock Infinitesimal and finite perturbation analysis for queueing networks.
\newblock \emph{Automatica}, 19\penalty0 (4):\penalty0 439--445, 1983.
\newblock \doi{10.1016/0005-1098(83)90060-2}.

\bibitem[Kong et~al.(2024)Kong, Payne, Zhu, and Johnson]{kong2024saltation}
Nathan~J. Kong, J.~Joe Payne, James Zhu, and Aaron~M. Johnson.
\newblock Saltation matrices: The essential tool for linearizing hybrid
  dynamical systems.
\newblock \emph{Proceedings of the IEEE}, 112\penalty0 (6):\penalty0 585--608,
  2024.
\newblock \doi{10.1109/JPROC.2024.3440211}.
\newblock URL \url{https://ieeexplore.ieee.org/document/10638633}.

\bibitem[Lee et~al.(2020)Lee, Yu, Rival, and Yang]{lee2020correctness}
Wonyeol Lee, Hangyeol Yu, Xavier Rival, and Hongseok Yang.
\newblock On correctness of automatic differentiation for non-differentiable
  functions.
\newblock In \emph{Advances in Neural Information Processing Systems},
  volume~33, pp.\  6719--6730. Curran Associates, Inc., 2020.
\newblock URL
  \url{https://proceedings.neurips.cc/paper/2020/hash/4aaa76178f8567e05c8e8295c96171d8-Abstract.html}.

\bibitem[Lee et~al.(2023)Lee, Park, and Aiken]{lee2023machine}
Wonyeol Lee, Sejun Park, and Alex Aiken.
\newblock On the correctness of automatic differentiation for neural networks
  with machine-representable parameters.
\newblock In Andreas Krause, Emma Brunskill, Kyunghyun Cho, Barbara Engelhardt,
  Sivan Sabato, and Jonathan Scarlett (eds.), \emph{Proceedings of the 40th
  International Conference on Machine Learning}, volume 202 of
  \emph{Proceedings of Machine Learning Research}, pp.\  19094--19140. PMLR,
  2023.
\newblock URL \url{https://proceedings.mlr.press/v202/lee23p.html}.

\bibitem[Maclaurin et~al.(2015)Maclaurin, Duvenaud, and
  Adams]{maclaurin2015reversible}
Dougal Maclaurin, David Duvenaud, and Ryan~P. Adams.
\newblock Gradient-based hyperparameter optimization through reversible
  learning.
\newblock In \emph{Proceedings of the 32nd International Conference on Machine
  Learning}, volume~37 of \emph{Proceedings of Machine Learning Research}, pp.\
   2113--2122. PMLR, 2015.
\newblock URL \url{https://proceedings.mlr.press/v37/maclaurin15.html}.

\bibitem[Nurkanovi{\'c} et~al.(2024)Nurkanovi{\'c}, Sperl, Albrecht, and
  Diehl]{nurkanovic2024finite}
Armin Nurkanovi{\'c}, Mario Sperl, Sebastian Albrecht, and Moritz Diehl.
\newblock Finite elements with switch detection for direct optimal control of
  nonsmooth systems.
\newblock \emph{Numerische Mathematik}, 156\penalty0 (3):\penalty0 1115--1162,
  2024.
\newblock \doi{10.1007/s00211-024-01412-z}.
\newblock URL
  \url{https://link.springer.com/article/10.1007/s00211-024-01412-z}.

\bibitem[Park et~al.(2024)Park, Chun, and Lee]{park2024what}
Sejun Park, Sanghyuk Chun, and Wonyeol Lee.
\newblock What does automatic differentiation compute for neural networks?
\newblock In \emph{The Twelfth International Conference on Learning
  Representations}, pp.\  52559--52585, 2024.
\newblock URL
  \url{https://proceedings.iclr.cc/paper_files/paper/2024/hash/e8711daef520be07cb9852390c673de8-Abstract-Conference.html}.

\bibitem[Ren et~al.(2018)Ren, Zeng, Yang, and Urtasun]{ren2018reweight}
Mengye Ren, Wenyuan Zeng, Bin Yang, and Raquel Urtasun.
\newblock Learning to reweight examples for robust deep learning.
\newblock In \emph{Proceedings of the 35th International Conference on Machine
  Learning}, volume~80 of \emph{Proceedings of Machine Learning Research}, pp.\
   4334--4343. PMLR, 2018.
\newblock URL \url{https://proceedings.mlr.press/v80/ren18a.html}.

\bibitem[Stewart \& Anitescu(2010)Stewart and Anitescu]{stewart2010optimal}
David~E. Stewart and Mihai Anitescu.
\newblock Optimal control of systems with discontinuous differential equations.
\newblock \emph{Numerische Mathematik}, 114\penalty0 (4):\penalty0 653--695,
  2010.
\newblock \doi{10.1007/s00211-009-0262-2}.
\newblock URL
  \url{https://link.springer.com/article/10.1007/s00211-009-0262-2}.

\bibitem[Wardi et~al.(2017)Wardi, Cassandras, and Cao]{wardi2017pa}
Yorai Wardi, Christos~G. Cassandras, and Xi-Ren Cao.
\newblock Perturbation analysis: A framework for data-driven control and
  optimization of discrete event and hybrid systems.
\newblock \emph{IFAC-PapersOnLine}, 50\penalty0 (1):\penalty0 3028--3038, 2017.
\newblock \doi{10.1016/j.ifacol.2017.08.671}.
\newblock URL \url{https://people.bu.edu/cgc/Published/IFAC17_PAhistory.pdf}.
\newblock Author-hosted congress preprint has printed pages 3083--3093.

\bibitem[Xu et~al.(2023)Xu, Chen, and Li]{xu2023correcting}
Yewei Xu, Shi Chen, and Qin Li.
\newblock Correcting auto-differentiation in neural-{ODE} training, 2023.
\newblock URL \url{https://arxiv.org/abs/2306.02192}.
\newblock Version 3, revised 2026; accepted for publication in SIAM Journal on
  Applied Mathematics.

\bibitem[Zhang \& Chan(2007)Zhang and Chan]{zhang2007mpa}
He~Zhang and Wai Kin~Victor Chan.
\newblock Mathematical programming-based perturbation analysis for {GI/G/1}
  queues.
\newblock In \emph{Proceedings of the 2007 Winter Simulation Conference}, pp.\
  553--559, 2007.
\newblock \doi{10.1109/WSC.2007.4419647}.
\newblock URL \url{https://www.informs-sim.org/wsc07papers/066.pdf}.

\end{thebibliography}
\bibliographystyle{style/iclr2027/iclr2027_conference}

\appendix
\section*{Guide to the appendix}
The main argument is proved in Appendices~\ref{app:finite-smooth-response}
and~\ref{app:finite-event-certificate}. The following guide separates those
proofs from classical sensitivity background and supporting experiments.
\begin{center}
\begin{tabular}{@{}p{.44\linewidth}p{.50\linewidth}@{}}
\toprule
Question & Where to find it\\
\midrule
Which assumptions are required? & Appendix~\ref{app:assumptions}; the stronger critical-scale geometry is in Appendix~\ref{app:finite-smooth-response}.\\
Why does finite-program AD select the regional limit? & Appendix~\ref{app:general-proofs}.\\
How are discrete crossings and their errors derived? & Appendix~\ref{app:finite-smooth-response}; the explicit two-event risk is in Appendix~\ref{app:coupled-response}.\\
When is a finite comparison certified? & Appendix~\ref{app:contractive-certificate}, including the complete planar word test.\\
How do the three resolutions connect? & Appendices~\ref{app:perturbation-scales} and~\ref{app:resolution-synthesis}; phase frequencies are in Appendix~\ref{app:phase-classification}.\\
What changes for outer parameters or smoothing? & Appendices~\ref{app:parameter-adjoint} and~\ref{app:smoothing}.\\
What evidence supports each experiment? & Appendix~\ref{app:response-experiments}; network tails and coverage are in Appendices~\ref{app:resolution-network} and~\ref{app:derivative-coverage}.\\
What is added to prior work? & Appendix~\ref{app:prior-art-comparison}.\\
\bottomrule
\end{tabular}
\end{center}
\section{Assumptions, notation, and excluded regimes}
\label{app:assumptions}

Let \(\theta\in\mathbb R^d\).  In a neighborhood of a compact tube \(K\)
around the reference trajectory, a finite family of open regions
\(\{\mathcal R_\alpha\}_\alpha\) forms a piecewise-smooth partition.  On
\(\mathcal R_\alpha\), write \(f_\alpha\) for the regional vector field and
\(A_\alpha=Df_\alpha\).  The state and derivative-selection theorem is stated
for this piecewise-\(C^1\) field.  In the hard-ReLU gradient specialization,
\[
 \mathcal L(\theta)=\mathcal L_\alpha(\theta),\qquad
 f_\alpha(\theta)=-\nabla\mathcal L_\alpha(\theta),\qquad
 A_\alpha(\theta)=-H_\alpha(\theta).
\]
The following are the common base assumptions for the sharp-interface
results; individual theorems state any additions.

\begin{enumerate}
  \item[(A1)] Each \(f_\alpha\) has a \(C^1\) extension to a neighborhood of
  \(\overline{\mathcal R_\alpha}\cap K\), and \(f_\alpha\) and
  \(A_\alpha=Df_\alpha\) are uniformly bounded there.  When a result invokes
  gradient, Hessian, convexity, or ReLU structure, each \(\mathcal L_\alpha\)
  additionally has a \(C^2\) extension, its one-sided loss traces agree on
  common interfaces, and \(f_\alpha=-\nabla\mathcal L_\alpha\).

  \item[(A2)] Every interface met by the reference trajectory is locally a
  regular \(C^2\) hypersurface \(\Sigma_j=\{g_j=0\}\), with
  \(n_j(\theta):=\nabla g_j(\theta)\ne0\).  The defining function is oriented
  so that the trajectory moves from \(g_j<0\) to \(g_j>0\).  At the event
  state \(z_j\), the associated unit normal is
  \(\nu_j=n_j(z_j)/\|n_j(z_j)\|\).  We use \(n_j\) in formulas invariant to
  rescaling of \(g_j\), and \(\nu_j\) whenever a coefficient or spectral
  statement requires a unit normal.

  \item[(A3)] The patched Carath\'eodory solution
  \(\theta(t)=\varphi_t(\theta_0)\) exists on \([0,T]\), remains in \(K\), and
  has exactly \(m<\infty\) events
  \[
    0<t_1<\cdots<t_m<T,\qquad z_j=\theta(t_j).
  \]
  Consecutive events and the endpoints of the time interval are separated by
  at least \(\Delta_t>0\).  Each event has an isolating neighborhood containing
  no other interface met by the trajectory.  On a neighborhood of
  \(\theta_0\), this event count, order, separation, and itinerary persist.

  \item[(A4)] The crossings are uniformly same-direction transverse.  For
  some \(c_0>0\), after shrinking the isolating neighborhoods if necessary,
  \[
    n_j(\theta)^\top f_j^-(\theta)\ge c_0,
    \qquad
    n_j(\theta)^\top f_j^+(\theta)\ge c_0
  \]
  on the respective one-sided neighborhoods.

  Here ``same-direction'' is local to the oriented event: the normal points
  from the incoming region to the outgoing region.  It does not require all
  ReLU units to switch in one common global direction; it excludes grazing,
  sliding, and opposing one-sided normal velocities at that crossing.

  \item[(A5)] There is no grazing, sliding, chattering, simultaneous
  multi-surface hit, event creation, Zeno accumulation, or event at the
  terminal time \(T\).
\end{enumerate}

The assumptions in this section govern the derivative-selection foundation. The critical-scale theorem in Appendix~\ref{app:finite-smooth-response} states stronger $C^2$ field bounds and quantitative tubes explicitly, and derives its discrete itinerary from those data.

These assumptions give a unique piecewise-classical trajectory and a locally
stable event order.  At an exact discrete boundary landing, the state update
selects one of the adjacent one-sided fields whose oriented normal component
remains at least \(c_0\).  This
condition is needed for the state estimate.  Whenever ordinary AD is
identified with a classical derivative of the discrete endpoint map, we
additionally require that no iterate land exactly on an interface.  At an exact
hit, software AD still returns a selected branch product, but that product need
not be a classical derivative.

We use the canonical meshes \(\eta_N=T/N\).  Let \(\theta_k^{\eta_N}\)
denote the Euler/GD iterates and define the
endpoint map and its piecewise-linear interpolation by
\[
 \Theta_{\eta_N,T}(\theta_0):=\theta_N^{\eta_N},
 \qquad
 \Theta_{\eta_N}(t;\theta_0):=\text{the interpolation of }
 \{\theta_k^{\eta_N}\}_{k=0}^N.
\]
We use \(J_{\eta,T}^{\mathrm{sel}}\) for the product returned by the executed
branch convention, including at exact hits.  On a nonresonant mesh it is both
ordinary branchwise AD and the classical derivative \(D\Theta_{\eta,T}\), and
we write it as \(J_{\eta,T}^{\AD}\).  We use
\(J_{\reg}\) for the event-free regional continuum product,
\(D\varphi_T\) for the true flow derivative, and
\(J_{\eta,T}^{\SC}\) for the event-aware corrected sensitivity.  The
matrix \(\Phi_j\) propagates the regional variational equation between
\(t_j^+\) and \(t_{j+1}^-\), with \(t_0=0\) and \(t_{m+1}=T\).

\paragraph{Rate assumption.}
\label{ass:general-rate}
Under (A1), the regional Jacobians are uniformly continuous on the relevant
compact sets.  We write
\[
 \omega_A(r)
 =\max_\alpha\sup\left\{
  \|A_\alpha(x)-A_\alpha(y)\|:
  x,y\in K\cap\overline{\mathcal R_\alpha},\ \|x-y\|\le r
  \right\}.
\]
Thus \(\omega_A(r)\to0\) as \(r\downarrow0\).  An \(O(\eta)\) sensitivity
rate additionally requires locally Lipschitz regional Jacobians; in the
gradient setting, \(\mathcal L_\alpha\in C^{2,1}\) suffices.

\subsection{Additional scope and excluded regimes}
\label{app:data-audit}
\label{app:multidimensional}

\paragraph{Multidimensional scope.}
The sharp-interface results are fully multidimensional.  Tangential modes pass
through an interface of a continuous piecewise-smooth loss unchanged, while the
normal mode is multiplied by \((\nu^\top f^+)/(\nu^\top f^-)\).  Regional
dynamics before and after an event may rotate and couple these modes, which is
why the exact multiple-event criterion uses transported factors.  The smoothing theorem is
narrower: it extends directly only when the system admits an autonomous scalar
normal coordinate with separately controlled tangential propagation.  A
general curved layer introduces tangential drift, a varying normal,
noncommuting matrix cocycles, and flattening-curvature terms.  No general
curved-interface \(O(\eta/\tau)\) theorem is claimed.

\paragraph{Excluded regimes and rate qualifications.}
The sharp-interface and smoothing results retain the following restrictions.
\begin{enumerate}
  \item A piecewise \(C^2\) loss, hence a \(C^1\) regional field, gives the modulus estimate for ordinary AD;
  an \(O(\eta)\) sensitivity rate needs Lipschitz regional Hessians or an
  equivalent time-regularity assumption.
  \item At an exact discrete interface landing,
  \(D\Theta_{\eta,T}\) may not exist even though software AD returns a selected
  branch product.
  \item Constants are nonuniform near grazing because event-time and saltation
  formulas contain \((n^\top f^-)^{-1}\).
  \item Opposite-sign normal speeds can produce sliding or nonunique patched
  solutions.  Simultaneous hits require a separate multi-surface analysis and
  cannot be assigned an arbitrary ordered saltation product.
  \item Infinite or nonseparated event sequences invalidate the finite
  induction and finite-product perturbation arguments.
  \item If \(T\) is an event time, the fixed-time derivative may be one-sided or
  convention-dependent.
  \item Multiple nonidentity event factors can cancel after transport; a switch
  count alone is not a mismatch criterion.
  \item The identity \(\|\Xi\|_2=\max\{1,|r|\}\) requires \(d\ge2\); in one
  dimension the norm is \(|r|\).
  \item Positive-definite regional Hessians do not imply global convexity of a
  continuous nonsmooth loss.
  \item The resolved smoothing theorem includes a tail term and does not make
  \(\eta/\tau\) universal across profile families.  Its sharp-first identity
  limit is directed through nonresonant step sizes.
\end{enumerate}

\section{Regional smoothness and classical hybrid sensitivity}
\subsection{Piecewise-smooth ReLU objectives}
\label{app:relu-systems}

\begin{proposition}[Finite-sample ReLU objectives]
\label{prop:relu-objectives}
Consider a finite ReLU network on a finite data set, with a loss smooth in the
network outputs.  On every fixed activation-pattern cell, the empirical risk
is smooth (and polynomial for squared loss).  The risk is continuous across
activation boundaries.  At a point where exactly one preactivation vanishes
and its parameter gradient is nonzero, the boundary is locally a regular
hypersurface and the one-sided parameter gradients have well-defined traces.
\end{proposition}
\begin{proof}
Fixing every ReLU mask replaces each activation by either zero or its
preactivation.  The network output is then a finite composition of affine and
multilinear operations in the parameters.  It is polynomial on that mask cell,
and so is a squared empirical risk.  ReLU itself is continuous, hence the
network and risk traces agree across a mask boundary.  If a single
preactivation \(s(\theta)\) vanishes with \(\nabla s\ne0\), the implicit-function
theorem makes \(\{s=0\}\) a regular hypersurface.  Smooth regional formulas
provide the gradient traces.
\end{proof}

\begin{proposition}[Local codimension of exceptional activation events]
\label{prop:local-event-codimension}
Work inside a fixed surrounding activation pattern, where a preactivation
\(s:\mathbb R^d\to\mathbb R\) is \(C^2\) and the incoming regional field
\(f^-\) is \(C^1\).  The following statements are local.
\begin{enumerate}
 \item If \(\nabla s(z)\ne0\), then \(\Sigma=\{s=0\}\) is a regular
 hypersurface near \(z\).  The transverse subset
 \(\{x\in\Sigma:\nabla s(x)^\top f^-(x)\ne0\}\) is open in \(\Sigma\).
 \item Let \(h=(\nabla s^\top f^-)|_\Sigma\).  If zero is a regular value of
 \(h\), then the grazing set \(\{x\in\Sigma:h(x)=0\}\) has codimension one
 in \(\Sigma\), hence codimension two in \(\mathbb R^d\).
 \item If two preactivations vanish and their gradients are linearly
 independent, their simultaneous-hit set is a codimension-two submanifold.
 \item Fix an executed-branch cell of an \(N\)-step GD program and one
 evaluated state \(\theta_k^\eta(\theta_0)\).  If
 \(D_{\theta_0}(s\circ\theta_k^\eta)\ne0\) at an exact landing, then the
 initializations producing that landing form a local hypersurface.  For fixed
 \(N\) and finitely many activation surfaces, the union of all such regular
 exact-landing sets has Lebesgue measure zero inside the branch cell.
\end{enumerate}
\end{proposition}
\begin{proof}
The first statement is the implicit-function theorem, and continuity of
\(\nabla s^\top f^-\) makes its nonzero set open.  The regular-value theorem
applied to \(h:\Sigma\to\mathbb R\) gives the grazing codimension.  Applying
the implicit-function theorem to the map \((s_1,s_2)\) proves the simultaneous
hit statement when its two gradient rows are independent.  On a fixed branch
cell, every executed regional update is \(C^1\), so
\(s\circ\theta_k^\eta\) is \(C^1\).  The final hypersurface claim is again the
implicit-function theorem; a finite union of regular hypersurfaces has
Lebesgue measure zero.
\end{proof}

For canonical meshes \(\eta_N=T/N\), a useful conditional consequence is
available.  If every exact-landing map in every fixed branch cell has zero as
a regular value, then for each \(N\) the resonant initialization set is a
finite union of measure-zero hypersurfaces.  Taking the countable union over
\(N\in\mathbb N\) shows that almost every initialization is nonresonant for
all canonical meshes.  This statement is about initialization under the
displayed regular-value condition; it does not assert that nonresonant step
sizes are generic or dense for one fixed initialization.

These statements explain only local geometry under the displayed
nondegeneracy conditions.  They do not prove that a modern-network trajectory
has a globally stable finite itinerary, that grazing or simultaneous events
never occur, or that the sensitivity mismatch is prevalent.

The output is generally not piecewise affine in all network parameters.  For a
two-layer unit \(a\,\operatorname{ReLU}(w^\top x+b)\), it is bilinear on an
active cell, and the squared risk may have degree four.  Only piecewise
smoothness is used here.  For a shallow squared-loss sample, orient
\(s(\theta)=q^\top\theta+b\) from inactive to active.  At \(s=0\), with output
coefficient \(c\) and residual \(r=F_\theta(x)-y\), direct differentiation gives
\[
 [\nabla\mathcal L]=r c q
\]
up to the empirical averaging factor.  The jump is normal to the activation
surface.  A simultaneous change of several samples or units is outside the
single-interface result.

\subsection{Event-time derivative and saltation}
\label{app:event-time}

For one event, let \(z\) be the nominal impact at time \(t_*\), let
\(n=\nabla g(z)\), and let \(\zeta^-\) be the incoming perturbation evaluated at
the nominal time.  The perturbed event time \(t_*+\delta t\) obeys
\[
 0=n^\top(\zeta^-+f^-\delta t),
 \qquad
 \delta t=-\frac{n^\top\zeta^-}{n^\top f^-}.
\]
The perturbed impact point changes by \(\zeta^-+f^-\delta t\), while the
outgoing segment loses time \(\delta t\).  Returning to the nominal clock gives
\[
 \zeta^+
 =\zeta^-+(f^--f^+)\delta t
 =\left(I+\frac{(f^+-f^-)n^\top}{n^\top f^-}\right)\zeta^-.
\]
This is the saltation matrix without a state reset.  The event time is \(C^1\)
by the implicit-function theorem because \(n^\top f^-\ne0\).  Iterating this
calculation over the stable itinerary gives the hybrid product in
Theorem~\ref{thm:general-first-variation}.

\paragraph{Normal jump for a continuous loss.}
At a transverse crossing, an incoming perturbation \(\zeta^-\) shifts the
event time and fixed-clock state as follows. For a continuous piecewise-smooth
loss, tangential derivatives agree; with \(\nu=n/\|n\|\) and
\([\nabla\mathcal L]=\beta\nu\),
\begin{align}
 \delta t_*=-\frac{n^\top\zeta^-}{n^\top f^-},
 \qquad
 \zeta^+=\Xi\zeta^-,
 \qquad
 \Xi&=I+\frac{(f^+-f^-)n^\top}{n^\top f^-},                         \label{eq:saltation}\\
 \Xi=I-\frac{\beta\nu\nu^\top}{\nu^\top f^-},
 \qquad
 \lambda_\perp(\Xi)&=\frac{\nu^\top f^+}{\nu^\top f^-}.          \label{eq:gradient-saltation}
\end{align}

\section{State consistency and derivative selection}
\subsection{General state and first-variation theorem}
\label{app:general-proofs}

Put \(\eta=\eta_N=T/N\).  The executed-branch Euler/GD recursion and its
selected branch product are
\begin{align}
 \theta_{k+1}^\eta
 &=\theta_k^\eta+\eta f_{\alpha_k}(\theta_k^\eta),
 \label{eq:general-euler}\\
 J_{\eta,0}^{\mathrm{sel}}&=I,\qquad
 J_{\eta,(k+1)\eta}^{\mathrm{sel}}
 =\bigl(I+\eta A_{\alpha_k}(\theta_k^\eta)\bigr)
  J_{\eta,k\eta}^{\mathrm{sel}}.
 \label{eq:general-ad-recursion}
\end{align}
At an exact hit, \(\alpha_k\) is the admissible adjacent branch selected by
the convention in Appendix~\ref{app:assumptions}.  On the \(j\)th smooth
segment, let \(\Phi_j(t,s)\) solve
\begin{equation}
 \partial_t\Phi_j(t,s)=A_{\alpha_j}(\theta(t))\Phi_j(t,s),
 \qquad \Phi_j(s,s)=I,
 \label{eq:general-regional-fundamental}
\end{equation}
and abbreviate \(\Phi_j=\Phi_j(t_{j+1}^-,t_j^+)\), with \(t_0=0\) and
\(t_{m+1}=T\).

Let \(M_f\) and \(M_A\) bound the regional fields and Jacobians on the
compact tube.  Constants below may depend on
\(T,m,c_0^{-1},\Delta_t^{-1},M_f,M_A\), the \(C^2\) geometry and radii of
the isolating tubes, and the positive distance from the reference path to all
other interfaces outside those tubes.  They do not depend on \(N\).  This
dependence is necessarily nonuniform as \(c_0\downarrow0\).

\begin{theorem}[State consistency and first-variation selection]
\label{thm:general-first-variation}
Under (A1)--(A5), there are \(C<\infty\) and \(N_0\) such that, for every
\(N\ge N_0\),
\begin{equation}
 \sup_{0\le t\le T}
 \|\Theta_{\eta_N}(t;\theta_0)-\varphi_t(\theta_0)\|
 \le C\eta_N
 \label{eq:general-state-rate}
\end{equation}
and
\begin{equation}
 \|J_{\eta_N,T}^{\mathrm{sel}}-J_{\reg}\|
 \le C\bigl(\eta_N+\omega_A(C\eta_N)\bigr)=o(1),
 \qquad J_{\reg}:=\Phi_m\cdots\Phi_0.
 \label{eq:general-ad-modulus-rate}
\end{equation}
If the regional Jacobians are locally Lipschitz, the right-hand side is
\(O(\eta_N)\).

The terminal flow map is \(C^1\) on a neighborhood with the same itinerary,
and
\begin{equation}
 D\varphi_T(\theta_0)
 =\Phi_m\Xi_m\Phi_{m-1}\Xi_{m-1}\cdots
  \Phi_1\Xi_1\Phi_0,
 \qquad
 \Xi_j=I+\frac{(f_j^+-f_j^-)n_j^\top}{n_j^\top f_j^-}.
 \label{eq:general-saltation-product}
\end{equation}
For every nonresonant subsequence \(N_\ell\to\infty\), the endpoint map is
classically differentiable near \(\theta_0\) and
\begin{equation}
 D_{\theta_0}\Theta_{\eta_{N_\ell},T}(\theta_0)
 =J_{\eta_{N_\ell},T}^{\AD}
 =J_{\eta_{N_\ell},T}^{\mathrm{sel}}
 \longrightarrow J_{\reg}.
 \label{eq:general-ad-is-discrete-derivative}
\end{equation}
At a resonant mesh, the selected product still satisfies
\eqref{eq:general-ad-modulus-rate}, but need not be a classical derivative.
\end{theorem}

The proof is split into six lemmas to expose the event-tube induction and the
matrix limit.

\begin{lemma}[Smooth-segment Euler stability]
\label{lem:general-smooth-segment}
Suppose the exact and numerical paths use the same regional extension on grid
indices \(k_0,\ldots,k_1\).  With
\(e_k=\|\theta_k^\eta-\theta(k\eta)\|\),
\begin{equation}
 \max_{k_0\le k\le k_1}e_k
 \le e^{M_A(k_1-k_0)\eta}
 \left[e_{k_0}+\tfrac12M_AM_f(k_1-k_0)\eta^2\right].
 \label{eq:general-smooth-segment-bound}
\end{equation}
The corresponding piecewise-linear interpolation has an additional
\(O(M_f\eta)\) error.
\end{lemma}
\begin{proof}
The exact regional solution has local defect at most
\(M_AM_f\eta^2/2\).  Lipschitz continuity of the field gives
\[
 e_{k+1}\le(1+M_A\eta)e_k+\tfrac12M_AM_f\eta^2.
\]
Discrete Gronwall and \(1+x\le e^x\) prove
\eqref{eq:general-smooth-segment-bound}.  Both the exact path and the
interpolant move at uniformly bounded speed within a step.
\end{proof}

\begin{lemma}[Single-event localization]
\label{lem:general-single-event}
Consider the isolating tube of event \(j\), and suppose the numerical path
reaches its incoming part with error \(e_{\rm in}\).  For sufficiently small
\(\eta\), let \(k_{j,\rm in}^\eta\) be its first grid index in that incoming
part.  The first nonnegative interface index
\[
 \kappa_j^\eta:=\min\{k\ge k_{j,\rm in}^\eta:
                    g_j(\theta_k^\eta)\ge0\}
\]
exists and is unique inside the tube, the numerical path does not recross
there, and
\begin{equation}
 |\kappa_j^\eta\eta-t_j|
 +\|\theta_{\kappa_j^\eta}^\eta-z_j\|
 \le C(e_{\rm in}+\eta).
 \label{eq:general-event-localization}
\end{equation}
If an exact landing executes the incoming branch once at the zero iterate,
the executed branch-switch index is \(\kappa_j^\eta+1\); otherwise it is
\(\kappa_j^\eta\).  Thus its clock obeys the same estimate after increasing
the constant by one mesh unit.
In particular, an \(O(\eta)\) incoming error gives first-order event
localization.
\end{lemma}
\begin{proof}
Uniformly on either one-sided part of the tube, Taylor expansion gives
\begin{equation}
 g_j(x+\eta f_j^\pm(x))
 =g_j(x)+\eta\nabla g_j(x)^\top f_j^\pm(x)+O(\eta^2).
 \label{eq:general-discrete-transverse-progress}
\end{equation}
The remainder is uniform because \(g_j\) is \(C^2\) and \(f_j^\pm\) is
bounded.  Shrink the isolating tube so that its compact incoming and outgoing
cross-sections satisfy \(g_j\le-\gamma_j\) and \(g_j\ge\gamma_j\), respectively,
for some \(\gamma_j>0\).  By (A4), the numerical increment of \(g_j\) is at
least \(c_0\eta/2\) for all small \(\eta\), including an exact hit under the
admissible convention.  Starting on the negative cross-section, the discrete
normal coordinate therefore reaches zero before the positive cross-section,
does so at most once, and cannot recross before leaving the tube.

Up to the earlier of the exact and numerical crossing clocks, both paths use
the incoming regional extension.  Lemma~\ref{lem:general-smooth-segment} and
the smoothness of \(g_j\) give an \(O(e_{\rm in}+\eta)\) normal-coordinate
error there.  Between the two clocks, either one-sided field has bounded
speed, while both one-sided normal velocities are at least \(c_0\).  Comparing
with the corresponding incoming or outgoing regional continuation therefore
bounds the clock difference by \(C(e_{\rm in}+\eta)/c_0\).  Bounded state speed
gives \eqref{eq:general-event-localization}.  The possible extra incoming
update at an exact zero delays the executed switch by at most one step.
\end{proof}

\begin{lemma}[Finite-itinerary induction]
\label{lem:general-itinerary}
For all sufficiently small \(\eta\), the numerical path visits the same
regions and isolating tubes as the reference path, in the same order, and has
exactly one branch transition in each tube.  Its error on leaving every tube
is \(O(\eta)\), with a constant uniform over the finite itinerary.
\end{lemma}
\begin{proof}
Before the first tube, Lemma~\ref{lem:general-smooth-segment} gives
\(O(\eta)\) error.  Lemma~\ref{lem:general-single-event} localizes the
transition.  Substituting this incoming \(O(\eta)\) error into
\eqref{eq:general-event-localization} shows that the switch clocks differ by
\(O(\eta)\), hence by only \(O(1)\) grid steps.  Each such step contributes at
most \(2M_f\eta\) between the two one-sided updates; thus the outgoing restart
error remains \(O(\eta)\).  The positive distance from all other
interfaces fixes the next region.  Restart Lemma~\ref{lem:general-smooth-segment}
on the outgoing segment.  The \(\Delta_t\)-separated tubes and finite event
count allow this argument to be iterated; a finite composition of its
Gronwall constants is still independent of \(N\).  The final segment and
the interpolation estimate prove \eqref{eq:general-state-rate}.
\end{proof}

\begin{lemma}[Branch-mismatch set]
\label{lem:general-branch-mismatch}
Let \(\alpha(t)\) be the exact branch and let
\(\alpha_\eta(t)=\alpha_k\) on \([k\eta,(k+1)\eta)\).  Define
\[
 E_\eta=\{t\in[0,T]:\alpha_\eta(t)\ne\alpha(t)\}.
\]
Then \(|E_\eta|\le Cm\eta\).  For
\[
 A(t)=A_{\alpha(t)}(\theta(t)),\qquad
 A^\eta(t)=A_{\alpha_k}(\theta_k^\eta),
 \quad t\in[k\eta,(k+1)\eta),
\]
we have
\begin{equation}
 \|A^\eta-A\|_{L^1(0,T)}
 \le C\eta+T\omega_A(C\eta).
 \label{eq:general-coefficient-l1}
\end{equation}
\end{lemma}
\begin{proof}
Lemma~\ref{lem:general-single-event} confines a branch disagreement at event
\(j\) to an \(O(\eta)\) interval around \(t_j\); the intervals are
disjoint for small \(\eta\).  Their total length is at most \(Cm\eta\).
Outside them, both paths use the same regional extension, and
Lemma~\ref{lem:general-itinerary} plus within-step motion gives
\(\|A^\eta(t)-A(t)\|\le\omega_A(C\eta)\).  Inside them, both
coefficients have norm at most \(M_A\).  Integration proves
\eqref{eq:general-coefficient-l1}.
\end{proof}

\begin{lemma}[Fundamental-matrix stability]
\label{lem:general-fundamental}
Let \(U_\eta\) and \(U\) be the fundamental matrices for the stepwise
coefficient \(A^\eta\) and the exact coefficient \(A\), respectively.  Then
\begin{align}
 \|U_\eta(T)-U(T)\|
 &\le e^{2M_AT}\|A^\eta-A\|_{L^1(0,T)},
 \label{eq:general-fundamental-l1}\\
 \|J_{\eta,T}^{\mathrm{sel}}-U_\eta(T)\|
 &\le C\eta.
 \label{eq:general-euler-cocycle}
\end{align}
\end{lemma}
\begin{proof}
Variation of constants gives
\[
 U_\eta(T)-U(T)
 =\int_0^T U_\eta(T,s)(A^\eta(s)-A(s))U(s)\,\dd s.
\]
Both propagators have norm at most \(e^{M_AT}\), which proves
\eqref{eq:general-fundamental-l1}.  On one step, the exact cocycle of
\(A^\eta\) is \(e^{\eta A_k}\), and
\[
 \|e^{\eta A_k}-(I+\eta A_k)\|
 \le\tfrac12\eta^2M_A^2e^{\eta M_A}.
\]
A telescoping product estimate over \(N=T/\eta\) uniformly bounded factors
proves \eqref{eq:general-euler-cocycle}.
\end{proof}

\begin{lemma}[Flow-map differentiability and saltation]
\label{lem:general-flow-differentiability}
On a neighborhood with the stable itinerary, every event time and event state
is \(C^1\).  If \(\psi_t^-\) denotes the incoming regional flow and
\(\tau_j(x)\) its hitting time, then
\begin{equation}
 D\tau_j(x)
 =-\frac{n_j^\top D\psi^-_{\tau_j(x)}(x)}
         {n_j^\top f_j^-}.
 \label{eq:general-hitting-time-derivative}
\end{equation}
The fixed-clock derivative across the event is \(\Xi_j\), and finite
chronological composition gives \eqref{eq:general-saltation-product}.
\end{lemma}
\begin{proof}
The hitting equation is
\(g_j(\psi^-_{\tau_j(x)}(x))=0\).  Its derivative with respect to time is
\(n_j^\top f_j^-\ne0\), so the implicit-function theorem proves the first
claim and \eqref{eq:general-hitting-time-derivative}.  An incoming
perturbation \(\zeta^-\) changes the event time by
\[
 \delta t_j=-\frac{n_j^\top\zeta^-}{n_j^\top f_j^-}.
\]
Returning the outgoing perturbation to the nominal clock gives
\[
 \zeta^+
 =\zeta^-+(f_j^--f_j^+)\delta t_j
 =\left(I+\frac{(f_j^+-f_j^-)n_j^\top}
                  {n_j^\top f_j^-}\right)\zeta^-.
\]
Regional flow maps and event maps are \(C^1\); composing the finitely many
maps in chronological order proves \eqref{eq:general-saltation-product}.
\end{proof}

\begin{proof}[Proof of Theorem~\ref{thm:general-first-variation}]
Lemma~\ref{lem:general-itinerary} proves the state estimate.
Lemmas~\ref{lem:general-branch-mismatch} and
\ref{lem:general-fundamental} give
\[
 \|J_{\eta,T}^{\mathrm{sel}}-J_{\reg}\|
 \le C\{\eta+\|A^\eta-A\|_{L^1}\}
 \le C\bigl(\eta+\omega_A(C\eta)\bigr).
\]
Lemma~\ref{lem:general-flow-differentiability} proves the flow derivative.
If the mesh is nonresonant, every evaluated regional sign is strict and
persists in a neighborhood of \(\theta_0\).  The endpoint is then a finite
composition of \(C^1\) regional Euler maps, and the chain rule identifies its
classical derivative with the selected branch product.  At an exact landing,
this neighborhood argument can fail, while the already proved product
convergence remains valid.
\end{proof}

For completeness, applying the reverse triangle inequality to
\(J_{\eta,T}^{\mathrm{sel}}-D\varphi_T
=(J_{\reg}-D\varphi_T)+(J_{\eta,T}^{\mathrm{sel}}-J_{\reg})\) and then using
\eqref{eq:general-ad-modulus-rate} gives
\eqref{eq:finite-mesh-gap-norm-convergence}.  This statement concerns the
selected product on every canonical mesh.  It becomes a statement about the
classical derivative \(D\Theta_{\eta,T}\) only along nonresonant meshes, by
\eqref{eq:general-ad-is-discrete-derivative}.

In the continuous-loss gradient specialization,
\([\nabla\mathcal L]_j=\beta_j\nu_j\) implies
\begin{equation}
 \Xi_j=I-
 \frac{\beta_j\nu_j\nu_j^\top}{\nu_j^\top f_j^-},
 \qquad
 r_j=\frac{\nu_j^\top f_j^+}{\nu_j^\top f_j^-}.
 \label{eq:general-gradient-saltation}
\end{equation}
The matrix \(\Xi_j-I\) is symmetric and has rank at most one.
For \(\Xi_j\), tangential directions have eigenvalue one, and the normal
direction has eigenvalue \(r_j\). For \(d\ge2\),
\(\|\Xi_j\|_2=\max\{1,|r_j|\}\); for \(d=1\),
\(\|\Xi_j\|_2=|r_j|\).

\subsection{Supporting derivative identities}
The finite-mesh comparison separates into an asymptotic event gap and a
discretization remainder:
\begin{equation}
 J_{\eta,T}^{\mathrm{sel}}-D\varphi_T
 =\bigl(J_{\reg}-D\varphi_T\bigr)
  +\bigl(J_{\eta,T}^{\mathrm{sel}}-J_{\reg}\bigr).
 \label{eq:finite-mesh-gap-decomposition}
\end{equation}
Consequently, the reverse triangle inequality gives the two-sided spectral-norm
estimate
\begin{equation}
 \left|
  \|J_{\eta,T}^{\mathrm{sel}}-D\varphi_T\|_2
  -\|J_{\reg}-D\varphi_T\|_2
 \right|
 \le C\{\eta+\omega_A(C\eta)\}.
 \label{eq:finite-mesh-gap-norm-convergence}
\end{equation}
In particular,
\(\|J_{\eta,T}^{\mathrm{sel}}-D\varphi_T\|_2
\to\|J_{\reg}-D\varphi_T\|_2\).  Along a nonresonant subsequence the first
term is exactly \(\|D\Theta_{\eta,T}-D\varphi_T\|_2\).  Thus any fixed nonzero
event gap remains visible on sufficiently fine meshes;
the constants deteriorate near grazing through their dependence on
\(c_0^{-1}\).  This is a trajectory-local visibility statement, not a
uniform lower bound over changing event geometries.

\paragraph{The selection is not specific to Euler.}
Theorem~\ref{thm:event-blind-one-step} proves the same event-free derivative
limit for a class of fixed-step, event-blind one-step schemes.  It assumes
regional consistency, mixed-event Jacobians \(I+O(\eta)\), a uniformly bounded
number of mixed steps per event, and oriented progress through each event
tube.  Higher regional order alone therefore does not recover saltation; the
claim does not cover a Runge--Kutta, adaptive, implicit, or event-localizing
method unless these hypotheses are verified for that method.

\subsection{A restricted event-blind one-step extension}
\label{app:one-step}

The Euler calculation is not the only way to obtain the event-free derivative
limit.  Regional consistency alone, however, is insufficient: it places no
restriction on the derivative of a step that samples more than one region.
The following result states the additional hypotheses needed for a safe
extension.

\begin{theorem}[Conditional event-blind one-step derivative criterion]
\label{thm:event-blind-one-step}
Assume (A1)--(A5).  Consider a fixed-step method whose executed update is a
piecewise-\(C^1\) map \(x_{k+1}=\Psi_{\eta,\sigma_k}(x_k)\), where
\(\sigma_k\) records a locally fixed finite pattern of regional evaluations.
Suppose the following bounds hold uniformly on the relevant compact tube.

\begin{enumerate}
 \item On a step whose evaluations remain in one region \(\alpha\),
 \begin{align}
  \Psi_{\eta,\alpha}(x)
  &=x+\eta f_\alpha(x)+O(\eta^2),
  \label{eq:one-step-regional-state}\\
  D\Psi_{\eta,\alpha}(x)
  &=I+\eta A_\alpha(x)+O(\eta^2).
  \label{eq:one-step-regional-jacobian}
 \end{align}
 \item Every mixed step in an isolating event tube satisfies
 \begin{equation}
  \|\Psi_{\eta,\sigma}(x)-x\|\le C\eta,
  \qquad
  \|D\Psi_{\eta,\sigma}(x)-I\|\le C\eta.
  \label{eq:one-step-mixed-bound}
 \end{equation}
 \item The method makes oriented interface progress at least \(c\eta\) in
 each event tube and executes at most \(M\) mixed steps per event, with
 \(c>0\) and \(M<\infty\) independent of \(\eta\).
 \item No evaluation used by the executed program lies exactly on an
 interface.
\end{enumerate}

Then the numerical state converges uniformly on \([0,T]\), and the exact
Jacobian of the executed program satisfies
\begin{equation}
 D\Theta_{\eta,T}^{\Psi}\longrightarrow J_{\reg}.
 \label{eq:one-step-derivative-limit}
\end{equation}
The derivative of the limiting flow remains the saltation-interleaved product
in Theorem~\ref{thm:main}.  Thus the two derivatives differ whenever the
transported event product is nonidentity.
\end{theorem}

\begin{proof}
On pure regional steps, \eqref{eq:one-step-regional-state} gives the standard
first-order consistency and stability recursion.  In each isolated event
tube, positive progress and the mixed-step bound imply an \(O(\eta)\) timing
error and only \(O(1)\) local defects of size \(O(\eta)\).  Restarting the
regional recursion after every separated event gives uniform state
convergence by the same finite induction as
Theorem~\ref{thm:general-first-variation}.

For the derivative, pure regional factors have the expansion
\(I+\eta A_\alpha+O(\eta^2)\).  Their ordered products converge to the
corresponding regional fundamental matrices.  Every mixed factor is
\(I+O(\eta)\), and only finitely many such factors occur, so their combined
product converges to \(I\), not to a nonidentity saltation matrix.  A finite
product perturbation argument across the separated itinerary proves
\eqref{eq:one-step-derivative-limit}.  Strict branch evaluations make this
product the classical derivative of the executed finite program.
\end{proof}

The theorem includes Euler/GD.  It also includes any event-blind regional
one-step construction that satisfies the mixed-step hypotheses, irrespective
of its higher regional order.  It does not cover an event-localizing method,
whose mixed-step derivative is designed to retain the event-time transfer,
or an adaptive, implicit, or stage-switching scheme unless those hypotheses
are verified separately.  Smooth preconditioned Euler is covered relative to
its own limiting field \(f=-P(\theta)\nabla\mathcal L(\theta)\), provided the
same event and regularity assumptions hold; this is not a statement about the
unpreconditioned gradient flow.  Fixed-coefficient momentum and heavy-ball
limits are not covered.

\section{Finite-event critical-scale branch law}
\label{app:finite-smooth-response}

This result concerns a joint limit of the step size and the initial
perturbation. Event-wise propagation itself is standard in perturbation
analysis. The statement below additionally derives the integer indices
selected by explicit Euler, including accumulated Euler bias, and proves
a uniform endpoint expansion away from the resulting scaled boundaries.
No independence, density, or equidistribution of event phases is used.

\subsection{Geometric and regularity hypotheses}
Fix $T>0$ and a finite number $m\geq0$ of events. A finite collection of
$C^2$ functions $g_s$, $1\leq s\leq S$, defines the switching surfaces.
The executed field is selected by their signs. Every regional field used
below has a $C^2$ extension to an open neighborhood of the relevant path
and its event tubes. The Euler program evaluates the field at the current
grid state and then takes one explicit Euler step. At a zero sign it may
use any fixed convention; our conclusion concerns grids with strict signs.
These are $C^2$ assumptions on the regional vector fields, not just the
losses: when $f_i=-\nabla L_i$, regional $C^3$ loss extensions suffice.

The base continuous trajectory $x(t)$ is continuous and follows the
successive regional fields $f_0,\ldots,f_m$ on the intervals separated by
\[
 0=t_0<t_1<\cdots<t_m<t_{m+1}=T,
 \qquad L_i=t_{i+1}-t_i,\qquad x_i=x(t_i).
\]
Its events are isolated single-surface crossings. More precisely, assume
the following quantitative neighborhood geometry.
\begin{enumerate}
\item $L_i\geq\Delta>0$. There is $a\in(0,\Delta/4)$ such that the
time windows $[t_j-a,t_j+a]$ are disjoint. Within a neighborhood $U_j$
of each such path arc, just one surface can vanish, and its two sign
regions select exactly $f_{j-1}$ and $f_j$. Orient its local defining
function as $g_j^*=\sigma_jg_{\ell_j}$, $\sigma_j\in\{-1,1\}$, so that
the crossing is from negative to positive. Thus the same $G_1,G_2$
bounds apply to the oriented functions. The other defining functions have absolute values at least
$\gamma>0$ on $U_j$.
\item Both adjacent smooth extensions satisfy
\[
 \nabla g_j^*(y)^\top f_{j-1}(y)\geq\nu>0,
 \qquad \nabla g_j^*(y)^\top f_j(y)\geq\nu>0
 \quad(y\in U_j).
\]
Outside these time windows, every surface has absolute value at least
$\gamma$ on the base path. These base arcs and event arcs have a common
positive spatial tube radius $\rho$ within the indicated neighborhoods.
The same physical surface may occur at different, separated events;
its orientation is chosen separately at each occurrence.
\item On a closed neighborhood containing these tubes, the short smooth
extensions at their endpoints, and straight segments used in Taylor
estimates, there are bounds
\[
 \|f_i\|\leq F_0,\quad\|Df_i\|\leq F_1,
 \quad\|D^2f_i\|\leq F_2,\qquad
 \|\nabla g_s\|\leq G_1,\quad\|\nabla^2g_s\|\leq G_2.
\]
The indicated neighborhood is fixed independently of the mesh.
\end{enumerate}
For example, isolated transverse events of a fixed finite base path give
these constants after shrinking the windows and tubes, provided the
outgoing normal speed has the same strictly positive sign. Merely having
a nonzero incoming speed is not sufficient: the outgoing condition
excludes immediate return, sticking, or sliding. No discrete crossing
index, discrete event order, or perturbed itinerary is assumed.

On segment $i$, write $\Phi_i(t,s)$ for the ordinary smooth variational
propagator along its base path, and define $P_i=\Phi_i(t_{i+1},t_i)$.
Let $w_i(t)$ solve
\begin{equation}
 \dot w_i=Df_i(x(t))w_i-\tfrac12Df_i(x(t))f_i(x(t)),
 \qquad w_i(t_i)=0,
 \qquad w_i=w_i(t_{i+1}).
 \label{eq:smooth-euler-bias}
\end{equation}
Also set
\[
 n_j=\nabla g_j^*(x_j),\qquad
 u_j=n_j^\top f_{j-1}(x_j)\geq\nu,
 \qquad \Delta_j=f_j(x_j)-f_{j-1}(x_j).
\]
These are base-path quantities. In particular, $P_i$ contains no
saltation matrix, and $w_i$ is the smooth-region Euler bias.

\subsection{Recursive branch law and theorem}
For $p\in\mathbb T^m$ and $z\in\mathbb R^d$, put $b_0=z$ and recurse
chronologically:
\begin{align}
 a_j&=P_{j-1}b_{j-1}+w_{j-1},\nonumber\\
 c_j(p,z)&=p_j-n_j^\top a_j/u_j,\nonumber\\
 s_j&=\lceil c_j(p,z)\rceil-p_j,
 \qquad b_j=a_j-\Delta_js_j,
 \qquad 1\leq j\leq m,\nonumber\\
 e_p(z)&=P_mb_m+w_m.
 \label{eq:finite-smooth-recursion}
\end{align}
Integer changes in a lift of $p_j$ leave $s_j$ unchanged. At an integer
$c_j$ the formula provides a candidate value but the following theorem
does not identify the actual branch. For $m=0$ the recursion is simply
$e(z)=P_0z+w_0$.

\paragraph{Detailed statement of Theorem~\ref{thm:finite-smooth-response}.}
Under the preceding hypotheses, let $K\subset\mathbb R^d$ be compact,
$Z=\sup_{z\in K}\|z\|<\infty$, and $0<\mu<1/2$.
There are finite constants $C,B$, depending only on
\[
 (m,T,Z,F_0,F_1,F_2,G_1,G_2,\nu^{-1},\Delta^{-1},
 a^{-1},\rho^{-1},\gamma^{-1}),
\]
and a sufficient mesh threshold $\eta_*>0$, also depending on $\mu$,
such that the following holds. For every integer $N$ with
$\eta=T/N\leq\eta_*$, define
$p_N=(\{t_1/\eta\},\ldots,\{t_m/\eta\})$. If
\begin{equation}
 \inf_{z\in K}\min_{1\leq j\leq m}
 \operatorname{dist}(c_j(p_N,z),\mathbb Z)\geq\mu,
 \label{eq:finite-smooth-margin}
\end{equation}
then every perturbed Euler path starting at $x_0+\eta z$, $z\in K$,
has exactly the base sequence of $m$ events, with first positive indices
\begin{equation}
 k_j(z)=\left\lceil\frac{t_j}{\eta}
                    -\frac{n_j^\top a_j(p_N,z)}{u_j}\right\rceil.
 \label{eq:finite-smooth-indices}
\end{equation}
All its executed signs are strict. Uniformly on $K$,
\begin{equation}
 \left\|\Theta_{\eta,T}(x_0+\eta z)-x_{m+1}
                          -\eta e_{p_N}(z)\right\|\leq C\eta^2,
 \qquad \sup_{z\in K}\|e_{p_N}(z)\|\leq B.
 \label{eq:finite-smooth-expansion}
\end{equation}
For $m=0$ the phase condition is vacuous and there are no event indices.
The compact set $K$ need not be connected: finite candidate sets may
occupy different ceiling cells while satisfying \eqref{eq:finite-smooth-margin}
at every evaluated point. A connected set spanning a ceiling transition
cannot satisfy a positive uniform margin. The arguments $c_j$ are affine
within preceding ceiling cells and may jump across upstream boundaries.
The proof below imposes $\eta J_j<\nu\mu/2$ for every event, among other
sufficient restrictions; it asserts no universal scaling of $\eta_*$ in $\mu$.

If $0\in K$, $q\in C^2$ near $x_{m+1}$, and
$\mathcal R_\eta(z)=[q(\Theta_{\eta,T}(x_0+\eta z))
-q(\Theta_{\eta,T}(x_0))]/\eta$, then uniformly on $K$,
\begin{equation}
 \mathcal R_\eta(z)=\nabla q(x_{m+1})^\top
       [e_{p_N}(z)-e_{p_N}(0)]+O_K(\eta).
 \label{eq:finite-smooth-objective}
\end{equation}
If $p_N\to p$ along a subsequence and every $c_j(p,z)$ on $K$ has
distance at least $\mu$ from the integers, the same conclusion holds on
that subsequence for all sufficiently large $N$, with
$e_{p_N}$ replaced by $e_p$ and error
$O_K(\eta+\operatorname{dist}_{\mathbb T^m}(p_N,p))$.

\subsection{A restart-aware smooth Euler estimate}
We first prove the regional estimate, including its bias sign. Let
$x'=f(x)$ on a smooth reference segment. For an initial fixed-clock
error $b$, let
\[
 E'=Df(x)E-\tfrac12Df(x)f(x),\qquad E(t_i)=b.
\]
Thus $E(t_{i+1})=P_ib+w_i$. Define the comparison curve
$y(t)=x(t)+\eta E(t)$. Since
$x''=Df(x)f(x)$, its one-step residual is
\begin{align*}
 y(t+\eta)-y(t)-\eta f(y(t))
 &=\eta^2[\tfrac12Df(x)f(x)+E'-Df(x)E]+O(\eta^3)\\
 &=O(\eta^3).
\end{align*}
This proves the negative sign in \eqref{eq:smooth-euler-bias}.
Only $C^2$ regularity is needed: $x'''=D^2f(x)[f,f]+Df(x)Df(x)f$
is bounded, and differentiating the equation for $E$ once uses only
$D^2f$. At an early grid restart, $t_i+\eta\sigma<t_i$, both $x$ and
$E$ mean their hypothetical smooth continuation under the outgoing
regional extension; no claim is made that the switched base flow follows
that extension before $t_i$. The fixed neighborhoods permit these short
backward continuations. The bounds below use absolute elapsed time and
therefore cover them as well as forward continuations. For explicit bounds, on a time interval of length at most
$D=T+2$ put
\begin{align*}
 \Gamma&=e^{F_1D}, & X_3&=F_2F_0^2+F_1^2F_0,\\
 V(U)&=\Gamma(U+\tfrac12F_1F_0D),&
 V'(U)&=F_1V(U)+\tfrac12F_1F_0,\\
 V''(U)&=F_2F_0V(U)+F_1V'(U)+\tfrac12X_3,&
 A(U)&=X_3/6+V''(U)/2+F_2V(U)^2/2,\\
 H_U(s)&=\tfrac12F_1F_0s^2+V'(U)|s|.&&
\end{align*}
For $\|b\|\leq U$, the one-step residual is at most $A(U)\eta^3$.
This follows from Taylor's integral remainders for $x$, $E$, and $f$;
the last contributes at most $F_2\eta^3V(U)^2/2$.

Suppose a regional Euler block starts at the grid clock
$t_i+\eta\sigma$ with state
\[
 x_i+\eta\{b+f_i(x_i)\sigma\}+r,
 \qquad\|r\|\leq C_i\eta^2.
\]
Taylor's formula gives an initial discrepancy from $y(t_i+\eta\sigma)$
of at most $\eta^2[C_i+H_U(\sigma)]$. A discrete Gronwall estimate for
the one-step residual, in the actual order of Euler steps, then gives
at the grid clock $t_{i+1}+\eta s$ the expansion
\begin{equation}
 x_{i+1}+\eta\{P_ib+w_i+f_i(x_{i+1})s\}+r',
 \quad
 \|r'\|\leq\eta^2\{
 \Gamma[C_i+H_U(\sigma)+D A(U)]+H_U(s)\}.
 \label{eq:restart-smooth-lemma}
\end{equation}
Indeed, each discrepancy grows by at most $1+\eta F_1$ and there are
at most $D/\eta$ residuals. This argument also holds at every intermediate
grid clock, not just at the endpoint. Short extensions of length
$\eta|\sigma|$ or $\eta|s|$ are allowed inside the fixed neighborhoods.
A stopped-neighborhood argument closes the estimate: derive it up to the
first possible exit, and choose the mesh so its error is smaller than
the available tube radius; that exit is then impossible.

\subsection{Proof of the theorem and explicit dependency recursion}
\begin{proof}
First bound the proposed offsets without choosing any mesh or branch.
Set $U_0=Z$. For $j=1,\ldots,m$ recursively define
\begin{align*}
 A_j&=e^{F_1L_{j-1}}U_{j-1}
       +\tfrac12F_1F_0L_{j-1}e^{F_1L_{j-1}},\\
 S_j&=1+G_1A_j/\nu,\qquad
 U_j=A_j+2F_0S_j.
\end{align*}
Then $\|a_j\|\leq A_j$, $|s_j|\leq S_j$, and $\|b_j\|\leq U_j$
for every phase and $z\in K$. These follow from
$|\lceil p-\xi\rceil-p|\leq1+|\xi|$, the variational bound
$\|P_i\|\leq e^{F_1L_i}$, the corresponding integral bound for $w_i$,
and $\|\Delta_j\|\leq2F_0$. Thus
$B=e^{F_1L_m}(U_m+F_1F_0L_m/2)$ suffices for the endpoint coefficient.
Put $S_0=S_{m+1}=0$, $C_0=0$, and set
\begin{equation}
 C_{i+1}=\Gamma[C_i+H_{U_i}(S_i)+D A(U_i)]
              +H_{U_i}(S_{i+1}+\mathbf1_{i<m}),
 \qquad 0\leq i\leq m.
 \label{eq:finite-smooth-constants}
\end{equation}
The extra one covers the grid point just before a proposed crossing.
This recursion uses only the displayed data and bounds every remainder
needed in \eqref{eq:restart-smooth-lemma}. In particular $C=C_{m+1}$
is an admissible endpoint constant.

We now identify the actual first event. The smooth Euler extension of
$f_0$ at a grid point $k\eta=t_1+\eta s$ has the expansion
\[
 x_1+\eta[a_1+f_0(x_1)s]+r,\qquad\|r\|\leq C_1\eta^2.
\]
Applying the nonlinear surface and using $g_1^*(x_1)=0$ gives
\begin{equation}
 g_1^*(\theta_k)=\eta u_1(s+n_1^\top a_1/u_1)
                                      +O(\eta^2).
 \label{eq:nonlinear-surface-expansion}
\end{equation}
For $|s|\leq S_1+1$ and $\eta C_1\leq1$, the error in this equation
is at most $J_1\eta^2$, where generally
\[
 J_j=G_1C_j+\tfrac12G_2[A_j+F_0(S_j+1)+1]^2.
\]
Thus curvature enters only the $O(\eta^2)$ remainder. It produces no
additional term in the first-order ceiling argument.

At the integer $k_1$ proposed by \eqref{eq:finite-smooth-indices} the
leading surface value is at least $\eta u_1\mu$; at $k_1-1$ it is at
most $-\eta u_1\mu$. For $\eta J_1<\nu\mu/2$ these are strictly
positive and negative, respectively. Outside the first event window,
the smooth regional estimate and surface margin exclude any crossing.
Inside its tube, for either adjacent regional extension,
\[
 g_1^*(y+\eta f(y))-g_1^*(y)
 \geq\eta\nu-\tfrac12\eta^2G_2F_0^2\geq\eta\nu/2
\]
when $\eta G_2F_0^2\leq\nu$. Consequently no earlier zero or positive
grid state can precede the negative value at $k_1-1$: $k_1$ is exactly
the first positive index. The other surfaces cannot change in this tube.
The outgoing monotonicity excludes a return across the same surface.

At this actual crossing state the expansion can be rewritten exactly at
first order as
\[
 x_1+\eta[a_1+f_0(x_1)s_1]+r
 =x_1+\eta[b_1+f_1(x_1)s_1]+r,
 \qquad b_1=a_1-\Delta_1s_1.
\]
This is the restart error at the fixed reference clock; it is not an
independent continuously shifted event time. Apply
\eqref{eq:restart-smooth-lemma} in region $1$, starting at the actual
grid clock $k_1\eta$. It gives the incoming error
$a_2=P_1b_1+w_1$, including the first index's rounding error.
The same two-sign and monotonicity argument at the second surface
derives $k_2$. Induction derives every $k_j$, its restart $b_j$, and
the expansion on the next smooth block. No commutation of variational
propagators is used.

Here are explicit sufficient types of mesh restrictions completing that
induction. In addition to $\eta\leq1$, require
\[
 \eta\max_j(S_j+1)<\min\{a/4,1\},
 \quad \eta\max_j C_j\leq1,
 \quad \eta\max_j J_j<\nu\mu/2,
 \quad \eta G_2F_0^2\leq\nu.
\]
The first puts all candidate crossings inside their separated windows,
so $k_1<\cdots<k_m$ and $0<k_1$, $k_m<N$.
The regional lemma bounds the discrepancy from the base path throughout
all blocks by $\eta B_{\rm tube}$ for a finite constant depending on
the same recursion; for example, after decreasing the mesh to make the
second-order errors at most $\eta$, one may take
\[
 B_{\rm tube}=1+\max_iV(U_i)+2F_0\max_j(S_j+1).
\]
Choose additionally
$\eta B_{\rm tube}<\rho/4$,
$\eta F_0<\rho/4$, and
$\eta G_1B_{\rm tube}<\gamma/2$.
The tube bounds keep every relevant state and Euler step in the stipulated
neighborhoods. The surface bounds exclude extra events in the intervening
corridors and changes of other surfaces in event tubes. The intermediate
regional estimate invoked here is proved on the candidate smooth
extension before assuming its signs, so this is not circular.
There are only finitely many restrictions, all computable from the
displayed data, and they give a positive $\eta_*$.
The final block has no further ceiling and gives
\eqref{eq:finite-smooth-expansion} with $C_{m+1}$.

For the objective statement, take a ball of radius $r_q>0$ about
$x_{m+1}$ where $q$ is $C^2$ and $\|\nabla^2q\|\leq M_q$.
Decrease the objective's mesh threshold, if needed, so that
$\eta B+\eta^2C<r_q$. This additional threshold does not change the
state constants or its geometric hypotheses. If
$\|\nabla q(x_{m+1})\|\leq G_q$, Taylor's formula gives
an error in \eqref{eq:finite-smooth-objective} bounded by
\[
 2\eta\{G_q C+\tfrac12M_q(B+\eta C)^2\}.
\]
Finally, the ceiling arguments are affine in $(p,z)$ on each preceding
ceiling cell. A uniform positive margin preserves all their integer
values under a sufficiently small change of a consistent local phase
lift, by induction over $j$. The coefficients are bounded by the same
propagator and jump bounds. Therefore $e_p$ is uniformly Lipschitz in
phase on these neighborhoods, proving the subsequence conclusion.
\end{proof}

\subsection{Downstream boundaries, reductions, and excluded limits}
For $j>r$, write
$P_{j-1:r}=P_{j-1}P_{j-2}\cdots P_r$ in chronological order.
Unrolling the recursion gives
\[
 a_j=A_j^{\rm free}(z)-\sum_{r<j}P_{j-1:r}\Delta_rs_r,
\]
where $A_j^{\rm free}$ includes the initial displacement and all smooth
Euler biases but no event-rounding updates. Hence the $j$th scaled
boundary is
\begin{equation}
 p_j-\frac{n_j^\top A_j^{\rm free}(z)}{u_j}
 +\sum_{r<j}\frac{n_j^\top P_{j-1:r}\Delta_r}{u_j}s_r
 \in\mathbb Z.
 \label{eq:all-downstream-boundaries}
\end{equation}
This expression states the contribution of every upstream index to a
later threshold. Deleting these terms defines the
\emph{no-upstream-feedback ablation}; it is not a published competing
algorithm. Inside a fixed complete ceiling cell, $D_ze_p=P_m\cdots P_0$.
Across a cell change with integer offset differences $h_j$, the terminal
coefficient changes by
$-\sum_jP_{m:j}\Delta_jh_j$, with the downstream integers themselves
determined recursively. Nonlinear surface curvature affects the finite
location of the boundaries at the next order, so their leading scaled
hyperplanes need not be the exact finite-grid boundaries.

For $f_i(x)=M_ix+d_i$, $\Phi_i(t,s)=e^{(t-s)M_i}$ and
$f_i(x(t))=e^{(t-t_i)M_i}f_i(x_i)$. Variation of constants in
\eqref{eq:smooth-euler-bias} therefore gives exactly
\[
 w_i=-\tfrac12L_i e^{L_iM_i}M_if_i(x_i),
\]
recovering the affine recursion. No symmetry or contractivity is needed
for the smooth first-order theorem. The sharper finite spectral
certificate below retains those stronger affine hypotheses.

With one event this is the single ceiling law. With zero events it is
the ordinary first-order Euler-error expansion. If an event has
$\Delta_j=0$, its rounding update vanishes at first order, although
the variational field may still change. With all jumps zero the
first-order response is phase-independent and equals the continuous
variational response. Independent coordinate fields remove the relevant
cross-coordinate coefficients in \eqref{eq:all-downstream-boundaries};
commutativity alone does not imply that every such coefficient is zero.

The theorem makes no uniform assertion when $m$ grows, $\nu$ or event
separation tends to zero, a new surface enters a tube, or a derived
phase margin collapses. At an exact scaled boundary, the second-order
surface terms can choose the integer; the leading ceiling law alone
does not determine it. Coincident events, sliding, grazing, and reset
maps are outside the stated class. A numerical nonlinear network trace
does not verify these neighborhood hypotheses merely by having a
plausible event list.

\subsection{Two coupled events as an explicit example}
\label{app:coupled-response}
The general theorem specializes to the following exact affine geometry. Its two-index asymptotic proof is the $m=2$ case above; the finite boundaries and invariant-region checks below supply the example-specific information.
\subsubsection{A strongly convex affine realization}
Let $H=A^\top A\succ0$, $d=A^\top b$, and consider
\begin{equation}
 L(\theta)=\tfrac12\|A\theta-b\|^2+
 \tfrac12\sum_{j=1}^2[\operatorname{ReLU}(n_j^\top\theta-\tau_j)-c_j]^2,
 \qquad c_j\le0.
 \label{eq:coupled-risk}
\end{equation}
Each scalar residual term is convex: its derivative jumps from zero to
$-c_j\ge0$ at zero and has nonnegative derivative slope on each side.
Thus $L$ is globally $\lambda_{\min}(H)$-strongly convex. In active set
$S$, the negative gradient is affine,
\[
 f_S(\theta)=M_S\theta+d_S,\quad
 M_S=-H-\sum_{j\in S}n_jn_j^\top,\quad
 d_S=d+\sum_{j\in S}(\tau_j+c_j)n_j.
\]
The regional matrices are symmetric negative definite; this does not make
them commute. The jump on activating surface $j$ is $\Delta_j=c_jn_j$.

Assume the base flow from $\theta_0$ has precisely two activations, in
order $1,2$, at $0<t_1<t_2<T$, and no deactivations. There are disjoint
event tubes, positive distance from every other surface outside its tube,
and a positive minimum separation $\Delta$ of events and endpoints.
In each tube both adjacent normal speeds are at least $\nu>0$.
These are hypotheses about the continuous base path and its neighborhood,
not about a desired discrete word. They exclude grazing, recrossing,
simultaneity and new events. Fixed nonzero coupling is allowed throughout.
Write $t_0=0,t_3=T$, $x_i=\varphi_{t_i}(\theta_0)$, and let $M_i,d_i$
be the fields on the three successive regions, indexed $i=0,1,2$.
Define
\begin{align}
 L_i&=t_{i+1}-t_i,& P_i&=e^{L_iM_i},&
 w_i&=-\tfrac12L_iP_iM_if_i(x_i),\\
 u_j&=n_j^\top f_{j-1}(x_j)>0,&
 \Delta_j&=f_j(x_j)-f_{j-1}(x_j),&&j=1,2.
 \label{eq:coupled-data}
\end{align}
Here $w_i$ is the first-order Euler drift at a fixed clock, not an event
transfer. Assume $q\in C^2$ on a neighborhood of the terminal states.

\subsubsection{The two-event recursion}
For $p=(p_1,p_2)$ on the two-dimensional torus and a scaled displacement
$z\in\mathbb R^d$, construct $e_p(z)$ as follows:
\begin{align}
 a_1&=P_0z+w_0,&
 s_1&=\lceil p_1-n_1^\top a_1/u_1\rceil-p_1,&
 b_1&=a_1-\Delta_1s_1,\nonumber\\
 a_2&=P_1b_1+w_1,&
 s_2&=\lceil p_2-n_2^\top a_2/u_2\rceil-p_2,&
 b_2&=a_2-\Delta_2s_2,\nonumber\\
 e_p(z)&=P_2b_2+w_2.&&&&
 \label{eq:coupled-recursion}
\end{align}
Ceilings at integer arguments are excluded in the local-limit statement.
The phases $p_j$ alone are not corrected event phases: the drift and
incoming error inside each ceiling are essential.

\paragraph{The coupling and the branch jumps.}
There are exact finite-grid boundaries behind this expansion. Write
$\mathcal E_i(y)=(I+\eta M_i)y+\eta d_i$. Before the first event,
the boundary for landing at index $k$ is the affine hyperplane
$n_1^\top\mathcal E_0^k(\theta_0+\eta z)=\tau_1$, restricted by all
earlier strict signs. On a fixed first-index cell $k_1$, the analogous
second boundary is
$n_2^\top\mathcal E_1^{k_2-k_1}\mathcal E_0^{k_1}
(\theta_0+\eta z)=\tau_2$. These equations follow by solving the
executed affine recurrences; their expansion gives the boundaries below.
If a first boundary changes the first positive index from $k+1$ to $k$
while the second index stays $k_2$, the right (earlier-crossing) minus
left endpoint jump is exactly
\[
 \eta(I+\eta M_2)^{N-k_2}(I+\eta M_1)^{k_2-k-1}\Delta_1.
\]
Indeed the incoming state at index $k$ is on the surface, and the two
one-step continuations differ by $\eta\Delta_1$; all subsequent affine
differences propagate in the displayed order. If the second index also
changes, this expression alone is insufficient; the recursive second
ceiling determines its additional leading jump.

Inside a fixed first ceiling cell, the first scaled boundary is
$p_1-n_1^\top(P_0z+w_0)/u_1\in\mathbb Z$. The second boundary is
\begin{equation}
 p_2-\frac{n_2^\top[P_1(P_0z+w_0)+w_1]}{u_2}
 +\frac{n_2^\top P_1\Delta_1}{u_2}s_1\in\mathbb Z.
 \label{eq:coupled-second-boundary}
\end{equation}
The last term is the first event's rounding displacement of the second
boundary. Deleting it is the no-upstream-feedback ablation used in E2.
When a first-event boundary changes $s_1$ by an integer $h$, the direct
terminal error changes by $-P_2P_1\Delta_1h$; the second ceiling can also
change by an integer $h_2$, adding $-P_2\Delta_2h_2$. These formulas give
the endpoint jumps of the scaled map and preserve multiplication order.
On a fixed complete word the slope is $P_2P_1P_0$, as required by the
regional derivative limit. Saltation instead continuously changes the
event time; it contains neither ceiling nor downstream boundary shift.

\paragraph{Nonseparability and limiting checks.}
In E2, $H=\left(\begin{smallmatrix}1&0.4\\0.4&1.3\end{smallmatrix}\right)$,
$d=(2,2)$, $n_1=e_1,n_2=e_2$, $\tau_j=0$, $c=(-0.4,-0.6)$ and
$\theta_0=(-0.4,-1)$, $T=1$. Its $H$ is positive definite, and
$\|[M_1,M_0]\|_2=0.4$. A fixed linear change producing independent
scalar regional equations would simultaneously diagonalize these
matrices, contradicting their nonzero commutator. The second threshold
coefficient $n_2^\top P_1\Delta_1/u_2$ is nonzero (numerically $0.01336$).
The argument uses the specified fixed coupling, not its continuity at zero.
Its two-event geometry can be checked without selecting a trajectory by
its observed response. While both coordinates are negative, they increase,
with $2\le\dot x\le2.8$ and $2\le\dot y\le3.46$. Hence
$1/7\le t_1\le1/5$ and $-5/7\le y(t_1)\le-0.308$; the second
coordinate cannot cross first. In the next region the rectangle
$[0,1]\times[-1,0]$ is invariant until the second crossing, and
$1.6\le\dot y\le3.3$. Its duration is between $0.308/3.3$ and
$(5/7)/1.6$, so $t_2<0.647<T$. Afterward
$[0,1]\times[0,14/23]$ is invariant: all four boundary fields point
inward. Outgoing normal speeds at the two crossings are at least $1.6$
and $1$, respectively. Compact subsegments and continuity supply the
stated isolating tubes and margins. Finally the off-diagonal entry of
$e^{L_1M_1}$ is strictly negative for $L_1>0$ (diagonalizing this symmetric
two-by-two matrix gives a negative multiple of a hyperbolic sine).
Since $\Delta_1=-0.4e_1$, the coupling coefficient is strictly positive,
not merely nonzero to numerical precision.
For $H$ diagonal and these coordinate normals, this coefficient is zero
and each event depends only on its own coordinate. Deleting the second
event recovers the single-event ceiling law; in the scalar residual case
$a_1=e^{-t}z+mt/2$, so the effective phase is
$\{t/\eta-t/2\}$, agreeing with $\{t/[-\log(1-\eta)]\}$ to $O(\eta)$.
Its jumps reproduce Theorem~\ref{thm:scale-critical}. If $\Delta_j=0$,
that event's rounding displacement vanishes at first order, though the
regional Hessian can still change. With both jumps zero the response is
the ordinary continuous variational response; no spurious phase effect
survives. No independence or equidistribution of the two event phases
is assumed or needed.

\section{Finite-grid comparison certificates in the affine class}
\label{app:finite-event-certificate}
The asymptotic branch law applies to smooth regional fields. The explicit spectral certificates in this section retain symmetric negative-definite affine regional matrices; the nonlinear theorem does not inherit these bounds. All sign guarantees below are real-arithmetic statements, evaluated numerically in the experiments.
\subsection{A first-order comparison bound for the two-event example}
The following finite bound is conservative but does not hide an
unspecified asymptotic constant. It also handles inexact event localization.
Choose any strictly ordered \emph{supplied} knots $0<\widehat t_1<
\widehat t_2<T$. Construct scheduled states $\widehat x_i$ by exact regional
flows for these durations, whether or not the knots are exact event times.
Compute the recursion with this scheduled data, including the actual field
differences $f_j(\widehat x_j)-f_{j-1}(\widehat x_j)$ and positive incoming
speeds. Its candidate integers need not be right. Verify the candidate
word by evaluating all its regional affine grid extensions and checking
the required strict signs at every executed state. If this fails, return
\emph{unresolved}; do not repair the prediction with observed endpoint
objectives. A validated root/tube method can replace this exhaustive check,
but is not assumed here.

Use homogeneous coordinates
\[
 G_i=\begin{pmatrix}M_i&d_i\\0&0\end{pmatrix},\qquad
 Y_0=\begin{pmatrix}\theta_0\\1\end{pmatrix},\qquad
 Z=\begin{pmatrix}z\\0\end{pmatrix}.
\]
Here $G_i\in\mathbb R^{(d+1)\times(d+1)}$ and $Y_0,Z\in\mathbb R^{d+1}$.
Thus $I+\eta G_i$ represents one affine Euler update and $\|Z\|_2=\|z\|_2$.
All norms in this bound are Euclidean or induced operator norms.
Set $k_0=0,k_3=N$, $s_j=k_j-\widehat t_j/\eta$, $s_0=s_3=0$,
$\delta_i=s_{i+1}-s_i$ and $\widehat L_i=\widehat t_{i+1}-\widehat t_i$.
Let $g_i=\|G_i\|$, $r_i=\eta g_i<1$, and define
\begin{align}
 B_i&=\frac{|\delta_i|g_i}{1-r_i}
 +\frac{(\widehat L_i+\eta|\delta_i|)g_i^2}{2(1-r_i)^2},\nonumber\\
 C_i&=\frac{|\delta_i|g_i^2}{(1-r_i)^2}
 +\frac{2(\widehat L_i+\eta|\delta_i|)g_i^3}{3(1-r_i)^3},\nonumber\\
 K_i&=1+2\|M_i^{-1}d_i\|,\qquad B_*={\textstyle\sum_i}B_i,
 \quad C_*={\textstyle\sum_i}C_i,\nonumber\\
 r_\eta(z)&=\frac{\eta^2}{2}\prod_{i=0}^2K_i
 \left[(B_*^2+C_*)(\|Y_0\|+\eta\|z\|)+2B_*\|z\|\right].
 \label{eq:word-remainder}
\end{align}
Require $\widehat L_i+\eta\delta_i>0$. Let $\widehat e(z)$ denote the
scheduled-data recursion and $\widehat g=\nabla q(\widehat x_3)$.

\begin{proposition}[Finite comparison certificate]
\label{prop:coupled-certificate}
If the preceding word checks pass, then
\[
 \|\Theta_{\eta,T}(\theta_0+\eta z)-\widehat x_3-\eta\widehat e(z)\|
 \le r_\eta(z).
\]
For $z_\pm=\pm\zeta v$, $\|v\|=1$, define, using only the supplied
base data and the recursion,
\[
 \widehat D=\frac{\widehat g^\top[\widehat e(z_+)-\widehat e(z_-)]}{2\zeta}.
\]
If $\|\nabla^2q\|\le M_q$ on the endpoint Taylor segments, then
\begin{equation}
 |D_{\eta,\zeta\eta}(v)-\widehat D|\le
 \mathcal E=\frac{\|\widehat g\|(r_++r_-)
 +\frac12M_q\sum_{\sigma=\pm}(\eta\|\widehat e(z_\sigma)\|+r_\sigma)^2}
 {2\zeta\eta},\qquad r_\sigma=r_\eta(z_\sigma).
 \label{eq:coupled-certificate}
\end{equation}
Consequently $|\widehat D|>\mathcal E$ certifies the comparison sign;
otherwise the certificate returns unresolved. For a fixed positive margin,
bounded offsets and a nonzero limiting response this test succeeds on all
sufficiently fine meshes. It is not necessarily effective at coarse meshes.
\end{proposition}
\begin{proof}
The word check and induction over its grid states identify the actual
GD endpoint with the stated ordered affine product. For $0\le s\le\eta$
define the analytic family
\[
 Y(s)=U_2(s)U_1(s)U_0(s)(Y_0+sZ),\qquad
 U_i(s)=\exp[(\widehat L_i+s\delta_i)G_{i,s}].
\]
\[
 G_{i,s}=s^{-1}\log(I+sG_i),\qquad G_{i,0}=G_i.
\]
At $s=\eta$ this is exactly the candidate Euler word, because
$\widehat L_i/\eta+\delta_i=k_{i+1}-k_i$. At zero it is the scheduled
flow endpoint. Differentiating at zero, including the duration derivatives,
gives exactly $\widehat e(z)$ in the first $d$ coordinates: the terms from
$G_{i,s}'|_0=-G_i^2/2$ give $w_i$, and adjacent duration terms combine
as $-\Delta_js_j$ at each restart. This calculation does not require the
supplied knots to be exact surface roots.

The absolutely convergent logarithm series gives
\[
 \|G_{i,s}\|\le\frac{g_i}{1-r_i},\quad
 \|G_{i,s}'\|\le\frac{g_i^2}{2(1-r_i)^2},\quad
 \|G_{i,s}''\|\le\frac{2g_i^3}{3(1-r_i)^3}.
\]
For example the first derivative has coefficients $k/(k+1)\le k/2$;
the second has $k(k-1)/(k+1)\le k(k-1)/3$. Summing the geometric-series
derivatives proves the bounds. The exponent and its derivatives commute,
being functions of the same $G_i$. Its first two derivatives are bounded
by $B_i,C_i$. The top-left block of $U_i(s)$ is a contraction: the
eigenvalues of $M_i$ are negative, $1+s\lambda>0$, and the duration is
positive on $[0,\eta]$. The affine equilibrium is $-M_i^{-1}d_i$, so
$\|U_i(s)\|\le1+2\|M_i^{-1}d_i\|=K_i$. Hence
\[
 \|U_i'\|\le K_iB_i,\qquad
 \|U_i''\|\le K_i(B_i^2+C_i).
\]
Apply the product rule in its actual order, and the norm triangle
inequality, to obtain
\[
 \sup_{0\le s\le\eta}\|Y''(s)\|
 \le\prod_{i=0}^2 K_i[(B_*^2+C_*)(\|Y_0\|+\eta\|z\|)+2B_*\|z\|].
\]
Taylor's integral remainder proves the state bound. Taylor expansion of
$q$ about $\widehat x_3$ bounds each objective error by
$\|\widehat g\|r_\sigma+\tfrac12M_q
(\eta\|\widehat e(z_\sigma)\|+r_\sigma)^2$.
Subtract and divide by $2\zeta\eta$ to obtain the certificate.
\end{proof}

\paragraph{Localization, numerical arithmetic, and cost.}
No root error is silently dropped in the finite bound: supplied knots
define the expansion center, and their entire timing displacement occurs
in $\delta_i$ and the computed scheduled states. Poor localization can
make the word check fail or the bound large. In the asymptotic theorem
the knots are the exact base events; if approximate knots are used for
that conclusion their errors must vanish relative to $\eta$, or be
retained as phase errors. The finite certificate is a real-arithmetic
statement. Experiments evaluate its bounds and sign tests numerically,
with independent Euler and higher-precision checks; they are not
interval-arithmetic machine certificates. Reference trajectories are
base-point information, explicitly disclosed as controlled theory inputs.
Three regional propagators, two ceilings, and a terminal gradient suffice
for the prediction. The exhaustive word check costs $O(N)$ sign work and
removes a claim of computational advantage over evaluating two candidates.
This describes the original exhaustive test. The planar spectral test
in Appendix~\ref{app:contractive-certificate} replaces it with a bounded
number of sign evaluations; exact affine candidate evaluation is also
cheap, so no general training speedup follows.

\subsection{Ordered spectral jets for a finite number of events}
\label{app:contractive-certificate}

The preceding homogenized bound pays for affine translations repeatedly.
The regional symmetric negative definiteness permits a sharper bound in
physical coordinates. We state a second-order version; dropping its
quadratic term gives the first-order bound used in the ablation. Neither
version assumes that a predicted word is correct.

\paragraph{Data and finite-grid domain.}
For a fixed finite number $m$ of events, use any supplied ordered knots and their scheduled exact regional states
$x_i$ as in the preceding construction, now indexed by $i=0,\ldots,m$; hats are suppressed
in this subsection. For each proposed displacement $z$, compute candidate
indices by the finite-event recursion. Put $m_i=k_{i+1}-k_i>0$,
$L_i=t_{i+1}-t_i>0$, $\delta_i=m_i-L_i/\eta$, and
$h_i=-M_i^{-1}d_i$. Require $\eta\|M_i\|_2<1$ for each region.
All norms below are Euclidean/operator norms. For each eigenvalue
$\lambda<0$ of $M_i$, define
\begin{align}
 \alpha_i(\lambda)&=\delta_i\lambda-\tfrac12L_i\lambda^2,&
 \beta_i(\lambda)&=\tfrac13L_i\lambda^3-\tfrac12\delta_i\lambda^2,\nonumber\\
 t_{2i}(\lambda)&=\eta^2\left[\tfrac12|\delta_i|\lambda^2+
 \frac{m_i\eta|\lambda|^3}{3(1-\eta|\lambda|)}\right],&
 a_{2i}(\lambda)&=\eta|\alpha_i(\lambda)|+t_{2i}(\lambda),\nonumber\\
 t_{3i}(\lambda)&=\eta^3\left[\tfrac13|\delta_i||\lambda|^3+
 \frac{m_i\eta\lambda^4}{4(1-\eta|\lambda|)}\right],&
 b_i(\lambda)&=\eta^2|\beta_i(\lambda)|+t_{3i}(\lambda),\nonumber\\
 a_{3i}(\lambda)&=\eta|\alpha_i(\lambda)|+b_i(\lambda).&&
 \label{eq:spectral-tail-data}
\end{align}
The following maxima are over that finite spectrum:
\begin{align}
 d_{1i}&=\max_\lambda e^{L_i\lambda}(e^{a_{2i}}-1),\nonumber\\
 d_{2i}&=\max_\lambda e^{L_i\lambda}
       (t_{2i}+\tfrac12a_{2i}^2e^{a_{2i}}),\nonumber\\
 d_{3i}&=\max_\lambda e^{L_i\lambda}
       [t_{3i}+\tfrac12b_i(2\eta|\alpha_i|+b_i)
                      +\tfrac16a_{3i}^3e^{a_{3i}}],\nonumber\\
 \chi_i&=\max_\lambda(1+\eta\lambda)^{m_i}\le1.
 \label{eq:spectral-defects}
\end{align}
Functional calculus for the symmetric $M_i$ defines
$A_i=P_i\alpha_i(M_i)$ and
$B_i=P_i[\beta_i(M_i)+\alpha_i(M_i)^2/2]$.
Set $e_0=z$, $v_0=0$, $R_0=0$, and apply the ordered recursion for $i=0,\ldots,m$
\begin{align}
 e_{i+1}&=P_ie_i+A_i(x_i-h_i),\nonumber\\
 v_{i+1}&=P_iv_i+A_ie_i+B_i(x_i-h_i),\nonumber\\
 R_{i+1}&=\chi_iR_i+d_{3i}\|x_i-h_i\|
                    +\eta d_{2i}\|e_i\|+\eta^2d_{1i}\|v_i\|.
 \label{eq:finite-event-second-jet}
\end{align}
These $e_i$ are errors at the moving grid endpoints, unlike the
intermediate fixed-clock errors in the ceiling recursion. At the final
clock $e_{m+1}=e_{p_N}(z)$ with the supplied scheduled data. Matrix order is
never interchanged. In particular, $A_1e_1$ includes the first rounding
offset inside the second region's second-order error.

\paragraph{Detailed finite statement of Theorem~\ref{thm:contractive-main}.}
For every finite $\eta$, displacement $z$ and candidate word satisfying
the preceding domain conditions, if its strict signs are verified, then
\begin{equation}
 \|\Theta_{\eta,T}(\theta_0+\eta z)
       -(x_{m+1}+\eta e_{m+1}+\eta^2v_{m+1})\|\le R_{m+1}.
 \label{eq:finite-event-second-bound}
\end{equation}
This statement does not require a limiting phase or an unreported
``sufficiently small'' mesh. For fixed $m$, bounded $z,\delta_i,L_i,M_i,h_i,x_i$
and a common bound $\eta\|M_i\|\le r_*<1$, it has
$R_{m+1}=O(\eta^3)$ uniformly. The explicit formulas state every constant;
there is no inverse phase margin in the bound. Near-boundary failure is
handled by the word test, not by declaring the candidate itinerary true.
Unbounded scaled displacements retain the finite inequality but not this
uniform asymptotic rate.

\begin{proof}
The exact regional matrix is $E_i=(I+\eta M_i)^{m_i}$. On an eigenvalue,
write $a=m_i\log(1+\eta\lambda)-L_i\lambda$. The absolutely convergent
logarithm series yields both
\[
 |a-\eta\alpha_i|\le t_{2i},\qquad
 |a-\eta\alpha_i-\eta^2\beta_i|\le t_{3i}.
\]
For the second inequality, the cubic logarithm term contributes
$\eta^3\delta_i\lambda^3/3$, and the remaining series is bounded by
$m_i\eta^4|\lambda|^4/[4(1-\eta|\lambda|)]$.
The quadratic truncation gives the first inequality in the same way.
Taylor's integral remainder for the scalar exponential gives
$|e^a-1-a|\le a^2e^{|a|}/2$ and
$|e^a-1-a-a^2/2|\le |a|^3e^{|a|}/6$.
Moreover,
$|a^2-\eta^2\alpha_i^2|\le
b_i(2\eta|\alpha_i|+b_i)$.
Multiplication by $e^{L_i\lambda}$ and the spectral theorem therefore give
\[
 \|E_i-P_i\|\le d_{1i},\quad
 \|E_i-P_i-\eta A_i\|\le d_{2i},\quad
 \|E_i-P_i-\eta A_i-\eta^2B_i\|\le d_{3i}.
\]
The exact regional affine step is $h_i+E_i(y-h_i)$, whereas
$x_{i+1}=h_i+P_i(x_i-h_i)$. Substituting
$y=x_i+\eta e_i+\eta^2v_i+\rho_i$ and subtracting the displayed jet
leaves exactly
\[
 E_i\rho_i+(E_i-P_i-\eta A_i-\eta^2B_i)(x_i-h_i)
 +\eta(E_i-P_i-\eta A_i)e_i+\eta^2(E_i-P_i)v_i.
\]
Its norm is at most $R_{i+1}$ because $\|E_i\|=\chi_i$.
Induction from the exact initial state proves the result for the
candidate affine product. The strict word test identifies that product
with actual GD by induction over executed states. It does not use the
two rerun objective values. The rate follows directly from
$d_{ji}=O(\eta^j)$ and bounded ordered recursions.
\end{proof}

\paragraph{First-order ablation and objective control.}
For the unchanged first-order predictor, set $r_0=0$ and use
\[
 r_{i+1}=\chi_ir_i+d_{2i}\|x_i-h_i\|+\eta d_{1i}\|e_i\|.
\]
The same subtraction proves
$\|\Theta_{\eta,T}-(x_{m+1}+\eta e_{m+1})\|\le r_{m+1}$; inserting it into
\eqref{eq:coupled-certificate} isolates the effect of the tighter bound.
For second order let $\widetilde x_\pm=x_{m+1}+\eta e_{m+1}(\pm\zeta v)
+\eta^2v_{m+1}(\pm\zeta v)$ and $R_\pm=R_{m+1}(\pm\zeta v)$.
If $q$ is $C^2$ on an open neighborhood of the two closed balls of radii
$R_\pm$ about these points and $\|\nabla^2q\|\le M_q$ on the balls, then
\begin{align}
 \widehat D^{(2)}&=
   \frac{q(\widetilde x_+)-q(\widetilde x_-)}{2\zeta\eta},\nonumber\\
 \mathcal E^{(2)}&=
   \frac{\sum_{\sigma=\pm}
    [\|\nabla q(\widetilde x_\sigma)\|R_\sigma
                        +M_qR_\sigma^2/2]}{2\zeta\eta}.
 \label{eq:second-objective-certificate}
\end{align}
Taylor's integral formula on each ball proves
$|D_{\eta,\zeta\eta}(v)-\widehat D^{(2)}|\le\mathcal E^{(2)}$.
At fixed positive $\zeta$ and bounded offsets this response error is
$O(\eta^2)$, compared with the first-order $O(\eta)$ bound.
If both words pass and $|\widehat D^{(2)}|>\mathcal E^{(2)}$, the smaller
predicted objective is the smaller actual objective, and the actual
gap is at least $2\zeta\eta(|\widehat D^{(2)}|-\mathcal E^{(2)})$.
Thus a contrary AD/flow ordering has at least this two-candidate regret.
This last decision statement is an elementary corollary, not an
optimization convergence theorem. Equality is unresolved.

\paragraph{Constant-size planar word verification.}
For a two-dimensional symmetric negative definite regional matrix,
diagonalize $M=V\operatorname{diag}(\lambda_1,\lambda_2)V^\top$.
For a proposed block of $n_b$ executed states and any affine surface,
its signed value at integer $k\in\{0,\ldots,n_b-1\}$ is exactly
\[
 g(k)=a+b_1r_1^k+b_2r_2^k,\qquad r_l=1+\eta\lambda_l\in(0,1).
\]
The coefficients are obtained from the known block initial state and
equilibrium; no continuous event root within this block is assumed.
Its forward difference is
$g(k+1)-g(k)=b_1(r_1-1)r_1^k+b_2(r_2-1)r_2^k$.
If the two coefficients have the same sign, or one vanishes, this
sequence has no sign change. If the rates coincide, combine coefficients
and reach the same conclusion. Otherwise its unique possible zero is
\[
 k_*=
 \frac{\log|b_2(r_2-1)|-\log|b_1(r_1-1)|}
      {\log r_1-\log r_2}.
\]
Checking the endpoints and the clipped integers
$\lfloor k_*\rfloor+1,\lceil k_*\rceil+1$ gives the exact minimum and
maximum on all integer grid states: the forward difference has constant
sign on either side of its sole zero, so no other integer extremum is
possible. If $k_*$ is an integer, $g(k_*)=g(k_*+1)$; testing $k_*+1$
therefore captures their common extremal value. For $n_b=1$ only that state is needed.
Check the required strict sign of every surface at these at most four
indices, then propagate the block initial state by its exact spectral
power to start the next block. Each restart belongs to the next region;
the terminal state has no executed mask. These tests are necessary and
sufficient for the entire proposed strict word. Zero signs return
unresolved regardless of the convention used for ReLU at zero.
Let $S$ be the total number of surfaces selecting a region, including those absent from the event list. For two candidates and $m+1$ regions, at most $8(m+1)S$ scalar sign evaluations suffice, independently of $N$. The number $48$ applies only when $m=2$ and $S=2$. This is a real-arithmetic
operation count; bit precision near a zero is not uniformly bounded.

Endpoint masks alone cannot replace this test, even for a contracting
symmetric regional matrix. Take $\eta=1$,
$M=\operatorname{diag}(-0.1,-0.19)$, zero affine forcing, initial state
$(1,1)$ and surface $g(x)=0.24-x_1+x_2$.
Then $g(k)=0.24-0.9^k+0.81^k$ is positive at $k=0$ and $k=20$ but
negative at $k=6$. Thus a proposed word that maintains the positive
side throughout is false despite its matching endpoint signs. The
spectral-extremum test rejects it. This example concerns verification of
a regional extension, not a claim that the true switched path follows
that invalid extension after its first crossing.

\paragraph{Scope, cost, and limiting cases.}
The spectral remainder is valid in any dimension with symmetric negative
definite regional matrices; the constant-size test just proved is planar.
The experiment uses base-point event roots, scheduled states and regional
spectra. It does not use perturbed rerun states in its prediction or word
test. Checking a later block requires its exact initial affine state;
therefore the check is not a claim to avoid all candidate-state work.
It avoids an $N$-step replay, but exact word evaluation in this solvable
class is equally inexpensive. No speed advantage over that stronger
baseline is asserted. Matrix-log expansions, contractivity and the
two-exponential sign calculation are standard tools; their role here is
to turn the coupled response into a nonvacuous finite decision guarantee.

With no event the recursion is the ordinary second-order Euler-error
expansion. With commuting independent coordinate fields the coordinates
decouple. Repeated eigenvalues are covered by the spectral formulas and
the merged-coefficient word test. Vanishing jumps remove the first-order
rounding effect, but may leave a second-order effect when regional
Hessians change; setting both fields equal removes that artificial event
at every order. At a failed word, no error guarantee is issued. None of
these finite results asserts uniformity as $\eta\|M_i\|$ approaches one,
as a phase margin vanishes in the local-limit theorem, or as the number
of events grows. Floating-point and high-precision validations below are
not formal interval-arithmetic certificates.

\paragraph{General-dimensional work and arithmetic.}
The spectral remainder and jet recursion hold in any dimension $d$.
Regional eigendecompositions cost $O((m+1)d^3)$; after these and the
base data are available, ordered matrix-vector jets cost
$O((m+1)d^2)$ per candidate, excluding evaluation of $q$ and its
gradient. The planar extremum formula does not establish a
dimension-independent word test. A general safe alternative is to check
all $N$ executed states against all $S$ surfaces, costing
$O(N(d^2+Sd))$ per candidate for dense affine fields. A scalar surface
sequence in dimension $d$ contains up to $d$ exponential modes; any
faster general root isolation method requires its own proof and precision
analysis. Neither cheap word evaluation in this solvable affine class
nor the certificate proves a speedup for nonlinear network training.

The theorem is exact in real arithmetic. Its numerical implementation is
a numerically evaluated deterministic certificate, with floating-point
and independent high-precision checks. Those checks do not constitute
formal interval-arithmetic certification. The piecewise-smooth endpoint
theorem does not inherit this affine finite-grid certificate.

\section{Scalar branch cells, response levels, and frequencies}
\subsection{Finite-sample ReLU empirical-risk realization}
\label{app:relu-erm}

\begin{proposition}[Exact residual-ReLU flow and discrete derivatives]
\label{prop:relu-erm}
Fix \(m>0\), \(c>-m\), and \(T>0\), and consider
\begin{equation}
 \mathcal L_{m,c}(\theta)
 =\tfrac12(\theta-m)^2
  +\tfrac12(\operatorname{ReLU}(\theta)-c)^2.
 \label{eq:app-relu-erm-objective}
\end{equation}
For \(\theta_0\in(m(1-e^T),0)\), let
\(t_*=\log[(m-\theta_0)/m]\) and
\(r=(m+c)/m>0\).  The flow has exactly one stable transverse crossing
before \(T\), and
\[
 J_{\reg}=e^{-t_*}e^{-2(T-t_*)},
 \qquad
 D\varphi_T=rJ_{\reg}.
\]
The regional Hessians are \(H^-=1\) and \(H^+=2\), while the interface
coefficient is \(\beta=-c\).  If \(c\le0\), the objective is globally
1-strongly convex; in the crossing range its unique minimizer is
\((m+c)/2>0\).

For \(\eta=T/N<1/2\), define
\[
 \kappa_\eta=\frac{t_*}{-\log(1-\eta)}.
\]
If \(\kappa_\eta\notin\mathbb N\), the step size is nonresonant, the first
positive iterate is \(K_\eta=\lceil\kappa_\eta\rceil\), and
\begin{equation}
 D\Theta_{\eta,T}=J_{\eta,T}^{\AD}
 =(1-\eta)^{K_\eta}(1-2\eta)^{N-K_\eta}
 \longrightarrow J_{\reg}.
 \label{eq:relu-erm-discrete-jacobian}
\end{equation}
Moreover, the piecewise-linear GD interpolation converges uniformly to the
flow at \(O(\eta)\) for each fixed pair \((m,c)\).
\end{proposition}
\begin{proof}
For the constant input \(x=1\), define the standard two-output residual-ReLU
network and target
\[
 z_\theta(1)=\theta,
 \qquad
 F_\theta(1)=
 \begin{bmatrix}\theta\\\operatorname{ReLU}(\theta)\end{bmatrix},
 \qquad
 y=\begin{bmatrix}m\\c\end{bmatrix}.
\]
Then \(\frac12\|F_\theta(1)-y\|_2^2\) is exactly
\eqref{eq:app-relu-erm-objective}.  Its regional vector fields and Hessians
are
\[
 f^-(\theta)=m-\theta,\quad H^-=1,
 \qquad
 f^+(\theta)=m+c-2\theta,\quad H^+=2.
\]
Before the event,
\(\theta(t)=m+(\theta_0-m)e^{-t}\), which reaches zero at \(t_*\).
The initialization interval is exactly the condition \(0<t_*<T\).  The
one-sided speeds at zero are \(m\) and \(m+c\), so the crossing is stable,
same-direction transverse, and has saltation factor
\(r=(m+c)/m\).  After the event,
\[
 \theta(t)=\frac{m+c}{2}\bigl(1-e^{-2(t-t_*)}\bigr)>0,
\]
so there is no recrossing.  Multiplication of the regional propagators and the
event factor gives the stated formulas for \(J_{\reg}\) and \(D\varphi_T\).

The normal derivative jumps from \(-m\) to \(-(m+c)\), hence
\(\beta=-c\).  If \(c\le0\), the function
\(h(\theta)=\frac12(\operatorname{ReLU}(\theta)-c)^2\) is convex: its
left derivative at zero is \(0\), its right derivative has trace
\(-c\ge0\), and the latter then increases with slope one.  Since
\(\frac12(\theta-m)^2\) is 1-strongly convex, their sum is globally
1-strongly convex.  For \(c>-m\), its positive-region critical point
\((m+c)/2\) is positive and is therefore the unique minimizer.  For \(c>0\),
the downward derivative jump \(-c<0\) instead proves global nonconvexity.

For the discrete dynamics, while \(\theta_k^\eta\le0\),
\[
 \theta_k^\eta=m+(\theta_0-m)(1-\eta)^k.
\]
Thus an exact landing occurs precisely when \(\kappa_\eta\) is an integer,
and otherwise the first positive iterate is
\(K_\eta=\lceil\kappa_\eta\rceil\).  The incoming update gives
\(0<\theta_{K_\eta}^\eta<\eta m\).  Since \(0<1-2\eta<1\), the outgoing
recurrence
\[
 \theta_k^\eta=\frac{m+c}{2}
 +\left(\theta_{K_\eta}^\eta-\frac{m+c}{2}\right)
  (1-2\eta)^{k-K_\eta}
\]
is a convex combination of positive quantities at each step and therefore
stays positive.  Hence the discrete itinerary also has one crossing.  Its
branch sequence is locally fixed at nonresonant sizes, and differentiating the
two affine recurrences gives \eqref{eq:relu-erm-discrete-jacobian}.
Because \(-\log(1-\eta)=\eta+O(\eta^2)\),
\(K_\eta\eta=t_*+O(\eta)\).  Taking logarithms in
\eqref{eq:relu-erm-discrete-jacobian} therefore gives the claimed Jacobian
limit.  The standard uniform comparisons
\((1-\eta)^k=e^{-k\eta}[1+O(\eta)]\) and
\((1-2\eta)^k=e^{-2k\eta}[1+O(\eta)]\), together with the
\(O(\eta)\) crossing-time and overshoot errors, prove the uniform state bound.
Its constant may depend on \((m,c)\).
\end{proof}

\begin{corollary}[Uniform-margin amplification and convex attenuation]
\label{cor:app-uniform-relu-phases}
Fix \(a>0\), \(t_*\in(0,T)\), and a prescribed event ratio \(r>0\).  Choose
\[
 (m,c)=
 \begin{cases}
  (a/r,(r-1)a/r),&0<r\le1,\\
  (a,(r-1)a),&r\ge1,
 \end{cases}
 \qquad
 \theta_0=m(1-e^{t_*}).
\]
Then the event occurs at \(t_*\),
\(\min\{f^-(0),f^+(0)\}=a\), and
\(D\varphi_T=rJ_{\reg}\).  If \(r\le1\), the objective is globally
1-strongly convex.  In particular, for \(R>1\), the choice
\(m=Ra\), \(c=-(R-1)a\) gives
\begin{equation}
 D\varphi_T=R^{-1}J_{\reg},
 \qquad
 \frac{|J_{\reg}|}{|D\varphi_T|}=R,
 \label{eq:app-strongly-convex-reciprocal-gap}
\end{equation}
while a common interface tube \(|\theta|<a/4\) has normal-speed lower bound
\(a/2\) for every \(R\).  The limit is taken for each fixed \(R\) and then
\(\eta\downarrow0\); the Euler state-error constant is not asserted to be
uniform in \(R\).
\end{corollary}
\begin{proof}
The parameter choice gives \(m+c=rm\) and makes the smaller of \(m\) and
\(m+c\) equal to \(a\).  The event-time formula and
Proposition~\ref{prop:relu-erm} prove the ratio.  When \(r\le1\), \(c\le0\),
so global 1-strong convexity follows from that proposition.  In the reciprocal
family, the incoming speed is \(Ra\).  On the incoming side of
\(|\theta|<a/4\) it is at least \(Ra\), while on the outgoing side the speed
is \(a-2\theta\ge a/2\).  This proves the uniform transversality claim.
\end{proof}

\begin{proposition}[Exact-hit set for the residual-ReLU family]
\label{prop:relu-erm-resonance-set}
For the pre-event recurrence in Proposition~\ref{prop:relu-erm}, the positive
step sizes that place an evaluated iterate exactly at the interface are
\begin{equation}
 \eta_k=1-e^{-t_*/k},
 \qquad k\in\mathbb N.
 \label{eq:app-relu-resonance-set}
\end{equation}
This set is countable, has Lebesgue measure zero, and accumulates at zero.
Requiring \(T=N\eta\) selects only its intersection with the mesh family
\(\{T/N:N\in\mathbb N\}\).  For every \(0<t_*<T\), infinitely many
canonical meshes are nonresonant.
\end{proposition}
\begin{proof}
An exact hit at iterate \(k\) means
\((1-\eta)^k=e^{-t_*}\), which is equivalent to
\eqref{eq:app-relu-resonance-set}.  The set is countable and
\(1-e^{-t_*/k}\sim t_*/k\), proving accumulation at zero.  On a resonant
canonical mesh, the corresponding index is
\[
 k_N=\frac{t_*}{-\log(1-T/N)}
 =\frac{t_*}{T}N-\frac{t_*}{2}+O(N^{-1}).
\]
If every sufficiently large canonical mesh were resonant, both \(k_N\) and
\(k_{N+1}\) would be integers, while their integer difference would converge
to \(t_*/T\in(0,1)\), which is impossible.  Hence an infinite nonresonant
subsequence exists.
\end{proof}

\subsection{Perturbation radius, discrete branches, and derivative selection}
\label{app:perturbation-scales}

This section separates an elementary consequence of uniform state accuracy
from an exact, phase-dependent calculation for the residual-ReLU construction.
The latter resolves a finite perturbation question that the limiting Jacobian
alone cannot answer. Neither result is a claim of universal priority for
two-scale numerical differentiation.

\subsubsection{One target, three derivatives, and a directional branch radius}
Let $\lambda$ include the initialization and/or parameters of the regional
field, and fix a scalar terminal objective
\[
 Q_\eta(\lambda)=q(\Theta_{\eta,T}(\lambda),\lambda),\qquad
 Q_0(\lambda)=q(\varphi_T(\lambda),\lambda).
\]
\[
 D_{\eta,\epsilon}(v)=\frac{Q_\eta(\lambda+\epsilon v)-Q_\eta(\lambda-\epsilon v)}{2\epsilon},
 \quad \|v\|=1.
\]
Write $g_\eta=\nabla Q_\eta(\lambda)$ for the exact finite-program gradient,
$g_{\reg}=Z_{\reg}^{\top}q_\theta+q_\lambda$ for its regional limit, and
$g_0=\nabla Q_0(\lambda)=Z_{\rm flow}^{\top}q_\theta+q_\lambda$, where
the last two expressions evaluate $q$ at the flow endpoint. These identities
require smooth terminal traces; a terminal ReLU boundary needs a separate
analysis even if every training evaluation is nonresonant.

For a fixed mesh with strict training and terminal-validation masks, define
$\rho_\eta(v)$ as the supremum of $r>0$ for which every $|s|<r$ retains
\emph{all} executed training branches and terminal-validation branches at
$\lambda+sv$, within the parameter domain and regional smooth extensions.
This is a connected directional cell radius, not an endpoint mask test.
If $\psi_i(s)$ are all signed branch tests of the base-branch extensions,
and $|\psi_i'(s)|\le L_i$ on $[-r_0,r_0]$, then
\[
 \rho_\eta(v)\ge \min\{r_0,\min_{i:L_i>0}|\psi_i(0)|/L_i\}.
\]
Constant nonzero tests impose no restriction. The bound follows from the
mean-value theorem and induction over the finite executed program: before
the first changed test the base extension is the actual program. The $L_i$
must be bounds on an interval; a derivative sampled only at zero does not
certify them. They and the margins can depend on $\eta$.

\begin{proposition}[Elementary response bounds]
\label{prop:scale-elementary}
If $\epsilon<\rho_\eta(v)$ and
$s\mapsto Q_\eta(\lambda+sv)$ has third derivative bounded by $M_{3,\eta}$
on $[-\epsilon,\epsilon]$, then
\[
 |D_{\eta,\epsilon}(v)-g_\eta^\top v|
 \le M_{3,\eta}\epsilon^2/6.
\]
With only a second derivative bound $M_{2,\eta}$, the bound is
$M_{2,\eta}\epsilon/2$. Thus a varying-mesh conclusion requires the
corresponding remainder to vanish, as well as branch preservation.

Suppose on a fixed neighborhood $U$ that
$\sup_{\mu\in U}\|\Theta_{\eta,T}(\mu)-\varphi_T(\mu)\|\le C_s\eta$,
and $q$ is uniformly $L_q$-Lipschitz in its state argument on segments
joining these endpoints. If $\lambda+[-\epsilon,\epsilon]v\subset U$ and
$Q_0$ is $C^1$ there, then
\begin{equation}
 |D_{\eta,\epsilon}(v)-g_0^\top v|
 \le L_q C_s\eta/\epsilon+
 \omega_{\nabla Q_0}(\epsilon).
 \label{eq:scale-coarse-bound}
\end{equation}
If the third directional derivative of $Q_0$ is bounded by $M_{3,0}$,
replace the modulus term by $M_{3,0}\epsilon^2/6$.
\end{proposition}
\begin{proof}
Taylor expansion at $s=0$ cancels the even quadratic term in the central
difference. Each cubic remainder is bounded by $M_3\epsilon^3/6$, giving
the first claim after division by $2\epsilon$. Alternatively integrate the
directional first derivative over $[-\epsilon,\epsilon]$; its difference
from its value at zero is bounded by $M_2|s|$, whose average is
$M_2\epsilon/2$. Uniform state accuracy and the Lipschitz bound give
$|Q_\eta(\mu)-Q_0(\mu)|\le L_q C_s\eta$. Apply this at the two perturbed
endpoints, subtract, and divide by $2\epsilon$. The integral formula for
the central difference of $Q_0$ supplies the remaining modulus or Taylor
bound. No branch preservation of $Q_\eta$ is needed for this last step.
\end{proof}

\paragraph{When the neighborhood bound is available.}
For initialization perturbations, restrict (A1)--(A5) to a compact closure
of a sufficiently small neighborhood with common isolating tubes, a common
finite itinerary, uniformly positive endpoint/event separation $\Delta_t$,
normal speed $c_0$, and distance from other interfaces outside the tubes.
Require uniform regional field/Jacobian bounds and $C^2$ interface geometry
on a common compact tube. Then the proof of
Theorem~\ref{thm:general-first-variation} is uniform on that neighborhood.
Indeed each smooth-segment recurrence has the same Gronwall constant;
the interface coordinate error divided by $c_0$ bounds each clock error;
bounded speed converts it to a restart error; a finite induction over the
same number of tubes bounds every restart by $C_j\eta$. Taking their maximum
and the common mesh threshold proves the asserted uniform bound. The constants
depend on $T,m,M_f,M_A,c_0^{-1}$, the tube sizes and geometry, and the inverse
separation and non-event margins, not the grid or perturbed initialization.
For field/initialization/surface parameters, require the same assumptions
uniformly in $\lambda$ with jointly smooth regional extensions and smooth
initialization; apply the same argument to $(\theta,\lambda)'=(f(\theta,\lambda),0)$.
The state estimate alone needs no differentiability across discrete cells.
Thus $\epsilon\to0$ and $\eta/\epsilon\to0$ suffice \emph{with} these
uniform hypotheses and terminal regularity. Pointwise accuracy at a single
base point, or a collapsing event separation or crossing speed, does not
supply this conclusion. This is a corollary of numerical consistency and
the triangle inequality, not a new sensitivity mechanism.

\subsubsection{Exact cells and jumps of the residual-ReLU endpoint}
Fix $m>0$, $c>-m$, $T>0$, and $x\in U_T=(m(1-e^T),0)$. Put
\[
 t=\log((m-x)/m),\quad P=(m+c)/2,\quad B=m-x,
\]
\[
 E=e^{-2(T-t)},\quad F_0(x)=P(1-E),\quad J=e^{-t}E.
\]
Here $x$ is the initialization and the direction is the fixed scalar $v=1$.
For $\eta=T/N<1/2$, let $a=1-\eta$, $b=1-2\eta$, $h=-\log a$,
$\kappa=t/h$, and
\[
 z_k=m(1-a^{-k}),\qquad C_K=(z_K,z_{K-1}).
\]
Only cells with $1\le K<N$ are used below; compact subsets of $U_T$ satisfy
this condition for all sufficiently fine meshes. Boundary initializations
are excluded when a classical derivative is asserted.

\begin{theorem}[Exact branch and jump response]
\label{thm:scale-exact}
On $C_K$, the Euler endpoint and its derivative are
\begin{equation}
 F_\eta(y)=F_K(y):=P+\bigl[m+(y-m)a^K-P\bigr]b^{N-K},
 \qquad s_K=F_K'(y)=a^K b^{N-K}.
 \label{eq:scale-exact-cell}
\end{equation}
For a smooth terminal objective the training-cell radius at $x\in C_K$ is
\begin{equation}
 \rho_\eta=\min\{x-z_K,z_{K-1}-x\},\qquad
 |C_K|=m\eta a^{-K}.
 \label{eq:scale-exact-radius}
\end{equation}
If the objective has additional branches this radius must be intersected
with their inverse images. At an interior boundary $z_k$, $1\le k<N$,
the right-minus-left endpoint jump is
\begin{equation}
 [F_\eta]_{z_k}:=F_k(z_k)-F_{k+1}(z_k)
 =\eta c\,b^{N-k-1}.
 \label{eq:scale-jump}
\end{equation}
For $q\in C^1$, perturbation endpoints away from these boundaries, and
an interval contained in the above cells, the exact identity is
\begin{equation}
 D_{\eta,\epsilon}(1)=\frac{1}{2\epsilon}
 \left[\int_{x-\epsilon}^{x+\epsilon}q'(F_\eta(y))s_{K(y)}\,dy
 +\sum_{z_k\in(x-\epsilon,x+\epsilon)}
 \bigl(q(F_k(z_k))-q(F_{k+1}(z_k))\bigr)\right].
 \label{eq:scale-jump-decomposition}
\end{equation}
For $q(z)=\tfrac12(z-y_*)^2$ and $\epsilon<\rho_\eta$, the central response
equals the exact discrete derivative \emph{without a remainder}.
\end{theorem}
\begin{proof}
The incoming recurrence solves to $m+(y-m)a^k$. Its first positive state
has index $K$ exactly when $z_K<y<z_{K-1}$. That state lies in $(0,m\eta)$.
Every subsequent update is $b\theta+\eta(m+c)$, a convex combination of
positive values because $0<b<1$ and $P>0$. There is no recrossing. Solving
this recurrence proves \eqref{eq:scale-exact-cell} and the cell endpoints;
subtracting $z_K$ from $z_{K-1}$ gives the width. At $y=z_k$, the limit
from the right uses $k$ incoming steps and then starts the outgoing
recurrence from zero; the limit from the left uses one extra incoming
step and starts it from $m\eta$. Consequently
\[
 F_k(z_k)=P(1-b^{N-k}),\qquad
 F_{k+1}(z_k)=P+(m\eta-P)b^{N-k-1}.
\]
Their difference is $[P(1-b)-m\eta]b^{N-k-1}=\eta c b^{N-k-1}$.
Apply the fundamental theorem of calculus separately on each smooth
subinterval and add the intervening right-minus-left jumps; telescoping
gives \eqref{eq:scale-jump-decomposition}. Inside one cell a quadratic
$q$ composed with affine $F_K$ is quadratic, whose central difference is
exact. The value assigned at an interior jump does not affect this identity;
if a perturbation endpoint itself hits a jump, its selected value must be
used instead of an unspecified one-sided limit.
\end{proof}

\paragraph{Uniform constants in this construction.}
For $t(y)\in[\delta,T-\delta]$, $0<\delta<T/2$, and
$\eta<\min\{1/4,\delta\}$, the preceding formulas give
\begin{equation}
 \sup_y|F_\eta(y)-F_0(y)|\le C_s\eta,
 \qquad C_s=m+P(2+6T).
 \label{eq:scale-explicit-state-bound}
\end{equation}
To verify this bound, $h-\eta\le\eta^2/[2(1-\eta)]$ and $K=\lceil t/h\rceil$
give $|K\eta-t|\le(1+T)\eta$ (at a boundary either adjacent choice obeys
the same bound). Also
$|\log(1-2\eta)+2\eta|\le4\eta^2$. Thus the exponents of
$b^{N-K}$ and $e^{-2(T-t)}$ differ by at most $(2+6T)\eta$.
They are nonpositive, where the exponential is 1-Lipschitz. The incoming
overshoot is at most $m\eta$, proving the bound. All endpoints lie between
zero and $\max\{P,m/4\}$, so for the squared target $y_*$ one may take
$L_q=|y_*|+\max\{P,m/4\}$. No event separation other than the two endpoint
distances is needed for this one-event result. Since
$F_0(y)=P[1-e^{-2T}(1-y/m)^2]$, $Q_0=q\circ F_0$ is explicitly smooth
for a smooth $q$. The constants cease to be uniform if the instance scales
without bound; $c=-m$ is excluded because outgoing transversality fails.

\subsubsection{Critical scales and a fixed-initialization counterexample}
\begin{theorem}[Phase-resolved critical limit]
\label{thm:scale-critical}
For the fixed instance and initialization above, consider nonresonant
canonical meshes with fractional phases $p_N=\{t/h\}\to p\in[0,1]$ and
$\epsilon_N/\eta_N\to\zeta\in(0,\infty)$. Suppose
$\zeta\ne B|p-j|$ for every $j\in\mathbb Z$. For $q\in C^1$ near $F_0(x)$,
\begin{align}
 D_{\eta_N,\epsilon_N}(1)&\longrightarrow
 q'(F_0(x))\left[J+\frac{cE}{2\zeta}H(p,\zeta)\right],
 & H(p,\zeta)&=\#\{j\in\mathbb Z:B|p-j|<\zeta\},
 \label{eq:scale-critical-limit}\\
 \rho_{\eta_N}/\eta_N&\longrightarrow B\min\{p,1-p\}.
 \label{eq:scale-radius-phase}
\end{align}
Thus $\epsilon/\eta$ alone does not determine the limit. At the excluded
endpoint-equality phases, the higher-order endpoint approach and the value
selected at a hit matter, and the formula is not asserted.
\end{theorem}
\begin{proof}
Write $\kappa=\lfloor\kappa\rfloor+p_N$. For each fixed integer $j$,
\[
 \frac{z_{\lfloor\kappa\rfloor+j}-x}{\eta}
 =\frac{B[1-\exp(h(j-p_N))]}{\eta}
 \longrightarrow B(p-j).
\]
The two nearest boundaries prove the radius formula. On an $O(\eta)$
interval the cell widths divided by $\eta$ are uniformly bounded below,
so only a bounded number of boundaries can be present. The strict endpoint
condition makes their count eventually $H(p,\zeta)$. On every encountered
cell $K\eta\to t$, hence $s_K\to J$, uniformly over this bounded set;
also $F_\eta(y)\to F_0(x)$ uniformly. The integral term in
\eqref{eq:scale-jump-decomposition} therefore converges to $q'(F_0(x))J$.
For each encountered boundary, \eqref{eq:scale-jump} divided by $\eta$
converges to $cE$. The mean-value theorem and continuity of $q'$ give the
corresponding objective jump divided by $\eta$ as $q'(F_0(x))cE+o(1)$.
Divide their finite sum by $2\epsilon$ to finish the proof.
\end{proof}

There is no unspecified universal constant in
\eqref{eq:scale-critical-limit}: $B,E,J,c$ and the integer count are given
explicitly. For a compact family of strictly separated scaled endpoints,
the proof is uniform when $B$ is bounded away from zero, $m,c,T,t$ stay
in compact admissible sets, and $q'$ has a common continuity modulus.
Without a positive gap from $\zeta=B|p-j|$, stabilization of the count
is not uniform. At $c=0$ the jumps vanish and the leading phase dependence
disappears. For large $\zeta$, $H(p,\zeta)/(2\zeta)\to1/B$ uniformly in
$p$, and $J+cE/B=(m+c)E/B=D F_0(x)$. This checks agreement with the
coarser-radius bound, not an interchange of unproved joint limits.

\begin{corollary}[Small relative radius need not select the discrete gradient]
\label{cor:scale-near-resonance}
If $t/T$ is irrational and $c q'(F_0(x))\ne0$, there is a nonresonant
canonical subsequence and radii $\epsilon_N$ such that
\[
 \epsilon_N/\eta_N\to0,\qquad \epsilon_N>\rho_{\eta_N},\qquad
 |D_{\eta_N,\epsilon_N}(1)|\to\infty,
\]
although $Q_{\eta_N}'(x)\to q'(F_0(x))J$ is finite. In particular, with
$m=4,c=-3,T=1,t=\sqrt2/4$ and the fixed objective
$q(z)=\tfrac12(z-1)^2$, the exact discrete and flow derivatives are both
negative, but along that subsequence $D_{\eta_N,\epsilon_N}(1)>0$.
Consequently the ordering of the two predefined proposals $x\pm\epsilon_N$
is opposite to the ordering predicted by either infinitesimal derivative.
\end{corollary}
\begin{proof}
The expansion
$\kappa_N=(t/T)N-t/2+O(N^{-1})$ and the equidistribution proof in
Corollary~\ref{cor:phase-classification} give arbitrarily large indices
with fractional phase in every open subinterval of $(0,1)$.
A diagonal selection therefore gives strictly positive
$p_N\to0$ without hitting zero. Equation \eqref{eq:scale-radius-phase}
gives $\rho_N/\eta_N\to0$. Set $\epsilon_N=\sqrt{\rho_N\eta_N}$.
For large $N$ this interval contains the nearest boundary and no other
boundary, since neighboring widths are asymptotic to $B\eta_N$.
Its regular response is bounded, whereas its single objective jump divided
by $2\epsilon_N$ is
\[
 \frac{\eta_N}{2\epsilon_N}
 \bigl[cE q'(F_0(x))+o(1)\bigr],
\]
which diverges with the sign of $c q'(F_0(x))$. For the stated squared
objective $0<F_0(x)<1/2$, hence $q'(F_0(x))<0$; $c<0$ makes the jump
contribution positive, while $J>0$ and $DF_0(x)>0$ make both derivatives
negative. The central-response sign is precisely the sign of
$Q_\eta(x+\epsilon)-Q_\eta(x-\epsilon)$, proving the proposal comparison.
\end{proof}

The radii in this counterexample are phase-adaptive mechanism probes. The
task and direction, and the base initialization, are fixed. The result does
not assert that such phases are typical at an externally prescribed radius,
nor that a normalized flow gradient differs in this scalar example. Absolute
objective differences still vanish: a large response here divides a small
$O(\eta)$ jump by an even smaller radius. This explains why infinitesimal
correctness supplies no mesh-independent decision radius. Theorem~\ref{thm:finite-smooth-response}
handles fixed finite separated event sequences; changing itineraries,
colliding events and certified network cell radii remain open.

\subsection{Critical-scale response classification}
\label{app:phase-classification}
The following consolidates consequences of Theorem~\ref{thm:scale-critical};
it is not a separate sensitivity mechanism. Phase measure below is Lebesgue
measure on $[0,1)$, not a distribution over data, seeds, or trained models.

\begin{corollary}[Classification and canonical-mesh frequencies]
\label{cor:phase-classification}
In Theorem~\ref{thm:scale-critical}, fix $\zeta>0$ and put
\[
 \ell=2\zeta/B,\quad n=\lfloor\ell\rfloor,\quad r=\ell-n,
 \quad a=q'(F_0)J,\quad b=q'(F_0)cE/(2\zeta).
\]
Except at the finitely many endpoint phases $B|p-j|=\zeta$, the
count $H$ takes $n$ on a set of measure $1-r$ and $n+1$ on a set of
measure $r$ (the second set is empty when $r=0$). Consequently the critical
response has values $a+bn$ and $a+b(n+1)$ with those masses, and
\begin{equation}
 \int_0^1 R_\zeta(p)\,dp=Q_0'(x),\qquad
 \operatorname{Var}_p R_\zeta=b^2r(1-r).
 \label{eq:phase-moments}
\end{equation}
The variance vanishes exactly when $c q'(F_0)=0$ or $\ell$ is an integer.
For positive integer $\ell$, every allowed phase has response $Q_0'(x)$.

If $t/T$ is irrational and $\epsilon_N=\zeta T/N$, then the empirical
distribution of the actual responses $D_{T/N,\epsilon_N}$ converges to
this two-point distribution. In particular, if every level with positive
mass is nonzero, the frequency of positive responses is
\[
 (1-r)\mathbf1_{a+bn>0}+r\mathbf1_{a+b(n+1)>0}.
\]
This is also the frequency with which the predefined positive proposal
has higher objective than the negative proposal. For comparison to any
fixed nonzero derivative sign $s$, replace each indicator by
$\mathbf1_{s(a+bk)<0}$. If a positive-mass level is zero, its limiting
sign is not determined: the lower and upper frequency bounds for positive
responses are the masses of $R_\zeta>0$ and $R_\zeta\ge0$, respectively.
\end{corollary}
\begin{proof}
Write $u=p-\zeta/B$. The number of integers in $(u,u+\ell)$, away
from integer endpoints, is $\lfloor u+\ell\rfloor-\lfloor u\rfloor$.
It equals $n+\mathbf1_{\{u\}>1-r}$ outside the endpoint phases.
Translation on the circle preserves measure, proving the two masses.
The mean is $a+b\ell=q'(F_0)(J+cE/B)=Q_0'(x)$; the variance is that
of a Bernoulli variable of mass $r$, multiplied by $b^2$. Since $E,B,\zeta$
are strictly positive, these are exactly the stated degeneracies.

For the frequency statement one needs equidistribution, not density.
Expanding $-\log(1-T/N)$ gives
\[
 \frac{t}{-\log(1-T/N)}=\alpha N-t/2+O(N^{-1}),\qquad \alpha=t/T.
\]
For each nonzero integer $h$, the geometric sum of
$\exp(2\pi i h\alpha N)$ is bounded by $2/|1-e^{2\pi i h\alpha}|$.
Changing its argument by $O(N^{-1})$ changes the normalized sum through
$M$ by $O_h((1+\log M)/M)$. Thus every nonconstant Fourier average of
the actual phases tends to zero. Trigonometric-polynomial approximation
gives convergence for continuous periodic functions; bounding interval
indicators above and below by continuous functions gives equidistribution.

It remains to transfer phase measure to actual finite-grid responses.
Remove a neighborhood of width $\delta$ around each of the finitely many
endpoint phases. On its compact complement, the boundary expansion in
Theorem~\ref{thm:scale-critical} is uniform, the integer counts stabilize,
and $D_{T/N,\zeta T/N}-R_\zeta(p_N)\to0$ uniformly. Inside the removed
set the responses remain uniformly bounded: the cell widths are bounded
below by a constant times $\eta$, there are only boundedly many jumps,
and each jump is $O(\eta)$. Equidistribution bounds the upper density of
removed indices by the set's measure. First let $N\to\infty$, then
$\delta\downarrow0$. This proves the distribution and moment limits and,
when the occupied levels are nonzero, the sign frequencies. The same
argument bounds zero-level sign frequencies without assigning their signs.
An exact finite-grid endpoint hit can only occur in the removed sets
eventually, so its convention cannot change these frequencies.
\end{proof}

At $q'(F_0)=0$ both limiting responses vanish and higher-order terms decide
finite comparisons. At $c=0$ the phase term disappears; a nonzero $q'(F_0)$
still gives the common nonzero derivative sign. Rational $t/T$ is not
covered by the equidistribution statement. Positive integer $\ell$ is a
special critical-scale cancellation, not a phase-independent law for
arbitrary radii. No joint or independent phase distribution is asserted
for the coupled system.

\subsection{Proof and scope of the resolution synthesis}
\label{app:resolution-synthesis}

The three sufficient regimes in Section~\ref{sec:resolution} combine
existing statements; this section supplies the complete deduction and its
scope. Fix $Q_\eta=q\circ\Theta_{\eta,T}$,
$Q_0=q\circ\varphi_T$, a base initialization $x$, and a unit direction
$v$. Write $g_\eta=\nabla Q_\eta(x)$ at a nonresonant program,
$g_{\rm reg}$ for its regional limit, and $g_0=\nabla Q_0(x)$.
Terminal covectors in the latter two objects are evaluated at the flow
endpoint. The central response is
$D_{\eta,\epsilon}(v)=[Q_\eta(x+\epsilon v)-Q_\eta(x-\epsilon v)]/(2\epsilon)$.
The critical clause assumes the fixed finite-event geometry of
Theorem~\ref{thm:finite-smooth-response}, $q\in C^2$ near its terminal
point, and a common positive recursive phase margin on the stated
compact candidate set. The coarse clause assumes a fixed neighborhood
containing the whole proposed perturbation interval.

\paragraph{The directional cell is a whole-interval property.}
At a nonresonant grid, list all signed tests $\psi_i(s)$ of the
base-branch program, including terminal validation and any domain tests.
Assume their smooth extensions exist on $[-r_0,r_0]$, their values at
zero are nonzero, and $|\psi_i'(s)|\le L_i$ throughout that interval.
Then
\[
 \rho_\eta(v)\ge
 \min\{r_0,\min_{i:L_i>0}|\psi_i(0)|/L_i\}.
\]
Tests with $L_i=0$ impose no restriction. The mean-value theorem keeps
each extended test strictly on its original side inside the displayed
radius. Induction through the program then identifies the extended
execution with the actual one. Bounds computed only at the base point,
or agreement of two endpoint masks, do not establish this claim.
All tests, margins, derivative bounds and the resulting radius may
depend on $\eta$.

\paragraph{Within-cell remainder with only $C^2$ regularity.}
Let $h(s)=Q_\eta(x+sv)$. Two integral Taylor formulas give exactly
\[
 \frac{h(\epsilon)-h(-\epsilon)}{2\epsilon}-h'(0)
 =\frac{1}{2\epsilon}\int_0^\epsilon
 (\epsilon-u)[h''(u)-h''(-u)]\,du.
\]
Define the centered modulus
$\omega_{h''}(r)=\sup_{|s|\le r}|h''(s)-h''(0)|$.
The absolute value is at most
\[
 \frac{1}{\epsilon}\int_0^\epsilon
 (\epsilon-u)\omega_{h''}(u)\,du
 \le \frac{\epsilon}{2}\omega_{h''}(\epsilon).
\]
For fixed $\eta$, continuity of $h''$ makes this $o(\epsilon)$.
Across varying grids, pointwise $C^2$ regularity gives no uniform
modulus: the displayed product must tend to zero.
A simpler bound $\sup|h''|\le M_{2,\eta}$ gives
$M_{2,\eta}\epsilon/2$ directly from the same identity.
If $|h'''|\le M_{3,\eta}$, then
$|h''(u)-h''(-u)|\le2M_{3,\eta}u$, and integration gives
$M_{3,\eta}\epsilon^2/6$.
The critical theorem's $C^2$ regional-field hypotheses do not by
themselves supply this additional third-derivative bound.
Under the derivative-selection theorem, write
$\alpha_\eta=\|g_\eta-g_{\rm reg}\|\to0$ at nonresonant meshes.
The within-cell bound plus $\alpha_\eta$ also bounds
$|D_{\eta,\epsilon}(v)-g_{\rm reg}^\top v|$.
This extra limit concerns the regional object, not $\nabla Q_0$.

\paragraph{Critical response and an explicit objective constant.}
The endpoint theorem provides, uniformly on the specified compact $K$,
\[
 \Theta_{\eta,T}(x+\eta z)=x_T+\eta e_{p_N}(z)+r_\eta(z),
 \quad \|e_{p_N}(z)\|\le B,\quad
 \|r_\eta(z)\|\le C\eta^2.
\]
Let $G=\|\nabla q(x_T)\|$ and bound the Hessian of $q$ by $M_q$ on
a ball around $x_T$. Add to the mesh threshold the requirement that
$\eta B+C\eta^2$ lies inside this ball, and fix an upper mesh
$\bar\eta$. Taylor's formula yields
\[
 \left|Q_\eta(x+\eta z)-q(x_T)
 -\eta\nabla q(x_T)^\top e_{p_N}(z)\right|
 \le C_Q\eta^2,\qquad
 C_Q=GC+\tfrac12M_q(B+C\bar\eta)^2.
\]
Apply this at $z=\pm\zeta v$ and divide their difference by
$2\zeta\eta$. This proves the critical bound. The detailed endpoint
theorem does not need $0\in K$ unless a base-subtracted response is
requested, so a two-point $K=\{-\zeta v,\zeta v\}$ is sufficient
for this central comparison. The three-point version in the main
box matches the main theorem's presentation and is slightly
stronger than necessary.

The constant depends on the fixed event count, time, scaled displacement
bound, field and surface derivative bounds, normal speeds, event
separation and isolating-tube margins. The mesh threshold also depends
on the positive recursive phase margin $\mu$; this controls the actual
integer choices against the $O(\eta^2)$ surface residuals.
It is not enough to require only raw distances
$\operatorname{dist}(t_j/\eta,\mathbb Z)$.
For a phase limit $p_N\to p$ whose recursive margins are positive on
$K$, induction through the finitely many ceilings fixes their integers
for nearby phases and gives
$\sup_K\|e_{p_N}-e_p\|\le C_p\operatorname{dist}(p_N,p)$.
This supplies the additional term
$GC_p\operatorname{dist}(p_N,p)/\zeta$.
The induction does not assume independent, dense or uniform event phases.

The compact set can be finite or disconnected. Requiring a positive
margin on a connected interval crossing a recursive branch boundary
would exclude the very cross-branch comparisons the theorem describes.
By contrast, the within-cell radius deliberately requires every point
of its interval to preserve the branch. These are different conditions.
Fixed $\zeta>0$ is essential to the stated $O(\eta)$ response bound.
For a family $\zeta_\eta$ in a fixed compact subset of $(0,\infty)$,
the same bound follows if one common compact $K$ and one common phase
margin contain all requested endpoints. Letting
$\zeta_\eta\to0$ or $\zeta_\eta\to\infty$ requires separate uniform
remainder and geometry control; it is not an unqualified consequence.

\paragraph{Coarser perturbations.}
At the two endpoints, uniform objective consistency contributes at most
$2C_0\eta$ to the numerator. Also
\[
 \frac{Q_0(x+\epsilon v)-Q_0(x-\epsilon v)}{2\epsilon}
 =\frac{1}{2\epsilon}\int_{-\epsilon}^{\epsilon}
 \nabla Q_0(x+sv)^\top v\,ds.
\]
Subtract $\nabla Q_0(x)^\top v$, use $\|v\|=1$, and add the two
bounds. Uniform state accuracy $C_s\eta$ and a uniform state-Lipschitz
constant $L_q$ imply $C_0=L_qC_s$. Such a state constant follows from
the existing separated-crossing state proof on a fixed neighborhood
with common finite itinerary, event/end-point separation, transverse
speeds, field bounds, and tube/non-event margins. A state estimate at
the base point alone is insufficient. Smoothness of the flow objective
must include the terminal validation trace; training transversality
alone does not imply $Q_0\in C^1$.
This triangle-inequality argument is a direct corollary, not a new
approximation theorem.

\paragraph{A critical response need not be a gradient.}
This failure already occurs for a planar single event. Let $T=1$,
$x=(-1/2,0)$, guard $x_1=0$, incoming field $(1,0)$, outgoing field
$(2,0)$, and $q(y)=y_1$. On odd grids $N$, the event phase is $p=1/2$,
the smooth bias vanishes, and the endpoint response is exactly
\[
 e_p(z)=\bigl(z_1-\lceil1/2-z_1\rceil+1/2,\ z_2\bigr).
\]
All fields and the guard satisfy the finite-event theorem's smoothness,
separation and same-direction speed conditions. Fix $\zeta=3/4$.
Coordinate comparisons give
$\widehat D_{\eta,\zeta}(e_1)=7/3$ and
$\widehat D_{\eta,\zeta}(e_2)=0$.
For the unit vector $v=(1/3,2\sqrt2/3)$, however,
\[
 \widehat D_{\eta,\zeta}(v)=1/3
 \ne (7/3,0)^\top v=7/9.
\]
Every ceiling argument used here is at least $1/4$ from an integer.
Thus the counterexample is strictly away from boundary ambiguity, and
the critical prediction equals the exact finite response on these grids.
One may define a \emph{finite coordinate-response vector}
$h_i^{(\zeta)}=\widehat D_{\eta,\zeta}(e_i)$ and report its cosine
with another explicitly defined vector, but that cosine does not measure
agreement of directional derivatives and $h^{(\zeta)\top}v$ does not
generally predict another direction's finite response.

\paragraph{Decision and parameter scope.}
Any verified error bound $|D-\widehat D|\le E$ certifies the sign only
when $|\widehat D|>E$; otherwise the bound returns unresolved. It also
certifies an absolute candidate-objective gap of at least
$2\epsilon(|\widehat D|-E)$. An asymptotic $O(\eta)$ with an
uncomputed constant is not by itself a numerical certificate. The
affine finite-word certificate supplies a separately evaluated bound;
the network experiment currently supplies observed response errors.
No Taylor bound guarantees that AD is the most accurate predictor at
every cell-preserving radius, and no coarse-radius consistency bound
guarantees that the flow proposal has lower loss on every finite grid.
For a fixed-dimensional smooth outer parameter, the critical theorem
can be applied to the augmented state $(\theta,\lambda)$ with
$\dot\lambda=0$ and jointly $C^2$ fields/guards, provided its same
geometric hypotheses are checked. A parameter changing architecture,
program length, or the time discretization is not covered by this
constant-state augmentation.

\section{Supporting convexity, event geometry, and mismatch results}
\subsection{Atomic curvature and one-event geometry}
\label{sec:geometry}
\paragraph{Spatial versus pathwise curvature.}
The distributional Hessian has regional density and interface mass
(Appendix~\ref{app:atomic-geometry}); along a trajectory the event mass is
normalized by crossing speed.

\begin{proposition}[Speed-normalized pathwise curvature]
\label{prop:pathwise-curvature}
Along the reference itinerary, define
\begin{equation}
 \dd K_\theta(t)=H_{\alpha(t)}(\theta(t))\,\dd t+
 \sum_{j=1}^m\frac{\beta_j}{\nu_j^\top f_j^-}
 \nu_j\nu_j^\top\,\delta_{t_j}(\dd t).
 \label{eq:pathwise-curvature-measure}
\end{equation}
The left/prepoint Stieltjes system
\(Z(t)=I-\int_{(0,t]}\dd K_\theta(s)Z(s^-)\) satisfies
\(Z(t_j^+)=\Xi_jZ(t_j^-)\) and \(Z(T)=D\varphi_T\).
Deleting its atoms gives the discrete-AD limit \(J_{\reg}\).
\end{proposition}

The atoms act by prepoint multiplication, not exponentiation; the spatial and
pathwise distinction is detailed in Appendix~\ref{app:atomic-geometry}.
BV calculus and saltation are classical
\citep{demengel1984hessien,ambrosio2000bv,hiskens2000trajectory,
kong2024saltation}; our result identifies which part survives in the exact
discrete derivative sequence.

\paragraph{One-event discrepancy.}
\begin{proposition}[One-event rank-one geometry]
\label{prop:one-event-geometry}
For one transverse event,
\begin{align}
 \Delta_T:=D\varphi_T-J_{\reg}
 &=-\frac{\beta}{\nu^\top f^-}
   (\Phi_1\nu)(\nu^\top\Phi_0),                                      \label{eq:one-event-gap}\\
 \|\Delta_T\|_2=\|\Delta_T\|_F
 &=\frac{|\beta|}{|\nu^\top f^-|}
   \|\Phi_1\nu\|_2\|\Phi_0^\top\nu\|_2.                              \label{eq:one-event-exact-norm}
\end{align}
If \(\beta\ne0\), the gap has rank one.  For a terminal scalar objective, its
outer-gradient gap then vanishes exactly when the terminal gradient is
orthogonal to \(\Phi_1\nu\). Appendix~\ref{app:atomic-geometry} gives the
kernel, image, adjoint, and random-direction formulas.
\end{proposition}

\subsection{Distributional Hessian and singular interface curvature}
\label{app:distributional}

\begin{proposition}[Distributional Hessian]
\label{prop:distributional-hessian}
Let a \(C^2\) hypersurface \(\Sigma\) split a neighborhood into minus and plus
regions.  Suppose \(\mathcal L\) is continuous and has \(C^2\) extensions
\(\mathcal L^\pm\) with one-sided gradients.  For the unit normal \(\nu\) from
minus to plus, there is a scalar trace \(\beta\) such that
\[
 [\nabla\mathcal L]=\beta\nu,
\]
and, as a matrix-valued distribution,
\[
 D^2\mathcal L
 =H^-\mathbf 1_-\,\mathrm dx+H^+\mathbf 1_+\,\mathrm dx
 +\beta\nu\nu^\top\,
   \mathcal H^{d-1}\!\restriction_\Sigma.
\]
\end{proposition}
\begin{proof}
The traces of \(\mathcal L^+\) and \(\mathcal L^-\) agree on \(\Sigma\).
Their tangential derivatives therefore agree, so their gradient difference is
normal: \([\nabla\mathcal L]=\beta\nu\).  Integrating each component of
\(\nabla\mathcal L\) by parts on the two regions gives the regional Hessians as
volume terms.  The outward normal of the minus region is \(\nu\), while that
of the plus region is \(-\nu\); the boundary terms combine as
\([\nabla\mathcal L]\otimes\nu=\beta\nu\nu^\top\), which is the stated
measure.
\end{proof}

The decomposition is standard BV calculus.  Its role here is to locate the
curvature omitted by regional Hessian products.  Pulling the interface term
through a transverse trajectory must account for the changing normal speed.
In the resolved layer of Appendix~\ref{app:smoothing}, the scalar identity
\(\mathrm d\log Z=\mathrm d\log f\) converts the concentrated curvature into
\(f^+/f^-\), exactly the normal saltation factor.  A naive exponential of the
bare hypersurface mass would be incorrect.

\subsection{Atomic curvature and event geometry}
\label{app:atomic-geometry}

This section proves the pathwise-curvature, one-event, convexity, and
ReLU-event statements used in Sections~\ref{sec:geometry} and
\ref{sec:convex-relu}.  The Stieltjes convention below is part of the
definition; it is important because the vector field itself jumps at an
event.

\subsubsection{Prepoint Stieltjes formulation}

Along the reference trajectory, write
\[
 H(t)=H_{\alpha(t)}(\theta(t)),
 \qquad
 a_j=\frac{\beta_j}{\nu_j^\top f_j^-},
\]
and define the finite matrix-valued measure
\begin{equation}
 K_\theta(B)
 =\int_B H(t)\,\dd t
  +\sum_{j:t_j\in B}a_j\nu_j\nu_j^\top.
 \label{eq:app-pathwise-curvature-measure}
\end{equation}
Here and below, event data are evaluated at the nominal event state.  A
right-continuous matrix function \(Z\), with left limits, solves the
\emph{prepoint} Stieltjes system when
\begin{equation}
 Z(t)=I-\int_{(0,t]}\dd K_\theta(s)\,Z(s^-).
 \label{eq:app-prepoint-stieltjes}
\end{equation}
Thus an atom acts linearly on the left limit.  Equation
\eqref{eq:app-prepoint-stieltjes} is not an exponential convention for
matrix atoms.

\paragraph{Proof of Proposition~\ref{prop:pathwise-curvature}.}
Off the event times, the absolutely continuous part of
\eqref{eq:app-prepoint-stieltjes} gives
\[
 \dot Z(t)=-H(t)Z(t)=A_{\alpha(t)}(\theta(t))Z(t).
\]
At \(t_j\), subtracting the left limit from the right value gives
\begin{align*}
 Z(t_j^+)-Z(t_j^-)
 &=-K_\theta(\{t_j\})Z(t_j^-)\\
 &=-\frac{\beta_j}{\nu_j^\top f_j^-}
   \nu_j\nu_j^\top Z(t_j^-).
\end{align*}
Hence
\[
 Z(t_j^+)
 =\left(I-
 \frac{\beta_j\nu_j\nu_j^\top}{\nu_j^\top f_j^-}\right)Z(t_j^-)
 =\Xi_jZ(t_j^-).
\]
Solving the regional linear equations and applying these finitely many jumps
in chronological order yields
\[
 Z(T)=\Phi_m\Xi_m\Phi_{m-1}\cdots\Phi_1\Xi_1\Phi_0
 =D\varphi_T.
\]
Existence and uniqueness follow from the same finite construction.  If the
atoms are removed from \(K_\theta\), the solution remains continuous at each
event and its terminal value is
\(\Phi_m\cdots\Phi_0=J_{\reg}\).  Theorem~\ref{thm:main} identifies this
atom-free solution with the vanishing-step branchwise-AD limit.
\qed

The spatial measure in Proposition~\ref{prop:distributional-hessian} and the
pathwise measure in \eqref{eq:app-pathwise-curvature-measure} are related but
not identical objects.  Since the trajectory velocity has two one-sided
values at the interface, a bare distributional pullback does not select a
crossing speed.  The incoming-speed normalization and the prepoint convention
in \eqref{eq:app-prepoint-stieltjes} encode the fixed-clock event derivative.
In particular,
\(\exp(-a_j\nu_j\nu_j^\top)\) is generally different from
\(I-a_j\nu_j\nu_j^\top=\Xi_j\).

\subsubsection{One-event rank-one geometry}

\paragraph{Proof of Proposition~\ref{prop:one-event-geometry}.}
For one event, Theorem~\ref{thm:main} and
\eqref{eq:gradient-saltation} give
\begin{align}
 \Delta_T
 &:=D\varphi_T-J_{\reg}
 =\Phi_1(\Xi-I)\Phi_0\notag\\
 &=-\frac{\beta}{\nu^\top f^-}
   (\Phi_1\nu)(\nu^\top\Phi_0).
 \label{eq:app-one-event-outer-product}
\end{align}
Both regional fundamental matrices are invertible, so
\(\Phi_1\nu\ne0\) and \(\Phi_0^\top\nu\ne0\).  If \(\beta\ne0\),
\eqref{eq:app-one-event-outer-product} is therefore a nonzero outer product
and has rank one.  A rank-one matrix \(uv^\top\) has one nonzero singular
value \(\|u\|_2\|v\|_2\).  This proves
\[
 \|\Delta_T\|_2=\|\Delta_T\|_F
 =\frac{|\beta|}{|\nu^\top f^-|}
   \|\Phi_1\nu\|_2\|\Phi_0^\top\nu\|_2.
\]
The same outer-product formula gives
\[
 \ker\Delta_T=(\Phi_0^\top\nu)^\perp,
 \qquad
 \operatorname{im}\Delta_T=\operatorname{span}(\Phi_1\nu),
\]
and, for \(g_T=\nabla q(\varphi_T(\theta_0))\),
\[
 \Delta_T^\top g_T
 =-\frac{\beta}{\nu^\top f^-}
   (\nu^\top\Phi_1^\top g_T)\Phi_0^\top\nu.
\]
Since the last vector factor is nonzero, the hypergradient discrepancy
vanishes exactly when \(\beta=0\) or
\(g_T\perp\Phi_1\nu\).

Finally, if \(V\) is uniform on the unit sphere in \(\mathbb R^d\), then
\(\mathbb E[VV^\top]=I/d\), and hence
\[
 \mathbb E\|\Delta_TV\|_2^2
 =\tr\!\left(\Delta_T^\top\Delta_T\,\mathbb E[VV^\top]\right)
 =\frac{\|\Delta_T\|_F^2}{d}.
\]
This is a directional statement at a fixed event, not a prevalence claim for
network trajectories.
\qed

\subsection{Convex and strongly convex consequences (supporting results)}
\label{sec:convex-relu}

The attenuation and determinant statements below are direct consequences of classical normal-jump geometry and hybrid composition; the squared-loss and coupled examples supply explicit realizations. See Appendix~\ref{app:prior-art-comparison} for ownership.

\paragraph{Does convexity restore tangent consistency?}
No. Convexity constrains direct event geometry, but it does not remove event
timing. In fact, a strict convex event prevents complete cancellation.

\begin{theorem}[Convex attenuation and no complete cancellation]
\label{thm:convex-events}
Suppose, in addition to Theorem~\ref{thm:main}, that \(\mathcal L\) is globally
convex.  At every traversed interface,
\begin{equation}
 [\nabla\mathcal L]_j=\beta_j\nu_j
 \quad\Longrightarrow\quad
 \beta_j\ge0,\qquad
 r_j=1-\frac{\beta_j}{\nu_j^\top f_j^-}\in(0,1].
 \label{eq:convex-event-sign}
\end{equation}
Moreover,
\begin{equation}
 \frac{\det D\varphi_T}{\det J_{\reg}}=\prod_{j=1}^m r_j.
 \label{eq:convex-determinant-ratio}
\end{equation}
Thus, if any event is strict (\(\beta_j>0\), equivalently \(r_j<1\)), then
the product is below one and transported event factors cannot cancel
completely.
\end{theorem}

\begin{corollary}[Quantitative event-volume certificate]
\label{cor:convex-event-volume-bound}
Under Theorem~\ref{thm:convex-events}, using the transported factors
of Eq.~\eqref{eq:general-transported-saltation} in Appendix~\ref{app:multi-event},
define the cumulative relative event map and its volume ratio by
\[
 \mathcal C_T:=J_{\reg}^{-1}D\varphi_T
 =\widetilde\Xi_m\cdots\widetilde\Xi_1,
 \qquad
 \delta_T:=\det\mathcal C_T=\prod_{j=1}^m r_j\in(0,1].
\]
Then
\begin{equation}
 \|\mathcal C_T-I\|_2\ge 1-\delta_T^{1/d},
 \qquad
 \|D\varphi_T-J_{\reg}\|_2
 \ge \sigma_{\min}(J_{\reg})\bigl(1-\delta_T^{1/d}\bigr).
 \label{eq:convex-volume-gap-lower-bound}
\end{equation}
The certificate is strict when any event is strict.  It is dimension-diluted
and does not identify the most affected direction.
\end{corollary}

This is a volumetric statement under same-direction transversality, not a
claim that every direction contracts.  Conversely, \(\beta<0\) gives \(r>1\):
negative singular interface curvature amplifies the direct normal mode and is
incompatible with global convexity.

\paragraph{ReLU event law.}
For one switching preactivation \(s(\theta)\), let
\(q=\nabla_\theta s\), write
\(F_{\rm act}-F_{\rm inact}=a\,s+o(\|\theta-z\|)\), and let \(g_y\) be the
loss gradient at the common boundary output.  Corollary
\ref{cor:relu-event-law} shows that, with empirical averaging factor
\(\gamma\), \(\chi=\langle g_y,a\rangle\), and \(n=q\) for activation,
\[
 [\nabla_\theta\mathcal L]=\gamma\chi n,\qquad
 \lambda_\perp(\Xi)
 =1-\frac{\gamma\chi\|q\|^2}{q^\top f^-}
\]
for inactive-to-active crossing.  Thus the projected residual determines
attenuation, invisibility, or amplification.  This is a single-simple-event
statement, not a simultaneous-switch or genericity result.

\paragraph{A minimal strongly convex residual-ReLU risk.}
For one constant input, consider the empirical squared loss
\begin{equation}
 F_\theta(1)=(\theta,\operatorname{ReLU}(\theta)),\qquad
 \mathcal L_{m,c}(\theta)
 =\tfrac12(\theta-m)^2+
  \tfrac12(\operatorname{ReLU}(\theta)-c)^2.
 \label{eq:relu-erm-objective}
\end{equation}

\begin{theorem}[Residual-ReLU risks with prescribed sensitivity ratios]
\label{thm:residual-relu-phase}
Fix \(m>0,c>-m,T>0\).  Every
\(\theta_0\in U_T=(m(1-e^T),0)\) crosses zero exactly once, at
\(t_*=\log((m-\theta_0)/m)\), with
\[
 (f^-,f^+,H^-,H^+,\beta,r)
 =\left(m,m+c,1,2,-c,\frac{m+c}{m}\right).
\]
Along every nonresonant canonical-mesh sequence,
\begin{equation}
 D\Theta_{\eta,T}=J_{\eta,T}^{\AD}\to
 J_{\reg}=e^{-t_*}e^{-2(T-t_*)},
 \qquad
 D\varphi_T=rJ_{\reg}.
 \label{eq:relu-phase-sensitivities}
\end{equation}
For \(c\le0\), the objective is globally 1-strongly convex; for \(c>0\), it
is globally nonconvex.  Given \(R>1\), \(a>0\), and \(t_*\in(0,T)\), the choice
\[
 m=Ra,\qquad c=-(R-1)a,\qquad
 \theta_0=m(1-e^{t_*})
\]
is globally 1-strongly convex, has minimum one-sided speed \(a\), and satisfies
\begin{equation}
 D\varphi_T=R^{-1}J_{\reg},\qquad
 |J_{\reg}|/|D\varphi_T|=R.
 \label{eq:strongly-convex-reciprocal-gap}
\end{equation}
\end{theorem}

The ratio holds on the open set \(U_T\) for each fixed instance; it is not a
uniform joint \(R,\eta\) limit.  It also does not arise from unstable regional
curvature: the regional Hessians are \(1\) and \(2\).  In this scalar result
\(\|\Xi\|_2=|r|\); after embedding the same normal transfer in \(d\ge2\),
\(\|\Xi\|_2=\max\{1,|r|\}\).

\begin{lemma}[Ratio--conditioning tradeoff]
\label{lem:ratio-conditioning-tradeoff}
At a same-direction event, let
\(r=(\nu^\top f^+)/(\nu^\top f^-)\).  If
\(\nu^\top f^+\ge c_0>0\) and \(\|f^-\|\le M\), then
\begin{equation}
 \frac1r\le\frac{M}{c_0}.
 \label{eq:ratio-conditioning-tradeoff}
\end{equation}
The bounded-data family
\begin{equation}
 \mathcal L_R(\theta)=\tfrac12(\theta-1)^2+
 \tfrac12\bigl(\operatorname{ReLU}(\theta)+1-R^{-1}\bigr)^2
 \label{eq:bounded-data-ratio-family}
\end{equation}
is globally 1-strongly convex.  For any
\(\theta_0\in(1-e^T,0)\), its trajectory crosses zero before \(T\) and realizes
\(D\varphi_T=R^{-1}J_{\reg}\), but its outgoing speed is \(R^{-1}\).
The uniform-margin construction in Theorem~\ref{thm:residual-relu-phase}
instead scales the incoming field, targets, and initialization as
\(\Theta(R)\).  Hence an unbounded reciprocal ratio cannot coexist with both
a uniform outgoing transverse margin and a uniformly bounded incoming field
scale.
\end{lemma}

\paragraph{A qualitative outer decision.}
\begin{corollary}[Strongly convex initialization-gradient ranking reversal]
\label{cor:hypergradient-ranking-reversal}
For positive event-free sensitivities \(a_1,a_2\) and transfer ratios
\(r_1,r_2\), the strict ranking \(a_1>a_2\) reverses after event transfer
exactly when
\begin{equation}
 1<\frac{a_1}{a_2}<\frac{r_2}{r_1}.
 \label{eq:ranking-reversal-criterion}
\end{equation}
This open set is nonempty exactly when \(r_2>r_1\).  A globally 1-strongly
convex realization at \(T=1\) is
\begin{equation}
 \mathcal L(\theta)=\sum_{i=1}^2
 \left[\tfrac12(\theta_i-m_i)^2+
 \tfrac12(\operatorname{ReLU}(\theta_i)-c_i)^2\right],
 \quad
 \begin{aligned}
 (m_1,c_1)&=(4,-3),\\
 (m_2,c_2)&=(4/3,-1/3),
 \end{aligned}
 \label{eq:ranking-reversal-loss}
\end{equation}
initialized at
\begin{equation}
 \theta_0=\left(4(1-e^{0.6}),\tfrac43(1-e^{0.4})\right).
 \label{eq:ranking-reversal-initialization}
\end{equation}
Its separated events have minimum speed one, and for
\(q(z)=z_1+z_2\),
\begin{equation}
 g_{\reg}=(e^{-1.4},e^{-1.6}),\qquad
 g_{\rm flow}=(\tfrac14e^{-1.4},\tfrac34e^{-1.6}).
 \label{eq:ranking-reversal-gradients}
\end{equation}
Thus the largest-magnitude coordinate is 1 for \(g_{\reg}\) and 2 for
\(g_{\rm flow}\).  Exact discrete AD has the first ranking on all sufficiently
fine nonresonant meshes.  The itinerary and strict reversal persist on an open
neighborhood relative to the displayed separable parameter family.
\end{corollary}

\begin{corollary}[Coupled strongly convex ranking reversal]
\label{cor:coupled-ranking-reversal}
Let \(\mathcal L_0\) be the preceding loss and set
\begin{equation}
 \mathcal L_{\kappa,\varepsilon}(\theta)
 =\mathcal L_0(\theta)+\frac\kappa2\|\theta\|^2
  -\varepsilon\theta_1\theta_2.
 \label{eq:coupled-ranking-loss}
\end{equation}
There are arbitrarily small \(0<|\varepsilon_0|<\kappa_0\) for which this
finite-sample affine/residual-ReLU squared loss is genuinely coupled, globally
more than 1-strongly convex, and retains the separated itinerary and opposite
rankings.  These properties persist on an open neighborhood of that nonzero
coupling point.
\end{corollary}

A local Taylor analysis relates directional disagreement to matched-radius
regret: Proposition~\ref{prop:matched-radius-local-regret} and its proof
are in Appendix~\ref{sec:parameter-adjoint}.

The constructions prove low-dimensional magnitude and directional effects on
explicit open sets, while Proposition~\ref{prop:matched-radius-local-regret} states only the
corresponding local decision consequence. It guarantees an ordering within a
sufficiently small radius controlled by its smoothness and curvature
assumptions, but no ordering at an arbitrary prescribed radius outside that
regime and no multi-step performance ordering. It makes no typicality claim.
The proofs of the convexity and ranking statements follow below; the scalar
squared-loss realization is established in Appendix~\ref{app:relu-erm}.

\subsubsection{Convexity and the sign of an event}

\paragraph{Proof of Theorem~\ref{thm:convex-events}.}
Let \(z\) be an event point and let \(\nu\) point from the minus to the plus
region.  For sufficiently small \(\varepsilon>0\), the points
\(x_\varepsilon=z-\varepsilon\nu\) and
\(y_\varepsilon=z+\varepsilon\nu\) lie in the adjacent smooth regions.
Monotonicity of the gradient of a differentiable convex function gives
\[
 \bigl(\nabla\mathcal L(y_\varepsilon)
       -\nabla\mathcal L(x_\varepsilon)\bigr)^\top
 (y_\varepsilon-x_\varepsilon)\ge0.
\]
Divide by \(2\varepsilon\) and pass to the one-sided traces.  Since
\([\nabla\mathcal L]=\beta\nu\), the limit is \(\beta\ge0\).

Gradient structure gives
\[
 \nu^\top f^+
 =\nu^\top f^- -\beta.
\]
Same-direction transversality makes both sides positive.  Therefore
\[
 r=\frac{\nu^\top f^+}{\nu^\top f^-}
 =1-\frac{\beta}{\nu^\top f^-}\in(0,1],
\]
with \(r<1\) exactly when \(\beta>0\).  This attenuates the direct normal
mode; tangential modes remain eigenvectors with eigenvalue one.

Every regional fundamental matrix is invertible and has positive determinant
by Liouville's formula.  The matrix determinant lemma gives
\(\det\Xi_j=r_j\).  Taking determinants in the saltation-interleaved product
and in \(J_{\reg}\) yields
\begin{equation}
 \frac{\det D\varphi_T}{\det J_{\reg}}=\prod_{j=1}^m r_j,
 \qquad
 \log|\det D\varphi_T|-\log|\det J_{\reg}|
 =\sum_{j=1}^m\log r_j.
 \label{eq:app-convex-log-volume}
\end{equation}
If any convex event is strict, one \(r_j<1\) while all other factors are at
most one and positive.  The product is then less than one, so
\(D\varphi_T\ne J_{\reg}\); exact transported cancellation is impossible.

If \(\beta<0\), the same speed identity gives \(r>1\) under a
same-direction crossing.  The trace monotonicity argument above also shows
that such a downward normal-gradient jump is incompatible with global
convexity.  This is negative singular interface curvature.  The converse
statement is deliberately local: \(\beta>0\) fixes the sign of the atomic
part but does not, without convex regional Hessians, prove global convexity.
\qed

\paragraph{Proof of Corollary~\ref{cor:convex-event-volume-bound}.}
Set \(\mathcal C_T=J_{\reg}^{-1}D\varphi_T\) and
\(\delta_T=\det\mathcal C_T=\prod_j r_j\).  All regional fundamental
matrices are invertible, and every \(\Xi_j\) is invertible because
\(\det\Xi_j=r_j>0\); hence \(J_{\reg}\), \(D\varphi_T\), and
\(\mathcal C_T\) are invertible.  If
\(s_{\min}=\sigma_{\min}(\mathcal C_T)\), then
\[
 s_{\min}^d
 \le \prod_{i=1}^d\sigma_i(\mathcal C_T)
 =|\det\mathcal C_T|=\delta_T,
\]
so \(s_{\min}\le\delta_T^{1/d}\le1\).  For a unit right singular vector
\(v\) associated with \(s_{\min}\),
\[
 \|\mathcal C_T-I\|_2
 \ge \|(\mathcal C_T-I)v\|_2
 \ge 1-\|\mathcal C_Tv\|_2
 =1-s_{\min}
 \ge1-\delta_T^{1/d}.
\]
The multiplication order is fixed by the definition:
\(D\varphi_T=J_{\reg}\mathcal C_T\).  Therefore
\[
 \|D\varphi_T-J_{\reg}\|_2
 =\|J_{\reg}(\mathcal C_T-I)\|_2
 \ge \sigma_{\min}(J_{\reg})\|\mathcal C_T-I\|_2,
\]
which proves \eqref{eq:convex-volume-gap-lower-bound}.
\qed

\subsubsection{A ReLU event in residual/downstream coordinates}

\begin{corollary}[Residual/downstream ReLU event law]
\label{cor:relu-event-law}
Suppose one preactivation \(s(\theta)\) switches while all other masks remain
fixed.  Let \(q=\nabla_\theta s\ne0\), write
\(F_{\rm act}-F_{\rm inact}=a\,s+o(\|\theta-z\|)\), let \(g_y\) be the loss
gradient at the common boundary output, and set
\(\chi=\langle g_y,a\rangle\).  With empirical averaging factor \(\gamma\)
and oriented normal \(n=\sigma q\),
\[
 [\nabla_\theta\mathcal L]=\gamma\chi n,\qquad
 \beta=\gamma\chi\|q\|.
\]
For an inactive-to-active crossing,
\[
 \lambda_\perp(\Xi)
 =1-\frac{\gamma\chi\|q\|^2}{q^\top f^-}.
\]
Under same-direction transversality, the sign of \(\chi\) gives attenuation,
invisibility, or amplification.  For squared loss, \(\chi\) is the residual
projected onto the switching unit's downstream output direction.
\end{corollary}

\paragraph{Proof of Corollary~\ref{cor:relu-event-law}.}
Let \(s(\theta)\) be the changing preactivation and
\(q=\nabla_\theta s(z)\ne0\).  With every other mask fixed, write the active
and inactive network outputs near \(z\) as
\[
 F_{\rm act}(\theta)-F_{\rm inact}(\theta)
 =a\,s(\theta)+o(\|\theta-z\|).
\]
The two outputs agree at \(z\), while differentiation gives
\begin{equation}
 DF_{\rm act}(z)-DF_{\rm inact}(z)=a q^\top.
 \label{eq:app-relu-output-jump}
\end{equation}
Let \(g_y\) be the loss gradient at the common output and let \(\gamma\) be
the empirical averaging factor.  The parameter-gradient jump from inactive
to active is
\[
 \gamma\bigl(DF_{\rm act}-DF_{\rm inact}\bigr)^\top g_y
 =\gamma\langle g_y,a\rangle q.
\]
Set \(\chi=\langle g_y,a\rangle\).  For activation, the oriented normal is
\(n=q\).  For deactivation, incoming and outgoing sides are exchanged and
the oriented normal is \(n=-q\); both the gradient jump and \(n\) change
sign.  Thus, in either orientation,
\[
 [\nabla_\theta\mathcal L]=\gamma\chi n,
 \qquad
 \beta=\gamma\chi\|q\|.
\]
For an inactive-to-active event, \(n=q\), and substitution into
\eqref{eq:gradient-saltation} yields
\[
 \lambda_\perp(\Xi)
 =1-\frac{\gamma\chi\|q\|^2}{q^\top f^-}.
\]
Under same-direction transversality, the signs of \(\chi\) and \(\beta\)
therefore give attenuation, invisibility, and amplification as stated.  For
squared loss, \(g_y\) is the residual, so \(\chi\) is its inner product with
the downstream output direction \(a\).
\qed

\paragraph{Proof of Corollary~\ref{cor:hypergradient-ranking-reversal}.}
For positive \(a_i,r_i\), the two desired strict inequalities are
\(a_1/a_2>1\) and \(a_1/a_2<r_2/r_1\), proving
\eqref{eq:ranking-reversal-criterion}.  This interval is nonempty exactly
when \(r_2>r_1\).  Strict inequalities define an open subset of
\(\mathbb R_{>0}^4\).

Equation~\eqref{eq:ranking-reversal-loss} is the half-squared risk of the
one-sample, four-output network
\[
 F_\theta(1)=(\theta_1,\operatorname{ReLU}(\theta_1),
               \theta_2,\operatorname{ReLU}(\theta_2))
\]
with target \((4,-3,4/3,-1/3)\).  Each coordinate loss has \(c_i<0\), so
Proposition~\ref{prop:relu-erm} makes it globally 1-strongly convex.  Their sum
is globally 1-strongly convex on \(\mathbb R^2\).

The initialization formula in Proposition~\ref{prop:relu-erm} gives coordinate
event times \(s_1=0.6\) and \(s_2=0.4\).  The corresponding ratios are
\[
 r_1=\frac{m_1+c_1}{m_1}=\frac14,
 \qquad
 r_2=\frac{m_2+c_2}{m_2}=\frac34.
\]
Both outgoing speeds equal one; the incoming speeds are \(4\) and \(4/3\).
The events are therefore separated by \(0.2\) and have one-sided speed margin
one.  Since the system is separable, the event-free and flow initialization
Jacobians are diagonal, with entries
\[
 (J_{\reg})_{ii}=e^{-2T+s_i},
 \qquad
 (D\varphi_T)_{ii}=r_i e^{-2T+s_i}.
\]
For \(q(z)=z_1+z_2\), its terminal gradient is \((1,1)\), which proves
\eqref{eq:ranking-reversal-gradients}.  The regional ranking follows from
\(e^{-1.4}>e^{-1.6}\).  The flow ranking is reversed because
\(e^{0.2}<3\), so
\(\frac14e^{-1.4}<\frac34e^{-1.6}\).

At nonresonant meshes, the two scalar executed programs are locally smooth and
their exact derivatives converge coordinatewise to the regional values.
Because the regional inequality is strict, exact discrete AD has the same
ranking for all sufficiently small step sizes.  Finally, the event-time and
ratio formulas are continuous where \(m_i>0\), \(-m_i<c_i<0\), and the event
times lie strictly inside \((0,T)\).  The two ranking margins, event separation,
 and normal speeds are also strict.  They therefore persist on an open
 neighborhood relative to the displayed separable target--initialization
 parameter family.
\qed

\paragraph{Proof of Lemma~\ref{lem:ratio-conditioning-tradeoff}.}
Since \(\nu\) is a unit vector and both normal speeds are positive,
\[
 \frac1r=\frac{\nu^\top f^-}{\nu^\top f^+}
 \le\frac{\|f^-\|}{c_0}\le\frac{M}{c_0}.
\]
For \eqref{eq:bounded-data-ratio-family}, the parameters in
\eqref{eq:relu-erm-objective} are \(m=1\) and
\(c=-1+R^{-1}\).  Proposition~\ref{prop:relu-erm} gives global
1-strong convexity, incoming speed one, outgoing speed \(R^{-1}\), and
\(D\varphi_T=R^{-1}J_{\reg}\).  The target and initialization remain bounded
for fixed \(t_*\), but the transverse margin tends to zero.  In the
uniform-margin parameterization, \(m=Ra\), \(|c|=(R-1)a\), and
\(|\theta_0|=Ra(e^{t_*}-1)\), proving the asserted linear scale growth.
\qed

\paragraph{Proof of Corollary~\ref{cor:coupled-ranking-reversal}.}
On a fixed activation region, the smooth quadratic part of
\(\mathcal L_{\kappa,\varepsilon}\) has Hessian
\[
 Q_{\kappa,\varepsilon}
 =\begin{bmatrix}1+\kappa&-\varepsilon\\
                  -\varepsilon&1+\kappa\end{bmatrix},
 \qquad
 \lambda_{\min}(Q_{\kappa,\varepsilon})
 =1+\kappa-|\varepsilon|.
\]
The two residual-ReLU terms are convex because \(c_i<0\).  Hence
\(\mathcal L_{\kappa,\varepsilon}\) is globally
\((1+\kappa-|\varepsilon|)\)-strongly convex and is more than 1-strongly
convex whenever \(\kappa>|\varepsilon|\).

At \((\kappa,\varepsilon)=(0,0)\), the event times, normal speeds, event
separation, and both ranking margins computed above are strict.  Regional ODE
solutions depend continuously on \((\kappa,\varepsilon)\).  In each isolating
tube, the event equation has nonzero time derivative, so the
implicit-function theorem gives a unique nearby event time and state.
Induction over the two separated events shows that the itinerary, regional
fundamental matrices, saltation matrices, and both terminal gradients vary
continuously.  All strict inequalities therefore persist near the origin.
Choose a sufficiently small point with
\(0<|\varepsilon_0|<\kappa_0\), and then a ball around it small enough that
\(\varepsilon\ne0\), \(\kappa>|\varepsilon|\), and all event and ranking
margins persist.  This is an open set in the explicit coupled parameter
family.

Finally, let \(A^\top A=Q_{\kappa,\varepsilon}\) and
\(b=A^{-\top}m\), where \(m=(m_1,m_2)\).  Up to an additive constant,
\[
 \mathcal L_{\kappa,\varepsilon}(\theta)
 =\tfrac12\|A\theta-b\|^2
 +\tfrac12\sum_{i=1}^2
   (\operatorname{ReLU}(\theta_i)-c_i)^2.
\]
Thus the coupled objective is a finite-dimensional affine/residual-ReLU
squared loss.  This completes the proof.
\qed

\subsection{Multi-event cancellation criterion}
\label{app:first-variation-proof}
\label{app:multi-event}

Several nontrivial saltation matrices do not by themselves imply a matrix
mismatch: transported event effects may cancel.

\begin{proposition}[Transported saltation criterion]
\label{prop:general-mismatch}\label{prop:multi-event-quantitative}
For \(j=1,\ldots,m\), define
\begin{equation}
 Q_j=\Phi_{j-1}\Phi_{j-2}\cdots\Phi_0,
 \qquad
 \widetilde\Xi_j=Q_j^{-1}\Xi_jQ_j.
 \label{eq:general-transported-saltation}
\end{equation}
Then
\begin{equation}
 J_{\reg}^{-1}D\varphi_T
 =\widetilde\Xi_m\widetilde\Xi_{m-1}\cdots\widetilde\Xi_1.
 \label{eq:general-exact-mismatch-factorization}
\end{equation}
Consequently,
\begin{equation}
 D\varphi_T=J_{\reg}
 \quad\Longleftrightarrow\quad
 \widetilde\Xi_m\cdots\widetilde\Xi_1=I.
 \label{eq:general-exact-cancellation-criterion}
\end{equation}
A sufficient, but not necessary, condition for mismatch is
\begin{equation}
 \prod_{j=1}^m
 \frac{\nu_j^\top f_j^+}{\nu_j^\top f_j^-}\ne1.
 \label{eq:general-determinant-mismatch}
\end{equation}
Equality of the scalar product to one only makes the determinant of the
transported matrix product equal to one; it does not imply that the matrix
product is \(I\).  Regional propagation can rotate and couple the affected
directions before subsequent events.  For any submultiplicative matrix norm
normalized by \(\|I\|=1\), set
\begin{equation}
 S_j=\Xi_j-I,
 \qquad
 \widetilde S_j=Q_j^{-1}S_jQ_j,
 \qquad
 E_T=\sum_{j=1}^m\|\widetilde S_j\|.
 \label{eq:general-transported-event-budget}
\end{equation}
Then the exact log-volume correction is
\begin{equation}
 \log|\det D\varphi_T|-\log|\det J_{\reg}|
 =\sum_{j=1}^m\log r_j,
 \qquad
 r_j=\frac{\nu_j^\top f_j^+}{\nu_j^\top f_j^-}>0,
 \label{eq:general-log-volume-correction}
\end{equation}
and
\begin{align}
 \|J_{\reg}^{-1}D\varphi_T-I\|
 &\le\prod_{j=1}^m(1+\|\widetilde S_j\|)-1
 \le e^{E_T}-1,
 \label{eq:general-event-budget-bound}\\
 J_{\reg}^{-1}D\varphi_T
 &=I+\sum_{j=1}^m\widetilde S_j+R_2,
 \label{eq:general-weak-event-expansion}\\
 \|R_2\|
 &\le e^{E_T}-1-E_T
 \le\tfrac12e^{E_T}E_T^2.
 \label{eq:general-event-interaction-bound}
\end{align}
The first-order sum contains the transported event contributions; \(R_2\)
contains their ordered interactions.  These are upper bounds only.  No lower
bound follows from \(E_T\), because transported factors may cancel.
For one event,
\begin{equation}
 D\varphi_T-J_{\reg}=\Phi_1(\Xi_1-I)\Phi_0,
 \label{eq:general-single-event-gap}
\end{equation}
so equality holds if and only if \(\Xi_1=I\).  In the continuous-loss gradient
setting this is equivalent to equality of the one-sided gradients.  A single
genuine gradient jump therefore forces a matrix mismatch, although perturbation
directions satisfying \(\nu_1^\top\Phi_0v=0\) are unaffected by that event.
\end{proposition}

\begin{proof}
Every regional fundamental matrix is invertible.  Expanding
\(J_{\reg}^{-1}D\varphi_T\) and cancelling regional factors from the outside
inward gives
\[
 J_{\reg}^{-1}D\varphi_T
 =(Q_m^{-1}\Xi_mQ_m)\cdots(Q_1^{-1}\Xi_1Q_1),
\]
which proves the factorization and exact cancellation criterion.  Moreover,
\[
 \det\Xi_j
 =1+\frac{\nu_j^\top(f_j^+-f_j^-)}{\nu_j^\top f_j^-}
 =\frac{\nu_j^\top f_j^+}{\nu_j^\top f_j^-},
\]
so the determinant condition follows.  Under (A4), every ratio is positive;
taking logarithms proves \eqref{eq:general-log-volume-correction}.

Expanding the ordered product in
\eqref{eq:general-exact-mismatch-factorization} gives
\[
 (I+\widetilde S_m)\cdots(I+\widetilde S_1)
 =I+\sum_j\widetilde S_j
  +\sum_{k=2}^{m}
    \sum_{j_k>\cdots>j_1}
    \widetilde S_{j_k}\cdots\widetilde S_{j_1}.
\]
Submultiplicativity therefore yields
\begin{align*}
 \left\|(I+\widetilde S_m)\cdots(I+\widetilde S_1)-I\right\|
 &\le\prod_j(1+\|\widetilde S_j\|)-1,\\
 \|R_2\|
 &\le\prod_j(1+\|\widetilde S_j\|)-1
      -\sum_j\|\widetilde S_j\|.
\end{align*}
Using \(1+x\le e^x\) proves
\eqref{eq:general-event-budget-bound} and the first inequality in
\eqref{eq:general-event-interaction-bound}.  The final inequality follows
from \(e^x-1-x\le\frac12e^x x^2\) for \(x\ge0\).

For one event, invertibility of
\(\Phi_0\) and \(\Phi_1\) makes \eqref{eq:general-single-event-gap} zero exactly
when \(\Xi_1-I=0\).  Since the denominator is nonzero, the latter is equivalent
to \(f_1^+=f_1^-\), or \(\beta_1=0\).
\end{proof}

\subsection{Arbitrary-gap construction}
\label{app:arbitrary-gap}

\begin{corollary}[Arbitrarily large discrete-to-flow derivative gap]
\label{cor:general-two-bowl-gap}
Let \(0<a_-<a_+\) and, for
\(\theta=(x,y)\in\mathbb R\times\mathbb R^{d-1}\), define
\begin{equation}
 \mathcal L_{a_-,a_+}(x,y)
 =\frac12\|y\|^2+
 \begin{cases}
  \frac12(x-a_-)^2, & x\le0,\\[1mm]
  \frac12(x-a_+)^2+\frac12(a_-^2-a_+^2), & x\ge0.
 \end{cases}
 \label{eq:general-two-bowl-loss}
\end{equation}
For every \(T>0\), the open strip
\begin{equation}
 U_T=\left\{(x_0,y_0):a_-(1-e^T)<x_0<0\right\}
 \label{eq:general-two-bowl-strip}
\end{equation}
has a stable one-event itinerary before time \(T\).  On \(U_T\), both regional
Hessians equal \(I_d\), and
\begin{align}
 J_{\reg}&=e^{-T}I_d,
 \label{eq:general-two-bowl-regional}\\
 D\varphi_T
 &=e^{-T}\operatorname{diag}
   \left(\frac{a_+}{a_-},I_{d-1}\right).
 \label{eq:general-two-bowl-true-derivative}
\end{align}
Consequently,
\begin{align}
 \|D\varphi_T-J_{\reg}\|_2
 &=e^{-T}\left(\frac{a_+}{a_-}-1\right),
 \label{eq:general-two-bowl-additive-gap}\\
 \frac{\|D\varphi_T\|_2}{\|J_{\reg}\|_2}
 &=\frac{a_+}{a_-}.
 \label{eq:general-two-bowl-multiplicative-gap}
\end{align}
Thus any multiplicative gap \(R>1\) is obtained by taking \(a_+/a_-=R\), and
any additive gap \(G>0\) by taking \(a_+/a_-=1+Ge^T\).  If
\(a_+/a_->e^T\), the true flow is expansive although ordinary branchwise AD
converges to the strict contraction \(e^{-T}I_d\).

For \(\eta=T/N\) and any step size with no exact interface landing,
\begin{equation}
 D\Theta_{\eta,T}
 =J_{\eta,T}^{\AD}
 =(1-\eta)^NI_d\longrightarrow e^{-T}I_d.
 \label{eq:general-two-bowl-discrete-ad}
\end{equation}
Inserting the exact event matrix once gives
\begin{equation}
 J_{\eta,T}^{\SC}
 =(1-\eta)^N
 \operatorname{diag}\left(\frac{a_+}{a_-},I_{d-1}\right)
 =D\varphi_T+O(\eta).
 \label{eq:general-two-bowl-corrected}
\end{equation}
Here the inserted saltation matrix is exact.  Since \(T=N\eta\),
\((1-\eta)^N=e^{-T}+O(\eta)\), so the entire \(O(\eta)\) difference in
\eqref{eq:general-two-bowl-corrected} is the ordinary Euler error of the
regional dynamics.  There is no event-localization or one-sided-data error in
this construction; those additional errors are represented by
\(\varepsilon_\eta\) and \(\delta_\eta\) in
Theorem~\ref{thm:general-event-aware-consistency}.
The loss is regionwise strongly convex but not globally convex.
\end{corollary}

\begin{proof}
The additive constant makes the two traces agree at \(x=0\).  On the left,
\(\dot x=a_--x\), so a point with \(x_0<0\) reaches the interface at
\begin{equation}
 t_*(x_0)=\log\frac{a_--x_0}{a_-}.
 \label{eq:general-two-bowl-event-time}
\end{equation}
The strict strip inequalities are exactly \(0<t_*(x_0)<T\).  Both one-sided
speeds at impact are positive, namely \(a_-\) and \(a_+\), so the crossing is
stable and same-direction transverse.  The regional fundamental matrices are
scalar contractions whose total duration is \(T\), while
\[
 \Xi=\operatorname{diag}\left(\frac{a_+}{a_-},I_{d-1}\right).
\]
This proves the regional and true derivatives and both gap formulas.  Since
both regional Hessians are \(I_d\), every executed Euler Jacobian is
\((1-\eta)I_d\), irrespective of the branch count.  The regional factors
commute with \(\Xi\), proving the corrected formula.  Finally, the normal
derivative of \(\mathcal L\) jumps downward from \(-a_-\) to \(-a_+\), so the
continuous loss has negative singular interface curvature and is not globally
convex.
\end{proof}

\section{Parameters, adjoints, and terminal-objective consequences}
\label{sec:parameter-adjoint}
\subsection{Hyperparameter sensitivity and reverse mode}
\label{app:parameter-adjoint}

This section proves Theorem~\ref{thm:parameter-limit} and
Corollary~\ref{cor:adjoint-events}.  The formulas are classical hybrid
sensitivity identities; their role here is to compare the derivative of a
hard-ReLU training flow with the limit of exact derivatives of its
nonresonant discrete programs.

\subsubsection{Parameter-dependent assumptions}

Let \(\lambda\in\mathbb R^p\).  In addition to (A1)--(A5), suppose that each
regional field \(f_\alpha(\theta,\lambda)\) is \(C^1\) jointly in state and
parameter on the relevant compact tube.  Write
\[
 A_\alpha=D_\theta f_\alpha\in\mathbb R^{d\times d},
 \qquad
 B_\alpha=D_\lambda f_\alpha\in\mathbb R^{d\times p}.
\]
The matrices \(A_\alpha\) and \(B_\alpha\) are uniformly continuous on that
tube.  Each event function \(g_j(\theta,\lambda)\) is \(C^2\), and the
initialization is \(h(\lambda)\) with \(h\in C^1\).  The event count, order,
separation, and same-direction transversality persist on a parameter
neighborhood of the reference value.  The mesh, horizon, state dimension,
and program depth are held fixed while differentiating with respect to
\(\lambda\).

Let
\[
 Z(t)=D_\lambda\theta(t;\lambda)\in\mathbb R^{d\times p},
 \quad
 n_j=\nabla_\theta g_j\in\mathbb R^d,
 \quad
 g_{\lambda,j}=\nabla_\lambda g_j\in\mathbb R^p.
\]

\begin{lemma}[Joint branch constancy for a fixed program]
\label{lem:parameter-branch-constancy}
Fix a base parameter \(\lambda_0\) and a finite Euler/GD program.  If every
activation function evaluated by the executed program is nonzero at
\(\lambda_0\), then all executed branch signs are constant on a neighborhood
of \(\lambda_0\).  On that neighborhood the endpoint is \(C^1\), and its
classical parameter derivative is the branchwise AD recursion
\eqref{eq:app-discrete-parameter-recursion}.
\end{lemma}
\begin{proof}
The initialization and all regional update maps are continuous in
\(\lambda\).  A nonzero finite list of activation values retains its signs on
a sufficiently small neighborhood.  Induction over the finite program depth
therefore fixes the branch word.  The endpoint is then a composition of
jointly \(C^1\) regional maps, and the chain rule gives the stated recursion.
\end{proof}

\begin{theorem}[Parameter derivatives and the vanishing-step limit]
\label{thm:parameter-limit}
Fix a base parameter \(\lambda_0\), and assume the conditions above, including
persistence of the stable separated itinerary on a neighborhood of
\(\lambda_0\).  The horizon, mesh, and program depth are held fixed while
taking \(D_\lambda\).  Regionally,
\[
 \dot Z=A_\alpha Z+B_\alpha,\qquad Z(0)=Dh(\lambda).
\]
At event \(j\),
\[
 Z_j^+=\Xi_jZ_j^-+C_j,\qquad
 C_j=\frac{(f_j^+-f_j^-)g_{\lambda,j}^\top}{n_j^\top f_j^-}.
\]
At each mesh that is nonresonant at \(\lambda_0\), branchwise AD is the
classical parameter derivative.  Along any such vanishing-mesh sequence it
converges to the regional inhomogeneous sensitivity continued without event
jumps, whereas the limiting flow derivative obeys the displayed affine event
transfer.
\end{theorem}

\subsubsection{Forward jump and the discrete learning limit}

\paragraph{Proof of Theorem~\ref{thm:parameter-limit}.}
Inside a smooth region, ordinary differentiation gives
\begin{equation}
 \dot Z=A_\alpha Z+B_\alpha,
 \qquad
 Z(0)=Dh(\lambda).
 \label{eq:app-parameter-regional}
\end{equation}
At event \(j\), perturb the parameter by \(\delta\lambda\) and denote the
resulting event-time change by \(\delta t_j\).  Differentiating
\(g_j(\theta(t_j+\delta t_j;\lambda+\delta\lambda),
\lambda+\delta\lambda)=0\) on the incoming side gives
\[
 \delta t_j
 =-\frac{(n_j^\top Z_j^-+g_{\lambda,j}^\top)\delta\lambda}
          {n_j^\top f_j^-}.
\]
Returning the outgoing perturbation to the nominal clock yields
\begin{align*}
 Z_j^+\delta\lambda
 &=Z_j^-\delta\lambda+(f_j^--f_j^+)\delta t_j\\
 &=\left(I+
 \frac{(f_j^+-f_j^-)n_j^\top}{n_j^\top f_j^-}\right)
 Z_j^-\delta\lambda
 +\frac{(f_j^+-f_j^-)g_{\lambda,j}^\top}
       {n_j^\top f_j^-}\delta\lambda.
\end{align*}
Since this holds for every parameter perturbation,
\begin{equation}
 Z_j^+=\Xi_jZ_j^-+C_j,
 \qquad
 C_j=\frac{(f_j^+-f_j^-)g_{\lambda,j}^\top}
           {n_j^\top f_j^-}.
 \label{eq:app-parameter-event-jump}
\end{equation}
The dimensions are \(Z_j^\pm,C_j\in\mathbb R^{d\times p}\).

Equivalently, augment the state by the constant parameter:
\[
 y=(\theta,\lambda),
 \qquad
 \dot y=(f_\alpha(\theta,\lambda),0),
 \qquad
 G_j(y)=g_j(\theta,\lambda).
\]
The ordinary saltation matrix of the augmented system is
\begin{equation}
 \mathcal S_j=
 \begin{bmatrix}
   \Xi_j&C_j\\
   0&I_p
 \end{bmatrix},
 \label{eq:app-augmented-saltation}
\end{equation}
which gives a second derivation of both the sign and the block dimensions.

At a fixed Euler/GD program that is nonresonant at \(\lambda_0\), with
\(T,\eta,N\) held fixed while differentiating,
Lemma~\ref{lem:parameter-branch-constancy} gives the exact derivative recursion
\begin{equation}
 Z_{k+1}^\eta
 =\bigl(I+\eta A_{\alpha_k}(\theta_k^\eta,\lambda)\bigr)Z_k^\eta
  +\eta B_{\alpha_k}(\theta_k^\eta,\lambda),
 \qquad Z_0^\eta=Dh(\lambda).
 \label{eq:app-discrete-parameter-recursion}
\end{equation}
Introduce the augmented regional coefficient and sensitivity
\[
 \mathcal A(t)=
 \begin{bmatrix}A(t)&B(t)\\0&0\end{bmatrix},
 \qquad
 W(t)=\begin{bmatrix}Z(t)\\I_p\end{bmatrix}.
\]
The discrete factor corresponding to
\eqref{eq:app-discrete-parameter-recursion} is
\[
 I+\eta\mathcal A_k
 =\begin{bmatrix}I+\eta A_k&\eta B_k\\0&I_p\end{bmatrix}.
\]
Outside event windows, the state estimate and uniform continuity imply that
both \(A_k\) and \(B_k\) approach their reference regional coefficients.
Inside event windows their norms are bounded and the total window length is
\(O(\eta)\).  Consequently,
\[
 \|\mathcal A^\eta-\mathcal A\|_{L^1(0,T)}
 \le C\bigl[\eta+\omega_A(C\eta)+\omega_B(C\eta)\bigr].
\]
The same fundamental-matrix comparison used in
Theorem~\ref{thm:general-first-variation} gives
\begin{equation}
 \|Z_{\eta,T}^{\AD}-Z_{\reg}(T)\|
 \le C\bigl[\eta+\omega_A(C\eta)+\omega_B(C\eta)\bigr],
 \label{eq:app-parameter-limit-rate}
\end{equation}
where \(Z_{\reg}\) solves \eqref{eq:app-parameter-regional} and is continued
without a jump at each nominal event time.  If \(A_\alpha,B_\alpha\) are
locally Lipschitz, the right-hand side is \(O(\eta)\).

The true flow derivative instead obeys
\eqref{eq:app-parameter-event-jump}.  For one event, subtracting the
event-free continuation and propagating to \(T\) gives the useful identity
\begin{equation}
 Z_{\mathrm{flow}}(T)-Z_{\reg}(T)
 =\Phi_1
 \frac{(f^+-f^-)(n^\top Z^-+g_\lambda^\top)}{n^\top f^-}.
 \label{eq:app-one-event-parameter-gap}
\end{equation}
If \(f^+\ne f^-\), a parameter direction has zero event-induced gap exactly
when its first-order event-time change vanishes.  If \(f^+=f^-\), the gap
vanishes identically.
\qed

The formula separates three cases.  If a parameter enters only the
initialization, then \(B=0\), \(g_\lambda=0\), and events act through
\(\Xi Z^-\).  If it enters the vector field but not the interface, then
regional forcing \(B\) is accumulated and subsequently transferred by
\(\Xi\), while \(C=0\).  If it enters the interface explicitly, the affine
term \(C\) is present.  These effects coexist when a parameter enters all
three objects.

\subsubsection{Reverse mode}

\begin{corollary}[Event-aware adjoint]
\label{cor:adjoint-events}
For a terminal scalar objective \(q(\theta(T),\lambda)\), let
\(p(T)=\nabla_\theta q\) and solve \(-\dot p=A_\alpha^\top p\) regionally.
Then
\[
 p_j^-=\Xi_j^\top p_j^+,
\]
and
\[
 \nabla_\lambda q
 =q_\lambda+Dh^\top p(0)+\int_0^T B^\top p\,\dd t
  +\sum_{j=1}^m g_{\lambda,j}
   \frac{(f_j^+-f_j^-)^\top p_j^+}{n_j^\top f_j^-}.
\]
Omitting these event terms gives the adjoint of the regional sensitivity
limit, not generally the derivative of the limiting flow.
\end{corollary}

\paragraph{Proof of Corollary~\ref{cor:adjoint-events}.}
For a terminal scalar objective \(q(\theta(T),\lambda)\), set
\[
 p(T)=\nabla_\theta q(\theta(T),\lambda),
 \qquad
 -\dot p=A_\alpha^\top p
\]
on each regional segment.  The fixed-clock forward event map in
\eqref{eq:app-parameter-event-jump} pulls covectors back by
\begin{equation}
 p_j^-=\Xi_j^\top p_j^+.
 \label{eq:app-adjoint-jump}
\end{equation}
Indeed,
\begin{align*}
 (p_j^+)^\top Z_j^+
 &=(p_j^+)^\top(\Xi_jZ_j^-+C_j)\\
 &=(\Xi_j^\top p_j^+)^\top Z_j^-+(p_j^+)^\top C_j.
\end{align*}
This proves that the transpose, rather than the inverse transpose, is the
correct jump for the fixed-clock derivative map.

Between events,
\[
 \frac{\dd}{\dd t}(Z^\top p)=B^\top p.
\]
At an event, \eqref{eq:app-adjoint-jump} gives
\[
 (Z_j^+)^\top p_j^+-(Z_j^-)^\top p_j^-=C_j^\top p_j^+.
\]
Integrating segment by segment and adding the jumps yields
\begin{equation}
 \nabla_\lambda q
 =q_\lambda+Dh(\lambda)^\top p(0)
  +\int_0^T B(t)^\top p(t)\,\dd t
  +\sum_{j=1}^m C_j^\top p_j^+,
 \label{eq:app-adjoint-parameter-gradient}
\end{equation}
where
\begin{equation}
 C_j^\top p_j^+
 =g_{\lambda,j}
  \frac{(f_j^+-f_j^-)^\top p_j^+}{n_j^\top f_j^-}.
 \label{eq:app-adjoint-event-accumulator}
\end{equation}
If \(g_{\lambda,j}=0\), the direct event accumulator vanishes, but the
event can still change the gradient through the jump
\(p_j^-=\Xi_j^\top p_j^+\).  Omitting all event terms gives the regional
continuous adjoint and hence the transpose of the event-free sensitivity
limit.  Reverse-mode AD through every fixed nonresonant discrete program
remains its exact classical derivative.
\qed

\subsection{Event-aware first-variation consistency}
\label{app:corrected-consistency}
\label{app:event-aware}

Let \(\widehat t_j^\eta\) and \(\widehat z_j^\eta\) be a posteriori event
times and states with the correct event count and order.  Suppose
\begin{equation}
 \max_j\left(
  |\widehat t_j^\eta-t_j|+\|\widehat z_j^\eta-z_j\|
 \right)\le\delta_\eta.
 \label{eq:general-event-localization-error}
\end{equation}
Evaluate one-sided data at the localized event and define
\begin{equation}
 \widehat\Xi_j^\eta
 =I+\frac{(\widehat f_j^+-\widehat f_j^-)(\widehat n_j)^\top}
 {\widehat n_j^\top\widehat f_j^-}.
 \label{eq:general-numerical-saltation}
\end{equation}
Let \(\widehat\Phi_j^\eta\) be approximations of the regional propagators and
assume
\begin{equation}
 \max_j\|\widehat\Xi_j^\eta-\Xi_j\|\le\varepsilon_\eta,
 \qquad
 \max_j\|\widehat\Phi_j^\eta-\Phi_j\|\le\gamma_\eta.
 \label{eq:general-event-factor-errors}
\end{equation}
Define the a posteriori corrected product
\begin{equation}
 J_{\eta,T}^{\SC}
 =\widehat\Phi_m^\eta\widehat\Xi_m^\eta
  \widehat\Phi_{m-1}^\eta\widehat\Xi_{m-1}^\eta\cdots
  \widehat\Phi_1^\eta\widehat\Xi_1^\eta\widehat\Phi_0^\eta.
 \label{eq:general-event-aware-product}
\end{equation}

\begin{theorem}[Consistency of an a posteriori event-corrected product]
\label{thm:general-event-aware-consistency}
Under (A1)--(A5), if the supplied event data have the correct stable
itinerary and \(\delta_\eta+\varepsilon_\eta+\gamma_\eta\to0\), then
\begin{equation}
 J_{\eta,T}^{\SC}\longrightarrow D\varphi_T(\theta_0).
 \label{eq:general-event-aware-convergence}
\end{equation}
More quantitatively,
\begin{equation}
 \|J_{\eta,T}^{\SC}-D\varphi_T(\theta_0)\|
 \le C\left[
  \delta_\eta+\varepsilon_\eta+\gamma_\eta
 \right].
 \label{eq:general-event-aware-rate}
\end{equation}
The corrected matrix is not, in general, the derivative of unmodified
hard-branch GD.  It is the derivative of an event-aware solver only when that
solver differentiates the localized event time and split steps.
\end{theorem}

\begin{proof}
All exact and numerical factors are uniformly bounded, and their number is
fixed.  A telescoping perturbation bound for the finite product, together with
\eqref{eq:general-saltation-product}, proves
\eqref{eq:general-event-aware-rate} and the stated conclusions.
\end{proof}

For the intended split hard-GD construction, use the \(O(\eta)\)-accurate
state path from Theorem~\ref{thm:general-first-variation}, localize each event
with error \(\delta_\eta\), and evaluate the one-sided fields and normals
consistently.  Transversality then gives
\(\varepsilon_\eta\le C(\delta_\eta+\text{one-sided-data error})\), while
Lemmas~\ref{lem:general-branch-mismatch}--\ref{lem:general-fundamental} give
\[
 \gamma_\eta\le
 C\bigl[\eta+\omega_A(C\eta)+\delta_\eta\bigr].
\]
This specialization recovers the earlier convergence rate without treating
the theorem as a complete event-detection algorithm.

\subsection{Scalar outer gradients and local candidate regret}
\begin{corollary}[Scalar outer-gradient consequence]
\label{cor:hypergradient-gap}
Under Theorem~\ref{thm:residual-relu-phase}, along any nonresonant mesh
sequence, let \(q\in C^1(\mathbb R)\) satisfy
\(q'(\varphi_T(\theta_0))\ne0\), let
\(G_\eta=q\circ\Theta_{\eta,T}\), and let
\(G=q\circ\varphi_T\).  Then
\begin{equation}
 DG_\eta\to q'(\varphi_T)J_{\reg},\qquad
 DG=r\,q'(\varphi_T)J_{\reg}.
 \label{eq:hypergradient-limits}
\end{equation}
\end{corollary}

\paragraph{Proof of Corollary~\ref{cor:hypergradient-gap}.}
At every nonresonant step size, the ordinary chain rule gives
\[
 DG_\eta(\theta_0)
 =q'(\Theta_{\eta,T}(\theta_0))D\Theta_{\eta,T}(\theta_0).
\]
Proposition~\ref{prop:relu-erm}, continuity of \(q'\), and state convergence
therefore yield
\(DG_\eta(\theta_0)\to q'(\varphi_T(\theta_0))J_{\reg}\).
The stable-itinerary flow map is classically differentiable, so
\[
 DG(\theta_0)=q'(\varphi_T(\theta_0))D\varphi_T(\theta_0)
 =r q'(\varphi_T(\theta_0))J_{\reg}.
\]
Both limits are nonzero because \(q'(\varphi_T(\theta_0))\ne0\) and
\(J_{\reg}>0\).  This proves the stated factor-\(r\) relation.

At resonance, an evaluated iterate equals zero.  If another step remains, the
incoming and outgoing updates from that state differ by \(\eta c\), which
is propagated through the remaining outgoing recurrence.  The endpoint map
therefore need not have a classical derivative; software AD returns a selected
branch product.  The state estimate still holds under the admissible
positive-speed selection.  This is why every derivative statement above,
including Theorem~\ref{thm:residual-relu-phase} and
Corollary~\ref{cor:hypergradient-gap}, follows a nonresonant sequence.

\begin{proposition}[Matched-radius local regret]
\label{prop:matched-radius-local-regret}
Fix a nonresonant mesh and let the actual finite-program outer objective
\(Q_\eta\) be \(C^2\) on an open neighborhood of
\(\overline B(\lambda,\epsilon_0)\).  Write
\(g_\eta=\nabla Q_\eta(\lambda)\ne0\), let \(h\ne0\) be a surrogate gradient,
and set \(u=g_\eta/\|g_\eta\|_2\), \(v=h/\|h\|_2\).  If
\(\sup_{x\in\overline B(\lambda,\epsilon_0)}
\|\nabla^2Q_\eta(x)\|_2\le M\), then, for
\(0\le\epsilon\le\epsilon_0\),
\begin{equation}
 \left|Q_\eta(\lambda-\epsilon v)-Q_\eta(\lambda-\epsilon u)
 -\epsilon\|g_\eta\|_2(1-u^\top v)\right|
 \le M\epsilon^2.
 \label{eq:matched-radius-local-regret}
\end{equation}
If \(u^\top v<1\), the exact discrete-gradient proposal has strictly lower
\(Q_\eta\) for all sufficiently small positive radii (in particular, for
\(0<\epsilon<\min\{\epsilon_0,
\|g_\eta\|_2(1-u^\top v)/M\}\) when \(M>0\)).  If \(v=u\), the two
matched-radius proposals coincide, regardless of their gradient-norm ratio.
\end{proposition}

\paragraph{Proof of Proposition~\ref{prop:matched-radius-local-regret}.}
For any unit vector \(w\) and \(0\le\epsilon\le\epsilon_0\), Taylor's formula
with integral remainder gives
\[
 Q_\eta(\lambda-\epsilon w)
 =Q_\eta(\lambda)-\epsilon g_\eta^\top w+R_w(\epsilon),
 \qquad
 R_w(\epsilon)=\epsilon^2\int_0^1(1-t)
 w^\top\nabla^2Q_\eta(\lambda-t\epsilon w)w\,\dd t.
\]
The Hessian bound therefore gives
\(|R_w(\epsilon)|\le \tfrac12M\epsilon^2\).  Apply this identity with
\(w=v\) and \(w=u\), subtract, and use
\(g_\eta^\top u=\|g_\eta\|_2\) and
\(g_\eta^\top v=\|g_\eta\|_2u^\top v\).  The combined remainder satisfies
\[
 |R_v(\epsilon)-R_u(\epsilon)|
 \le |R_v(\epsilon)|+|R_u(\epsilon)|
 \le M\epsilon^2,
\]
which proves \eqref{eq:matched-radius-local-regret}.  If \(u^\top v<1\), its
right-hand first-order coefficient is positive; the displayed threshold makes
the lower Taylor bound strict when \(M>0\), while the conclusion is immediate
when \(M=0\).  If \(u=v\), the two proposal points are identical, so their
objective values agree exactly.  Nonresonance is used only to make the finite
program's branch word locally constant and hence \(g_\eta\) its classical
outer gradient; the proposition separately assumes the stated \(C^2\) regularity.
\qed

\section{Scalar smoothing and discretization}
\subsection{Smoothing and discretization theory}
\label{app:smoothing}
\label{app:smoothing-noncommutation}

This section gives a self-contained scalar result that separates three effects:
truncating a smoothing profile, discretizing its inner dynamics, and taking the
sharp-interface limit.  It also records the phase qualification needed for the
sharp-first limit.  Throughout, the update is written with a plus sign because
the vector field already includes the negative gradient.

\subsubsection{Layer assumptions, rescaling, and entry--exit convention}
\label{app:smoothing:setup}

Let $F\colon\mathbb{R}\to\mathbb{R}$ satisfy the following standing
assumptions:
\begin{enumerate}
    \item[(S1)] $F\in C^2(\mathbb{R})$ and there are constants
    $0<m\leq M<\infty$ such that $m\leq F(u)\leq M$ for every $u$;
    \item[(S2)] the positive traces
    \[
        a_-:=\lim_{u\to-\infty}F(u),
        \qquad
        a_+:=\lim_{u\to+\infty}F(u)
    \]
    exist;
    \item[(S3)] $L_F:=\lVert F'\rVert_{\infty}<\infty$ and
    $F''\in L^1(\mathbb{R})$.
\end{enumerate}
For the sharp-first result we additionally assume
\begin{equation}
    |uF'(u)|\longrightarrow0
    \quad\text{as }u\to\pm\infty.
    \label{eq:smoothing-tail-derivative}
\end{equation}
Assumption (S1) is the scalar same-direction transversality condition.  In
particular, neither grazing nor reversal of the continuous trajectory occurs
inside the layer.  The scalar saltation factor is
\begin{equation}
    \Xi:=\frac{a_+}{a_-}>0.
    \label{eq:smoothing-scalar-saltation}
\end{equation}

For a smoothing width $\tau>0$, define
\begin{equation}
    f_\tau(x):=F(x/\tau),
    \qquad
    x_{k+1}=x_k+\eta f_\tau(x_k).
    \label{eq:smoothing-physical-euler}
\end{equation}
The inner variables
\begin{equation}
    u_k:=x_k/\tau,
    \qquad
    \rho:=\eta/\tau
    \label{eq:smoothing-inner-variables}
\end{equation}
give the exact map and fixed-itinerary cocycle
\begin{equation}
    u_{k+1}=u_k+\rho F(u_k),
    \qquad
    \delta u_{k+1}
    =\bigl(1+\rho F'(u_k)\bigr)\delta u_k.
    \label{eq:smoothing-inner-map}
\end{equation}

We use the following entry--exit convention for the canonical layer.  Given a
normalized half-window $R>0$, set $u_0=-R$ and
\begin{equation}
    N_{\rho,R}:=\min\{n\geq1:u_n\geq R\},
    \qquad
    J_{\rho,R}:=
    \prod_{k=0}^{N_{\rho,R}-1}
    \bigl(1+\rho F'(u_k)\bigr).
    \label{eq:smoothing-stopped-cocycle}
\end{equation}
The index is finite because every increment is at least $\rho m$.  The product
in \eqref{eq:smoothing-stopped-cocycle} is the derivative of the
$N_{\rho,R}$-fold iterate on a phase cell on which $N_{\rho,R}$ is locally
constant.  At a phase where the stopping index changes, the stopped map need
not be differentiable; no derivative across such a phase boundary is claimed.
The physical tube is $[-\tau R,\tau R]$.  A localized interface limit requires
both $R\to\infty$ and $\tau R\to0$.

\subsubsection{Exact flow transfer and resolved Euler transfer}
\label{app:smoothing:resolved}

\begin{lemma}[Exact scalar flow transfer]
\label{lem:smoothing-exact-flow}
For the inner flow and its variational equation,
\[
    \dot u=F(u),
    \qquad
    \dot v=F'(u)v,
\]
the transfer from $u=-R$ to $u=R$ is
\begin{equation}
    J_R^{\mathrm{flow}}=\frac{F(R)}{F(-R)}.
    \label{eq:smoothing-exact-flow-transfer}
\end{equation}
Consequently, $J_R^{\mathrm{flow}}\to\Xi$ as $R\to\infty$.
\end{lemma}

\begin{proof}
Along the trajectory,
\[
    \frac{d}{dt}\log v=F'(u)
    =\frac{d}{dt}\log F(u).
\]
Integrating between the entry and exit points proves
\eqref{eq:smoothing-exact-flow-transfer}.  The trace limit follows from (S2).
\end{proof}

Define the logarithmic tail truncation
\begin{equation}
    \delta_F(R):=
    \left|\log\frac{F(R)}{a_+}\right|
    +
    \left|\log\frac{F(-R)}{a_-}\right|.
    \label{eq:smoothing-tail-error}
\end{equation}

\begin{theorem}[Resolved discrete layer]
\label{thm:smoothing-resolved}
Under \emph{(S1)--(S3)}, suppose that
\begin{equation}
    \rho L_F\leq\frac12.
    \label{eq:smoothing-resolution-condition}
\end{equation}
Then
\begin{equation}
    \left|
        \log J_{\rho,R}-\log\Xi
    \right|
    \leq C_F\rho+\delta_F(R),
    \label{eq:smoothing-resolved-log-bound}
\end{equation}
where one admissible constant is
\begin{equation}
    C_F:=
    \frac{\lVert F''\rVert_{L^1}}
         {\min\{1/2,m/M\}}
    +\frac{ML_F}{m}.
    \label{eq:smoothing-resolved-constant}
\end{equation}
In particular, if $\rho_j\to0$ and $R_j\to\infty$, then
$J_{\rho_j,R_j}\to\Xi$.
\end{theorem}

\begin{proof}
Write $N=N_{\rho,R}$ and, for $0\leq k<N$, set
\[
    A_k:=1+\rho F'(u_k),
    \qquad
    B_k:=\frac{F(u_{k+1})}{F(u_k)}.
\]
Because $u_{k+1}-u_k=\rho F(u_k)$, Taylor's formula with integral remainder
gives
\begin{align*}
    |B_k-A_k|
    &\leq
    \frac{1}{F(u_k)}
    \int_{u_k}^{u_{k+1}}
       (u_{k+1}-s)|F''(s)|\,ds \\
    &\leq
    \rho\int_{u_k}^{u_{k+1}}|F''(s)|\,ds.
\end{align*}
Condition \eqref{eq:smoothing-resolution-condition} gives $A_k\geq1/2$,
whereas (S1) gives $B_k\geq m/M$.  With
$c:=\min\{1/2,m/M\}$, the mean-value theorem for the logarithm yields
\[
    |\log A_k-\log B_k|
    \leq c^{-1}|A_k-B_k|.
\]
The intervals $[u_k,u_{k+1}]$ are disjoint.  Summing and using the telescoping
product of the $B_k$ therefore gives
\begin{equation}
    \left|
    \log J_{\rho,R}
    -\log\frac{F(u_N)}{F(-R)}
    \right|
    \leq
    \frac{\rho\lVert F''\rVert_{L^1}}{c}.
    \label{eq:smoothing-telescoping-bound}
\end{equation}
The exit overshoot obeys
\[
    0\leq u_N-R\leq\rho M.
\]
Since $|(\log F)'|\leq L_F/m$,
\begin{equation}
    |\log F(u_N)-\log F(R)|
    \leq\frac{ML_F}{m}\rho.
    \label{eq:smoothing-overshoot-bound}
\end{equation}
Combining \eqref{eq:smoothing-telescoping-bound},
\eqref{eq:smoothing-overshoot-bound}, and
\eqref{eq:smoothing-tail-error} proves the claim.
\end{proof}

\begin{remark}[Rate, entry phase, and physical localization]
\label{rem:smoothing-resolved-scope}
If $E_{\rho,R}:=C_F\rho+\delta_F(R)$, then
\[
    |J_{\rho,R}-\Xi|
    \leq \Xi e^{E_{\rho,R}}E_{\rho,R}.
\]
An entry displacement $|u_0+R|=O(\rho)$ adds only another $O(\rho)$ endpoint
term to the proof.  If $\delta_F(R)\leq C e^{-\gamma R}$, the choice
$R(\rho)\geq\gamma^{-1}\log(1/\rho)$ gives an $O(\rho)$ bound, provided an
embedded interface also satisfies $\tau R(\rho)\to0$.  For fixed $R$, however,
the resolved limit is $F(R)/F(-R)$, not exactly $\Xi$; a fixed-window error
must not be reported as Euler discretization error.
\end{remark}

\subsubsection{The two orders of limits}
\label{app:smoothing:iterated}

Iterated limits are meaningful only after fixing a common state map.  Fix
$x_0<0$ and $T>-x_0/a_-$, take $\eta=T/n$, and define
\begin{align}
    x_{k+1}^{\eta,\tau}
    &=x_k^{\eta,\tau}
      +\eta F(x_k^{\eta,\tau}/\tau),
      \qquad x_0^{\eta,\tau}=x_0,
      \label{eq:smoothing-global-euler}\\
    J_{\eta,\tau}
    &=\prod_{k=0}^{n-1}
      \left(1+\frac{\eta}{\tau}
      F'(x_k^{\eta,\tau}/\tau)\right).
      \label{eq:smoothing-global-jacobian}
\end{align}
The horizon assumption places the limiting sharp trajectory strictly on the
positive side at time $T$.

\begin{theorem}[Smooth-first iterated limit]
\label{thm:smoothing-smooth-first}
Under \emph{(S1)--(S3)},
\begin{equation}
    \lim_{\tau\downarrow0}
    \left[
        \lim_{n\to\infty}J_{T/n,\tau}
    \right]
    =\frac{a_+}{a_-}=\Xi.
    \label{eq:smoothing-smooth-first-limit}
\end{equation}
\end{theorem}

\begin{proof}
For fixed $\tau$, the vector field $f_\tau$ is smooth and globally Lipschitz.
Standard Euler convergence applied jointly to the state and variational
equations gives
\[
    \lim_{n\to\infty}J_{T/n,\tau}
    =D\varphi_T^\tau(x_0),
\]
where $\varphi_t^\tau$ is the flow of $\dot x=F(x/\tau)$.  The scalar
autonomous identity from Lemma~\ref{lem:smoothing-exact-flow} gives
\begin{equation}
    D\varphi_T^\tau(x_0)
    =
    \frac{F(\varphi_T^\tau(x_0)/\tau)}{F(x_0/\tau)}.
    \label{eq:smoothing-smooth-flow-derivative}
\end{equation}

For any reached point $x$, elapsed time is characterized by
\[
    t=\int_{x_0}^{x}\frac{dz}{F(z/\tau)}.
\]
The integrand is bounded by $1/m$ and converges pointwise away from zero to
$a_-^{-1}$ on the negative side and $a_+^{-1}$ on the positive side.
Dominated convergence, followed by monotonic inversion of the time-of-flight
map, shows that $\varphi_T^\tau(x_0)$ converges to the sharp trajectory endpoint
\[
    a_+\left(T+\frac{x_0}{a_-}\right)>0.
\]
Thus the denominator and numerator in
\eqref{eq:smoothing-smooth-flow-derivative} converge to $a_-$ and $a_+$,
respectively.
\end{proof}

For the opposite order, define the sharp Euler orbit
\begin{equation}
    y_{k+1}=y_k+\eta f_0(y_k),
    \qquad
    f_0(y)=
    \begin{cases}
       a_-,&y<0,\\
       a_+,&y>0,
    \end{cases}
    \qquad y_0=x_0.
    \label{eq:smoothing-sharp-euler}
\end{equation}
The value of $f_0(0)$ is irrelevant when the orbit is nonresonant.

\begin{theorem}[Nonresonant sharp-first limit]
\label{thm:smoothing-sharp-first}
Assume \emph{(S1)--(S3)} and
\eqref{eq:smoothing-tail-derivative}.  Fix $\eta=T/n$ and suppose that
\begin{equation}
    y_k\neq0,
    \qquad 0\leq k\leq n-1.
    \label{eq:smoothing-nonresonance}
\end{equation}
Then
\begin{equation}
    \lim_{\tau\downarrow0}J_{\eta,\tau}=1.
    \label{eq:smoothing-fixed-eta-sharp-limit}
\end{equation}
Consequently, for the set $\mathcal{E}_{\mathrm{nr}}$ of step sizes satisfying
\eqref{eq:smoothing-nonresonance},
\begin{equation}
    \lim_{\substack{\eta\downarrow0\\
                     \eta\in\mathcal{E}_{\mathrm{nr}}}}
    \left[
       \lim_{\tau\downarrow0}J_{\eta,\tau}
    \right]
    =1.
    \label{eq:smoothing-sharp-first-limit}
\end{equation}
Before the first crossing, condition \eqref{eq:smoothing-nonresonance} is
equivalent to $-x_0/(\eta a_-)\notin\mathbb{N}$.
\end{theorem}

\begin{proof}
For fixed $\eta$, the sharp orbit has finitely many states and
\[
    d_\eta:=\min_{0\leq k<n}|y_k|>0.
\]
Induction in \eqref{eq:smoothing-global-euler}, using the trace convergence in
(S2) at each state away from zero, gives
$x_k^{\eta,\tau}\to y_k$ for every $k$.  Hence
$|x_k^{\eta,\tau}|\geq d_\eta/2$ for all sufficiently small $\tau$.
For each of the finitely many Jacobian factors,
\begin{align*}
    \left|\frac{\eta}{\tau}
      F'(x_k^{\eta,\tau}/\tau)\right|
    &=
    \frac{\eta}{|x_k^{\eta,\tau}|}
    \left|
      \frac{x_k^{\eta,\tau}}{\tau}
      F'(x_k^{\eta,\tau}/\tau)
    \right|\\
    &\longrightarrow0
\end{align*}
by \eqref{eq:smoothing-tail-derivative}.  Their product therefore converges to
one.  The outer limit in \eqref{eq:smoothing-sharp-first-limit} is then the
limit of the constant inner-limit value one over nonresonant step sizes.
\end{proof}

\begin{remark}[Why the phase qualification is necessary]
\label{rem:smoothing-resonance}
For the tanh profile below, $F'(0)\neq0$ and the tails are exponential.  If a
sharp iterate satisfies $y_j=0$, all preceding smooth-versus-sharp state errors
are exponentially small in $1/\tau$, so
$x_j^{\eta,\tau}/\tau\to0$.  The corresponding factor behaves as
\[
    1+\frac{\eta}{\tau}F'(0),
\]
whose magnitude diverges.  Resonant step sizes can accumulate at zero.  Thus
an unrestricted statement
$\lim_{\eta\to0}\lim_{\tau\to0}J_{\eta,\tau}=1$ is generally false.  The
proved replacement is the directed nonresonant limit
\eqref{eq:smoothing-sharp-first-limit}.
\end{remark}

\subsubsection{Finite-ratio transfer depends on phase and profile}
\label{app:smoothing:critical}

At fixed $\rho$, equation \eqref{eq:smoothing-inner-map}, rather than the
trace ratio alone, is the limiting inner problem.

\begin{proposition}[Phase and profile are genuine data]
\label{prop:smoothing-phase-profile}
At finite $\rho$, the fixed-itinerary cocycle need not be determined by
$(a_-,a_+,\rho)$.
\end{proposition}

\begin{proof}
Phase dependence follows already from the explicit tanh calculations in
Subsection~\ref{app:smoothing:tanh}: the same profile and $\rho$ give transfers
near $46$ and $1$ for two different incoming phases.

For profile dependence, fix a finite itinerary
$u_0,\ldots,u_N$ away from a stopping-index boundary.  Choose a smooth bump
$h$ supported in a small neighborhood of $u_0$, disjoint from all later
sampled states, such that $h(u_0)=0$ and $h'(u_0)\neq0$.  For sufficiently
small $\varepsilon$, the function
$\widetilde F=F+\varepsilon h$ remains positive and has the same traces.  Every
sampled state is unchanged because $\widetilde F(u_k)=F(u_k)$, but the first
cocycle factor changes because
$\widetilde F'(u_0)\neq F'(u_0)$.  Choose the base itinerary so that the
product of the remaining factors is nonzero---for example, for the increasing
tanh profile below all factors are positive.  The total transfer then changes
while the traces do not.
\end{proof}

\subsubsection{Tanh-profile constructions}
\label{app:smoothing:tanh}

Consider
\begin{equation}
    F(u)=a_-+(a_+-a_-)
    \frac{1+\tanh(u/2)}{2},
    \qquad
    F'(u)=\frac{a_+-a_-}{4}\,\mathrm{sech}^2(u/2).
    \label{eq:smoothing-tanh-profile}
\end{equation}
This profile satisfies all assumptions above.  In the expanding example
$a_-=0.2$ and $a_+=2$, so $\Xi=10$ and
$\lVert F'\rVert_\infty=0.45$.

\paragraph{Skip.}
With $u_0=-20$, exit threshold $20$, and $\rho=205$,
\[
    u_1=-20+205F(-20)
       =21.00000076056569>20.
\]
There is one factor and
\begin{equation}
    J=1+205F'(-20)=1.0000007605656835.
    \label{eq:smoothing-tanh-skip}
\end{equation}
This is a constructive skip phase, not a universal large-$\rho$ limit.

\paragraph{Tuned hit.}
Choose
\begin{equation}
    \rho_*:=\frac{20}{F(-20)}
    =99.99999814496176.
    \label{eq:smoothing-tanh-hit-rho}
\end{equation}
Then $u_1=0$ and
$u_2=\rho_*F(0)=109.99999795946>20$.  Thus
\begin{equation}
    J=
    \bigl(1+\rho_*F'(-20)\bigr)
    \bigl(1+\rho_*F'(0)\bigr)
    =46.0000162315841.
    \label{eq:smoothing-tanh-hit}
\end{equation}
More generally, taking $u_0=-R$ and
$\rho_R=R/F(-R)$ lands exactly at the maximizer $u=0$.  For the increasing
profile the following step exits $+R$ for all sufficiently large $R$, and
\[
    J_{\rho_R,R}\sim\rho_R F'(0)\longrightarrow\infty.
\]
The phase is retuned with $R$ (and hence with $\rho_R$); this is not a
fixed-phase divergence statement.

\paragraph{Same ratio, different phases.}
At $\rho=100$ with exit threshold $20$, direct iteration gives
\begin{equation}
    u_0=-20:\quad J=46.00001706635,
    \qquad
    u_0=-10:\quad J=1.01634222248.
    \label{eq:smoothing-tanh-phase-comparison}
\end{equation}
This supplies the phase comparison used in
Proposition~\ref{prop:smoothing-phase-profile}.

\paragraph{Orientation reversal of the discrete cocycle.}
For the contracting twin $a_-=2$, $a_+=0.2$, take $u_0=-20$, $R=20$, and
$\rho=5$.  The orbit contains
\[
    u_1=-10.00000001855,
    \qquad
    u_2=-0.00040859936,
    \qquad
    u_3=5.50051074919.
\]
At the near-central sample,
\begin{equation}
    1+5F'(u_2)=-1.24999990609<0.
    \label{eq:smoothing-tanh-negative-factor}
\end{equation}
The two preceding factors are positive.  For $u\geq u_3$, the magnitude of
the negative derivative in \eqref{eq:smoothing-tanh-profile} decreases, and
all remaining factors are at least $0.9635$.  Therefore exactly one sampled
factor is negative and
\begin{equation}
    J=-1.17941975827.
    \label{eq:smoothing-tanh-orientation}
\end{equation}
The continuous same-direction saltation factor is $a_+/a_-=0.1>0$; the sign
reversal is a finite-step critical-regime effect.  The inequality
$\rho\max(-F')>1$ merely makes a negative factor possible.  It does not force a
given phase to sample that factor, nor does it determine the parity of all
negative factors in a longer itinerary.

\subsubsection{Effective resolution and scope beyond the scalar flat layer}
\label{app:smoothing:scope}

The local dimensionless stiffness is
\begin{equation}
    \rho_{\mathrm{eff}}
    :=\eta\lVert f_\tau'\rVert_\infty
    =\rho\lVert F'\rVert_\infty.
    \label{eq:smoothing-effective-rho}
\end{equation}
For \eqref{eq:smoothing-tanh-profile} with traces $0.2$ and $2$,
$\rho_{\mathrm{eff}}=0.45\rho$.  The condition
$\rho_{\mathrm{eff}}\leq1/2$ is a sufficient condition for
Theorem~\ref{thm:smoothing-resolved}; it is not a necessary condition for an
accidental finite-$\rho$ value to lie near $\Xi$.  Across a family of profiles,
$\rho_{\mathrm{eff}}$ also does not replace the tail error, positivity bounds,
or $\lVert F''\rVert_{L^1}$ appearing in
\eqref{eq:smoothing-resolved-log-bound}.  The contrast quantity
$\rho|a_+-a_-|/a_-$ measures net trace gain, not local spike resolution.

The preceding proofs are scalar.  They extend without new ideas only when a
multidimensional system genuinely reduces to an autonomous normal coordinate,
for example at a flat interface with fixed normal, a rank-one normal jump, and
separately controlled tangential dynamics.  In that setting the scalar normal
factor is embedded in the usual matrix
\[
    I+\frac{(f^+-f^-)\nu^\top}{\nu^\top f^-}.
\]
The scalar proof does not establish a general matrix theorem: matrix cocycles
do not telescope through a scalar speed ratio, normal--tangential blocks need
not commute, and tangential drift can change both interface location and
incoming phase.  For a curved interface, flattening changes the normal across
the physical tube of radius $\varepsilon=\tau R$ and introduces curvature
terms.  Any claimed $O(\varepsilon)+O(\eta/\tau)$ bound therefore requires a
separate geometric argument with bounded curvature, transverse speed,
tangential regularity, and $\varepsilon\to0$.  Grazing, sliding, chattering,
and simultaneous interface hits are outside the scope of this section.

\paragraph{Logical scope.}
The results above do not assert an unrestricted sharp-first iterated limit, an
exact saltation value at fixed $R$, an $O(\rho)$ error without a tail term, or
universality of $\eta/\tau$ over arbitrary smoothing families.  They also do
not assert that every unresolved phase skips the layer, that the tuned-hit
divergence holds at fixed phase, or that
$\rho\max(-F')>1$ forces orientation reversal.  Finally, the scalar argument
does not by itself prove a curved-interface or general multidimensional
extension.

\section{Experimental protocols and evidence coverage}
\subsection{Finite-radius experiments: configuration, inputs and populations}
\label{app:response-experiments}
All experiments compare the same terminal objective at two specified
initializations. The datasets, seeds, objectives and directions are fixed.
Deterministic mesh populations and seed--mesh quartiles are descriptive
summaries, not independent statistical samples or confidence intervals.
Results, scripts, full error distributions and file manifests accompany
the paper. Scalar and affine analyses reuse preserved arrays. The smooth branch-response predictor (Branch) was computed from the same
reference trajectories and previously evaluated candidate pairs. Here
``smooth'' refers to the regional expansion, not to smoothing the activation.

\paragraph{E1: exact scalar response.}
The scalar risk has target $a=4$, residual target $c=-3$, $T=1$,
initialization $x=4(1-e^{\sqrt2/4})$, terminal objective
$q(y)=(y-1)^2/2$ and direction one. All $N=200,\ldots,10000$ are
used at $\epsilon=0.3/N$: 9,801 total rows, 9,801 finite responses,
and 1,034 proposal reversals. The limiting negative and positive levels
are $-0.1228170958$ and $0.7517137144$; their predicted positive
frequency is $0.1053282752$, versus $1034/9801$ observed.
The complete phase-count classification, rather than near-resonant
quotient growth, determines this comparison. The smallest critical-endpoint
margins select 20 precision checks; independent 60-digit Euler differs
from the recorded response by at most $2.27\times10^{-12}$.
Integer-length tests use $\ell=1,2,3$ at
$N=200,1000,10000,100000$. At the finest grid the response errors are
below $10^{-5}$. Absolute objective differences remain in every row.

\paragraph{E2: fixed coupled risk and feedback-active subset.}
The two-event risk and its proven geometry are in
Appendix~\ref{app:coupled-response}. The coupling values are $0$ and
$0.4$, direction $e_1$, horizon one and
$q(\theta)=(\theta_1-1)^2/2$. The original 72 comparisons cross
$N=200,500,1000,2000,5000,10000,20000,50000,100000$ with
$\zeta=0.1,0.3,1,3$. The dense evaluation crosses all
$N=1000,\ldots,5000$ with $\zeta=0.3,1,3$: 24,006 rows.
The coarse evaluation crosses $N=20,30,50,75,100,150,200,300,500,750,1000$
with $\epsilon=0.001,0.003,0.01,0.03,0.1$: 110 rows.
These populations were fixed before their respective evaluations. The
second-order refinement was developed after the first-order results,
using the whole population and the unchanged task; it is a method
development study rather than an independent confirmatory replication.

The no-upstream-feedback ablation keeps the regional propagators, Euler
bias and transported endpoint contributions but removes the first-event
rounding contribution from the second threshold. The informative subset
is defined by a difference in the predicted signed second-index count
$k_2(z_+)-k_2(z_-)$ between full and ablated recursions. This test uses
no rerun objective. It selects 224 of 12,003 coupled rows. The full
recursion has lower error on 223, the ablation on one; there are no ties
at tolerance $10^{-12}$. Full mean/median absolute error is
$5.93829\times10^{-5}/3.87803\times10^{-5}$, versus
$8.16781\times10^{-3}/2.52606\times10^{-3}$ for the ablation.
The full and ablated sign-error counts are zero and eight. The exception
favoring the ablation has a failed predicted word and is retained.
On the 223 valid informative words, every row favors the full recursion.

\begin{table}[!htbp]
\centering
\caption{Finite affine comparison tests. Incorrect counts refer to certified
comparisons; unresolved includes both invalid words and insufficient margins.}
\label{tab:affine-certificate-coverage}
\begin{tabular}{lrrrrr}
\toprule
Population & Total & Valid word & Certified & Incorrect & Unresolved\\
\midrule
Original E2 &72&71&71&0&1\\
Dense E2, both couplings &24006&23997&23997&0&9\\
Dense E2, coupled &12003&11999&11999&0&4\\
Feedback-active coupled &224&223&223&0&1\\
Coarse E2 &110&101&94&0&16\\
\bottomrule
\end{tabular}
\end{table}
Every row is eligible for evaluation. Unresolved includes failed words;
the categories are not additive. The coarse unresolved set contains nine
failed words and seven valid words with insufficient sign margins.
Certificate columns refer to the second-order deterministic test, whereas
the feedback-error comparison isolates the first-order recursive law.
The initial homogenized certificate gave 7/72 and 0/24006. Contractive
first-order bounds give 70/72, 23997/24006 and 71/110 without changing
that predictor; second order gives the displayed counts.

\paragraph{What the finite decisions change.}
All 19 coupled dense comparisons opposing AD and flow are certified;
their actual gaps are about $1.3\times10^{-6}$--$5.1\times10^{-6}$.
The coarse coupled cases at $N=150,200$, $\epsilon=0.001$, have
actual positive responses $0.035408$ and $0.010055$, while both
infinitesimal predictors are negative. The second-order response errors
are bounded by $0.001118$ and $0.000371$. Guaranteed loss gaps rounded
down are $6.85\times10^{-5}$ and $1.93\times10^{-5}$, below the
actual gaps $7.08\times10^{-5}$ and $2.01\times10^{-5}$.
The coupled disagreement at $N=75$, $\epsilon=0.003$ remains unresolved.
The largest relative candidate gaps in this coarse scan agree with the
infinitesimal predictions; larger objective effects are not automatically
an advantage for the new predictor.

\paragraph{Controlled smooth-field and curved-surface check.}
\label{app:curved-check}
Apply $h(x)=(x_1,x_2+x_1^2/4)$ to the fixed coupled flow and set
$f_y(h(x))=Dh(x)f_x(x)$. Its guards are $g_1(y)=y_1$ and
$g_2(y)=y_2-y_1^2/4$. The diffeomorphism carries the proven isolating
tubes and invariant region to bounded tubes with the same event order
and one-sided normal speeds. This checks the abstract piecewise-$C^2$
theorem; it does not preserve Euclidean gradient structure. Euler is
stepped directly in $y$, since
$h(x+\eta f_x)-h(x)-\eta Dh(x)f_x=(0,\eta^2f_{x,1}^2/4)$.
We fix $y_0=(-.4,-.96)$, direction $e_1$, $q(y)=(y_1-1)^2/2$,
$N=200,500,1000,2000,5000,10000$ and $z=-1,0,1$ before execution.
All 18 rows retain the predicted word; 11 meet the preset recursive
phase margin $0.05$. Across all rows the observed endpoint error divided
by $\eta^2$ ranges from $0.0561$ to $0.298$. Three of six central pairs
have both endpoints margin-eligible; two meshes also include the base
point under that margin. These are phase checks, not a numerical
certification of the theorem's sufficient mesh threshold. The six
response errors divided by $\eta$ are at most $0.141$; none reverses
the prediction's sign. These are observed ratios, not certified constants.
Three 60-digit direct-$y$ Euler checks at the geometrically selected
closest-margin mesh $N=10000$ agree in event indices and differ from
float64 endpoints by at most $1.24\times10^{-15}$.
No finite nonlinear certificate is applied.

\paragraph{E3: nonlinear training and branch-response inputs.}
The original width-4 network has 17 parameters. All seeds $0,\ldots,19$,
$N=500,2000$, and radii $10^{-5},10^{-4},10^{-3},10^{-2}$ are kept.
The direction is the first initialization coordinate, and the objective
is the frozen validation half-MSE. Seed 16 has no admissible flow reference,
leaving 152 comparable reference rows out of 160. Every plotted radius
uses the same 19 seeds and 38 seed--mesh observations.

The Branch predictor takes the frozen base event times and directed
surface sequence, reconstructs regional masks, and integrates the state,
regional variational matrix and \eqref{eq:main-smooth-bias} with DOP853
at relative/absolute tolerances $2\times10^{-11}$ and $2\times10^{-13}$.
The reconstructed endpoint, regional product and saltation-bearing
derivative differ from the preserved reference by at most
$1.42\times10^{-14}$, $3.20\times10^{-13}$ and $3.29\times10^{-13}$.
Independent Euler step-halving on each seed's longest regional segment
checks the smooth bias, with observed error ratios below $0.500$.
Only this base information constructs the response. Actual perturbed
endpoints, masks and objectives are used afterward to evaluate it.

For both perturbed candidates, 90/152 proposed index sequences are strict
and ordered; 81/152 match the complete rerun training word. Terminal
validation masks match the reference branch on 138/152 rows. These are
computed diagnostics, not a proof of all uniform-neighborhood hypotheses.
All 152 predictions, including failures, enter the main comparison.
There is no finite nonlinear certificate: certified/incorrect-certificate
counts are inapplicable, rather than zero passed tests.

\begin{table}[!htbp]
\centering
\caption{E3 absolute response errors and complete training-word matches.
Every radius uses the same 38 seed--mesh observations. Branch denotes the
branch-response predictor.}
\label{tab:network-radius-medians}
\begin{tabular}{rrrrr}
\toprule
Radius & AD median & Flow median & Branch median & Matched training word\\
\midrule
$10^{-5}$&$9.56\cdot10^{-12}$&$1.57\cdot10^{-3}$&$9.92\cdot10^{-6}$&29/38\\
$10^{-4}$&$1.96\cdot10^{-3}$&$1.88\cdot10^{-3}$&$8.32\cdot10^{-6}$&29/38\\
$10^{-3}$&$2.78\cdot10^{-3}$&$3.86\cdot10^{-4}$&$1.22\cdot10^{-5}$&17/38\\
$10^{-2}$&$1.72\cdot10^{-3}$&$1.87\cdot10^{-4}$&$1.43\cdot10^{-4}$&6/38\\
\bottomrule
\end{tabular}
\end{table}
Entries are absolute response errors. No network sign reversal occurs
relative to AD, flow or the Branch prediction. The Branch predictor's
median at $10^{-4}$ conceals a failed-word maximum error $0.0486928$;
its mean is $1.29792\times10^{-3}$. At $10^{-2}$ the Branch and flow
means are $1.02079\times10^{-3}$ and $1.10035\times10^{-3}$.
All quartiles, tail errors and the regional limit are included in the
supplementary arrays. Small medians do not validate incorrect words.

\paragraph{Sign policy, validation and reproducibility.}
Descriptive sign comparisons classify values of absolute magnitude at
most $10^{-12}$ as zero, never as a reversal. Error differences within
$10^{-12}$ are ties. The deterministic certificate uses its strict
unrounded sign margin, independently of this descriptive tolerance.
Its inequalities and strict word tests are evaluated in floating point.
Independent checks include 24 distinct 60-digit Euler comparisons,
353 high-precision matrix-power state bounds, and 1,200 planar extrema
tests including interior extrema and repeated eigenvalues. They validate
the implementation numerically; formal interval arithmetic is not used.
A separate algebra check covers 24 ordered affine-block examples in
dimensions two and three, with zero, one, three and five event slots.
All 65-digit exact matrix-power errors satisfy the finite second-order
remainder (maximum error/bound $0.10441$). These prescribed blocks check
the finite-event algebra; they do not verify event geometry. Hash checks
preserve all 951 earlier experiment files and 13 certificate-stage files.

The supplied scripts reproduce the calculations and vector figures.
From the paper directory, run
\begin{quote}\ttfamily
python scripts/network\_smooth\_response.py\\
python scripts/validate\_curved\_response.py\\
python scripts/final\_figures\_analysis.py\\
python scripts/validate\_final\_finite\_events.py
\end{quote}
The evidence package also includes the original scalar, coupled and
network generators, their protocols and raw arrays. Detailed hashes and
build manifests are supplementary artifacts rather than mathematical assumptions.

\subsection{Frozen network outputs: resolution, tails, and decision evidence}
\label{app:resolution-network}
This appendix reanalyzes saved outputs without running training, perturbed trajectories, or a reference solver. The descriptive protocol was fixed before aggregation. E3 retains all 20 seeds at both meshes $N\in\{500,2000\}$ and four radii $\epsilon\in\{10^{-5},10^{-4},10^{-3},10^{-2}\}$. Eight seed-16 rows lack a reference; four-method comparisons use the same 19 seeds and 38 seed--mesh observations at each radius. These repeated observations are not independent trials.
\paragraph{Definitions and raw-output checks.}
For both candidates, the complete predicted training word is compared with the saved training-word hash. The three disjoint reference strata are M: both complete words match (81 rows); O: both predicted words are strictly ordered, but at least one complete word differs (9); and U: at least one predicted word is not strictly ordered (62). Terminal validation-mask agreement is recorded separately and does not remove rows. The M/O distinction requires actual word comparisons. Unordered indices (U) can be detected at prediction time and left unresolved, but all 62 U rows remain in the reported comparison. The uncovered eight rows remain explicit in the full-population table.
Only 79/81 complete-training-word-matched rows also retain both reference terminal validation masks. Matching the training word alone therefore does not establish the smooth-terminal hypothesis. Across the increasing radii, the matched counts are 29, 29, 17, 6. The pooled strata differ in radius and event-count composition; their error contrast is descriptive and cannot isolate a causal effect of word correctness.
The stored flow and regional objective gradients use the validation gradient at the frozen continuous endpoint: $D\phi_T^\top\nabla q(\phi_T)$ and $J_{\rm reg}^\top\nabla q(\phi_T)$. Discrete AD uses the finite endpoint. Recomputing all three full vectors for 19 seeds and four historical meshes reproduces every stored component exactly in this environment; the suspected use of a finite terminal gradient for the flow object is therefore rejected. Saved $Q_+$ and $Q_-$ reconstruct every central response; saved endpoint-error vectors reconstruct every branch prediction; all stored word, phase-margin and terminal flags agree with their raw records. This checks data lineage and arithmetic, not an independent trajectory solve.
\begin{table}[t]
\centering
\begin{tabular}{llrrrrrr}
\toprule
Stratum & Method & $n$ & Median & Mean & 90th & 95th & Maximum\\
\midrule
M & AD & 81 & 0.000457 & 0.00958 & 0.024 & 0.0588 & 0.157\\
M & Flow & 81 & 0.000652 & 0.00642 & 0.00832 & 0.0249 & 0.135\\
M & Regional & 81 & 0.000453 & 0.00959 & 0.024 & 0.0588 & 0.157\\
M & Branch & 81 & 9.03e-06 & 2.85e-05 & 4.12e-05 & 4.99e-05 & 0.000777\\
\addlinespace
O & AD & 9 & 0.00143 & 0.0044 & 0.00991 & 0.016 & 0.0222\\
O & Flow & 9 & 7.49e-05 & 8.79e-05 & 0.000179 & 0.000204 & 0.00023\\
O & Regional & 9 & 0.00139 & 0.00438 & 0.00991 & 0.016 & 0.0222\\
O & Branch & 9 & 8.16e-05 & 0.000123 & 0.000241 & 0.000384 & 0.000527\\
\addlinespace
U & AD & 62 & 0.00314 & 0.0146 & 0.0549 & 0.0606 & 0.155\\
U & Flow & 62 & 0.000613 & 0.00769 & 0.0112 & 0.0279 & 0.133\\
U & Regional & 62 & 0.00311 & 0.0146 & 0.0549 & 0.0605 & 0.155\\
U & Branch & 62 & 7.17e-05 & 0.00171 & 0.00459 & 0.0076 & 0.0487\\
\bottomrule
\end{tabular}
\caption{Absolute finite-response errors, pooled across the frozen radii and meshes, in the three exact strata defined above. All rows are retained. Scientific notation uses e for powers of ten. These quantiles describe this finite, correlated cohort.}
\label{tab:resolution-network-strata}
\end{table}
\paragraph{Radius comparisons do not identify a universal winner.}
\begin{table}[t]
\centering
\begin{tabular}{rrrrrrrr}
\toprule
$\epsilon$ & AD wins & Flow & Regional & Branch & Branch mean & 95th & Maximum\\
\midrule
1e-05 & 25/38 & 0/38 & 0/38 & 13/38 & 2.76e-05 & 0.000101 & 0.000242\\
0.0001 & 12/38 & 1/38 & 0/38 & 25/38 & 0.0013 & 6.19e-05 & 0.0487\\
0.001 & 2/38 & 6/38 & 0/38 & 30/38 & 0.000527 & 0.00447 & 0.00767\\
0.01 & 1/38 & 11/38 & 4/38 & 22/38 & 0.00102 & 0.00751 & 0.00777\\
\bottomrule
\end{tabular}
\caption{Lowest absolute-error counts among the four methods on the same 38 observations per radius. The prespecified absolute tie tolerance is $10^{-12}$; no ties occur. Means and upper tails accompany the radius plot rather than reporting medians alone.}
\label{tab:resolution-network-winners}
\end{table}
At the largest radius, branch-response mean error is 0.00102 versus flow 0.0011; its 95th percentile and maximum are slightly larger. Thus neither the means nor the winner counts support a claim that flow uniformly dominates at the largest tested radius. The largest branch error, 0.0487, occurs at seed 9, $N=500$, $\epsilon=0.0001$, in stratum U; its recursive margin is 0.00246 and the actual absolute two-candidate objective difference is 2.74e-05. These quantities have different units: using the rounded values and $\epsilon=10^{-4}$ gives a gap-prediction error of about $2\epsilon(0.0487)=9.74\times10^{-6}$ and $|D|\approx0.137$. This is an illustrative conversion, not a new experiment. All four predictors preserve the sign on the fixed E3 direction in all 152 reference-covered rows. This establishes no optimization benefit or direction reversal.
\paragraph{Phase margin, event count, and mesh resolution.}
The recursive margin is the minimum distance to an integer of every event ceiling argument, over both candidates. It includes upstream feedback in the prediction, but it is not a measured distance to the exact nonlinear discrete branch boundary. Strict order also requires $0<K_1<\cdots<K_m<N$. The reported shortest reference segment includes the initial-to-first and last-to-terminal intervals; it is not silently replaced by an inter-event-only minimum. Event counts are grouped by their exact integer value, and mesh summaries use the same seeds.
\begin{table}[t]
\centering
\begin{tabular}{rrrrrrr}
\toprule
$N$ & Rows & Word matched & Unordered & Margin median & IQR & AUC\\
\midrule
500 & 76 & 34 & 38 & 0.0077 & [0.00299, 0.0265] & 0.676\\
2000 & 76 & 47 & 24 & 0.0111 & [0.00341, 0.0321] & 0.597\\
\bottomrule
\end{tabular}
\caption{Descriptive discrimination of complete word matching by a larger recursive margin. The remaining ordered-mismatched rows number 4 at $N=500$ and 5 at $N=2000$, completing each 76-row partition. AUC is the rank probability with half credit for ties; no threshold is fitted, and no independence-based uncertainty interval is claimed.}
\label{tab:resolution-network-margin}
\end{table}
\begin{figure}[t]
\centering
\includegraphics[width=\textwidth]{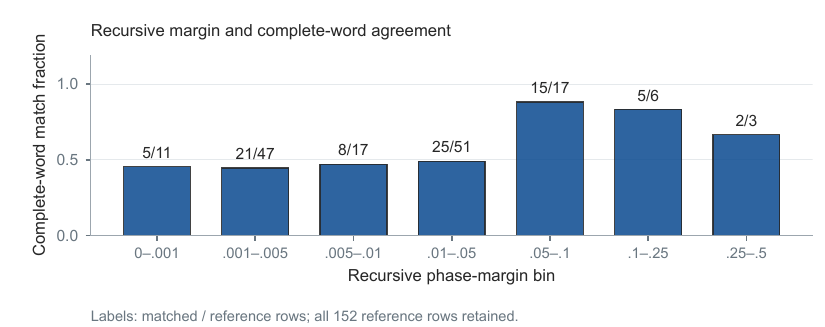}
\caption{Exact complete training-word match fractions in the prespecified recursive-margin bins for all 152 reference-covered width-4 rows. Labels give matched/total dependent seed--mesh--radius row counts, not independent trials. The eight uncovered rows have no margin. No bin is chosen or removed after observing the errors. The margin is a descriptive diagnostic and is not a nonlinear certificate.}
\label{fig:resolution-margin}
\end{figure}
Reference event counts range from 1 to 46 (median 23). The pooled margin AUC is 0.647 and the pooled Spearman correlation with branch-response error is -0.156. The prespecified $N=2000$, $\epsilon=10^{-4}$ subgroup has AUC 0.324; hence even the direction of association is not uniform across the tested groups. The fixed bins are $[0,.001),[.001,.005),[.005,.01),[.01,.05),[.05,.1),[.1,.25),[.25,.5]$. Every bin count, error distribution, exact-event-count group, and paired seed--mesh radius transition is retained. In this shallow model the activation guards are affine in the parameters ($D^2g=0$), so this subgroup result does not identify guard curvature as its cause. The computed phase margin alone does not control accumulated regional approximation errors or establish the theorem's finite-mesh conditions. Multiple crossings within one GD step can indicate insufficient grid resolution; they do not establish simultaneous continuous-time events.
\paragraph{Full initialization gradients versus finite vector decisions.}
On the same 38 E3 seed--mesh pairs, the full flow gradient differs from AD in at least one component sign on 16/38 (8/19 unique seeds), in the largest absolute coordinate on 6/38 (3/19 seeds), and in the ordered two largest absolute coordinates on 16/38 (8/19 seeds). The minimum cosine is 0.858397, the maximum normalized-direction distance is 0.532172, and its median is 0.300401. Regional gradients change none of these signs or ranks and have maximum normalized-direction distance 0.001488. These are objective-gradient and normalized-direction diagnostics, distinct from matrix discrepancies. They do not evaluate finite proposals along these full-vector directions. The 76-row, four-mesh historical diagnostic is reported separately, rather than doubling the E3 denominator.
\paragraph{Existing outer-variable proposals and preserved negative results.}
The separate frozen four-dimensional training-group-weight audit starts with 20 seeds, retains 15 eligible seeds and 14 fixed common-itinerary seeds across all primary radii, and stores 660 proposals including boundary probes. No proposals are rerun here. Historical method names are retained: these are not finite proposals generated by the new branch-response predictor.
Flow outer gradients differ from AD in some component sign on 2/15, in ordered top-two absolute-coordinate ranks on 1/15, and in top-one ranks on 0/15.
Event-aware outer gradients differ from AD in some component sign on 2/15, in ordered top-two absolute-coordinate ranks on 1/15, and in top-one ranks on 0/15.
At the three smallest primary radii, flow-minus-AD proposal-objective median differences are 9.96e-10, 1.99e-09, 3.99e-09, all positive in all 14 common seeds. At radius .005 the signs split 7 lower/7 higher; the median is 8.39e-10 and the largest absolute difference is 4.5e-06. After the stored five matched outer steps, flow achieves a lower final objective on 13/15 completed seeds and 9/11 seeds satisfying every stored itinerary flag; on the latter fixed subset the median flow-minus-AD final-objective difference is -2.41e-07. Thus gradient/rank differences do not entail uniformly better finite proposals, and the negative result for a universal AD advantage is preserved.
\paragraph{Reproduction and unavailable extensions.}
From the paper directory, run:
\begin{verbatim}
python -X utf8 scripts/resolution_network_analysis.py
\end{verbatim}
The frozen analysis protocol, row IDs, source hashes, exact numeric CSV tables, and figure inputs are under \path{experiments/resolution_revision_20260906/}. The script reads saved arrays and raw JSON only. The E3 initialization scan and its recursive-margin outputs cover only width 4. The historical width-8 and width-16 base-path resonance and matrix checks measure different quantities. Separately, Appendices~\ref{app:rebuttal-checks} and~\ref{app:paired-refinement} evaluate new group-weight candidate proposals and paired refinement at all three widths. Those studies must not be inferred from E3, which tests only initialization direction $e_1$. None establishes a nonlinear finite certificate or computational savings over direct candidate evaluation.

\subsection{Frozen network population and reference construction}
\label{app:repro}
The supporting matrix audit uses the frozen two-moons split, with 64
training and 64 validation samples, noise $0.08$, standardized training
inputs and targets in $\{-1,+1\}$. Full-batch hard-ReLU squared-loss
training uses float64, $T=1$, widths $4,8,16$ (respectively $17,33,65$
parameters) and seeds $0,\ldots,19$ at every width. The reference-admissible
counts are $19/20,20/20,18/20$: 57 of the fixed 60 configurations,
evaluated at $N=200,500,1000,2000$ for 228 mesh rows.
Width-4 seed 16 has a numerically grazing one-sided speed; width-16
seeds 4 and 15 have opposing one-sided normal velocities. Zero cases
are event-admissible but fail the subsequent reference comparison.
No sensitivity-gap threshold selects this population.

Within each activation region, DOP853 integrates the state, regional
variational equation and four group-weight sensitivity columns. Each
sample--unit preactivation is projected onto the accepted step's dense
polynomial. A Bernstein same-sign test rejects root-free polynomials;
otherwise derivative roots partition the step into monotonicity cells.
Companion roots and dense evaluations cross-check candidates, ambiguities
fail closed, and continuation stops at the earliest root before restarting
under the outgoing field. The 8-, 32- and 128-point scans check this
inventory afterward; they are not independent root locators.
Coarse/medium/tight $(\mathrm{rtol},\mathrm{atol},h_{\max})$ are
$(3\cdot10^{-8},3\cdot10^{-10},.005)$,
$(10^{-9},10^{-11},.003)$ and
$(3\cdot10^{-11},3\cdot10^{-13},.002)$.

Admissibility requires unit-normal speeds $\geq10^{-5}$, event separation
$\geq2\cdot10^{-7}$, normalized non-event margin $\geq10^{-9}$, and
absolute/relative root residuals $\leq2\cdot10^{-9}$. Initial/terminal
events, near-simultaneity, opposing speeds and numerical grazing fail
closed. The normalized margin is $|s_i|/\max(1,\|n_i\|_2)$.
All three solves must agree on directed event signatures; medium/tight
endpoints, event times, regional and flow Jacobians, parameter
sensitivities and transported products are also compared.
These are floating-point checks on numerical interpolants, not an
interval proof of the exact ODE event inventory.

The 57 references contain 2,357 events. Maximum root residual and
polynomial/evaluator discrepancy are $3.33\cdot10^{-16}$ and
$1.33\cdot10^{-15}$; minimum unit-normal speed, event separation and
normalized non-event margin are $5.70\cdot10^{-4}$,
$1.18\cdot10^{-6}$ and $3.04\cdot10^{-8}$.
All signatures agree across tolerances and no inventoried root is missed
by the auxiliary scans. Across 228 GD rows no evaluated zero activation
is detected; the minimum dimensionless resonance distance is
\[
 \min_{k,i}\frac{|s_i(\theta_k)|}
 {\eta\max(|n_i^\top f(\theta_k)|,10^{-14})}
 =2.09\cdot10^{-5}.
\]
This supports numerical nonresonance of those evaluated programs.

\subsection{Correction sources and independent evidence}
\label{app:correction-sources}

\begin{table}[!htbp]
\centering\small
\caption{Sources in the frozen multi-width implementation. Reference and
coarse-GD columns describe separate computations.}
\label{tab:correction-sources}
\begin{tabular}{@{}p{.19\textwidth}p{.36\textwidth}p{.39\textwidth}@{}}
\toprule
Quantity & Reference computation & Coarse-GD correction computation\\
\midrule
Event count/order & Dense DOP853 roots; earliest root followed by restart
& All changed affine preactivations on each Euler segment, sorted by fractional crossing location\\
Event time/state & Root-localized dense state & Linear interpolation of the actual Euler segment\\
Normals, one-sided fields & Analytic affine normals; fields at reference root under adjacent masks
& Same formulas at interpolated coarse state, toggling masks in estimated order\\
Regional propagation & DOP853 variational solve between reference events
& Entire-step factors $I+\eta A(\theta_k)$; no within-step regional subintegration\\
$J_{\reg}$ & Reference regional matrices without saltation & Not used to construct $J^{\SC}$\\
$D\varphi_T$ & Reference regional matrices and reference transfers & Comparison target only\\
$J^{\SC}$ & Not supplied by reference inventory & Full coarse factor followed by all estimated transfers at that step\\
\bottomrule
\end{tabular}
\end{table}

The final driver passes initialization, weights, mesh, dataset, width and
configuration to \texttt{gd\_with\_derivatives}, but no reference events.
Six replays (seed zero at each width, $N=200,2000$) reproduce saved
$J^{\AD}$ and $J^{\SC}$ arrays exactly in the recorded environment.
Thus labeling this network correction reference-assisted reconstruction
would be inaccurate. Reference events are used in the separate
transported-product identities and exact scalar correction controls.

\paragraph{Multiple events within one step.}
For this shallow network each preactivation is affine in all parameters,
hence affine in fractional time on a straight Euler segment. With strict
endpoints it has at most one zero, exactly characterized by opposite signs.
The code processes every changed sample--unit pair, sorts
$-s_i(\theta_k)/(s_i(\theta_{k+1})-s_i(\theta_k))$, and applies all transfers
in that order. The width-4 $N=200$ replay has one two-event step; width 16
has two multi-event steps, with up to three events in a step.
Completeness holds for these straight segments, not for the true flow,
which may recross between grid times. It does not extend to general
deep-network preactivations along parameter segments.

The reference locator audits each regional dense polynomial, advances only
to its earliest root, switches the field, and restarts. Subsequent candidates
are recomputed under outgoing dynamics. Its solver step need not be below
the smallest reference event separation $1.18\times10^{-6}$.
Coarse endpoint masks alone do not supply that inventory.

\paragraph{Retained itinerary mismatches.}
All 228 historical rows pass the correction's local speed and near-simultaneity
checks. Only 165/228 stored event-identity sequences match the reference:
$31,38,46,50$ of 57 at $N=200,500,1000,2000$. All 63 mismatched rows remain
in the frozen summaries. The historical flag \texttt{correction\_covered}
does not establish the correct-itinerary hypothesis of
Theorem~\ref{thm:general-event-aware-consistency}.
The seven unmatched finest-grid cases remain in the full-cohort summary;
the separate matched-cohort comparison uses a fixed intersection across
all four meshes (Appendix~\ref{app:derivative-coverage}).

For matched seed-zero replays, maximum event-time errors at $N=200$ are
$7.77\times10^{-3}$ (width 4) and $2.59\times10^{-3}$ (width 8); at
$N=2000$ they are $5.18\times10^{-4},3.32\times10^{-4},2.25\times10^{-4}$
by width. Width 16 at $N=200$ has unequal order despite equal counts (43);
no paired localization error is assigned. A small corrected matrix error
is a controlled consistency diagnostic, not certified event recovery
or the derivative of the unmodified finite GD program.

\paragraph{Three validation claims.}
The 57 transported-product and determinant tests are algebraic
self-consistency. Cross-tolerance root, state and variational comparisons
test reference stability. Perturb-and-rerun checks in
Table~\ref{tab:derivative-coverage} validate selected responses, not full
Jacobians. Leading singular-gap directions are deliberately chosen
worst-direction checks, not typical learning evidence. The new coordinate
scan is reported separately and changes no historical coverage denominator.

\subsection{Derivative coverage and fixed-cohort matrix summaries}
\label{app:derivative-coverage}
Relative errors divide by $\max(\|\text{reference}\|,10^{-14})$,
using Euclidean vector and spectral matrix norms. State errors use
endpoints, not the path supremum. Table~\ref{tab:derivative-coverage} summarizes the independent perturbation coverage.
\begin{table}[!htbp]
\centering
\caption{Independent checks retain their eligible denominators. A passed
directional test is not a full-Jacobian check; uncovered perturbations
remain in the checked count.}
\label{tab:derivative-coverage}
\begin{tabular}{lrrrr}
\toprule
Check & Eligible & Checked & Passed & Maximum relative error\\
\midrule
Discrete Rademacher JVP &228&4&4&$4.42\cdot10^{-8}$\\
Flow Rademacher JVP &57&1&1&$3.69\cdot10^{-8}$\\
Discrete singular-gap JVP &228&36&36&$1.96\cdot10^{-7}$\\
Flow singular-gap JVP &57&9&8&$7.60\cdot10^{-8}$\\
Initialization full Jacobian &57&0&0&--\\
Discrete group hypergradient &15&15&15&$5.79\cdot10^{-9}$\\
Network forward--adjoint &57&0&0&--\\
\bottomrule
\end{tabular}
\end{table}
The first two rows of Table~\ref{tab:derivative-coverage} use eight normalized Rademacher directions and overlap
only width-4 seed 0 at $T=1$ with four meshes. The additional audit fixes
three cases per width by minimum normal speed, maximum limiting gap,
and upper-median limiting gap (ties by seed), then tests one leading
right singular-gap direction per case. This is a worst-direction
diagnostic. Discrete perturbations shrink from $10^{-6}$ to $10^{-10}$
to preserve the full training word; flow perturbations use $2\cdot10^{-6}$
and require the same directed itinerary. The positive width-8 seed-3
flow perturbation changes event order and remains uncovered.
In total, 40/228 distinct meshes and 10/57 references have independent
initialization checks attempted; nine references have covered passing
checks. No full-Jacobian or network adjoint coverage is inferred.

At $N=2000$, pooled median endpoint, AD-to-regional and corrected-to-flow
relative errors are $3.58\cdot10^{-5}$, $1.31\cdot10^{-4}$ and
$1.51\cdot10^{-4}$. Median AD-to-flow and regional-to-flow gaps are
$0.384737$ and $0.384733$; their absolute difference has median
$3.45\cdot10^{-6}$. Median per-case slopes on $N=500,1000,2000$ are
$0.998,1.001,0.997$ for the first three errors and
$3.50\cdot10^{-6}$ for AD-to-flow error.
The maximum transported-product, flow-reconstruction and
determinant/speed-ratio residuals are $2.00\cdot10^{-14}$,
$7.35\cdot10^{-15}$ and $2.67\cdot10^{-14}$.
These last checks are algebraic identities on reference quantities.

All 57 configurations and 228 rows remain in the full summaries.
The fixed intersection matching the reference itinerary at every mesh
contains 31 configurations. Slopes of the final-three-mesh
\emph{pooled median} corrected error are $1.0094$ for all 57 and
$0.9986$ for the fixed 31. These are distinct from the preceding
median-of-case-slopes statistic. Each curve uses the same IDs at every
mesh; \texttt{correction\_fixed\_cohort.json} records them.
The per-width matrix summaries are retained in
Figure~\ref{fig:width-geometry-full}.
\begin{table}[!htbp]
\centering
\caption{Fixed-width summaries at $N=2000$. Event counts and relative
matrix errors are medians over the stated reference population.}
\label{tab:width-audit}
\begin{tabular}{rrrrrrr}
\toprule
Width & Parameters & References & Events & $G_\infty$ & $e_{\rm reg}$ & $e_{\rm corr}$\\
\midrule
4 &17&19/20&23.0&.346&$1.12\cdot10^{-4}$&$1.42\cdot10^{-4}$\\
8 &33&20/20&36.0&.397&$1.29\cdot10^{-4}$&$1.27\cdot10^{-4}$\\
16&65&18/20&53.5&.413&$1.39\cdot10^{-4}$&$3.13\cdot10^{-4}$\\
\bottomrule
\end{tabular}
\end{table}

Gate sensitivity retains the original 60-case denominator. Normal-speed
thresholds $10^{-6},10^{-5},10^{-4}$ retain 57 references, and $10^{-3}$
retains 53. Separation thresholds $2\cdot10^{-8},2\cdot10^{-7},10^{-6}$
retain 57, while $2\cdot10^{-6},10^{-5}$ retain 56. Normalized
near-touch thresholds through $10^{-8}$ retain 57, and $10^{-7}$
retains 54; root-residual thresholds $10^{-10}$ through $10^{-7}$
leave the count unchanged. Conditional median limiting gaps range
from $0.376$ to $0.386$. These are deterministic screening checks,
not prevalence estimates.

\subsection{Outer-weight provenance and negative optimization result}
With $\ell_i(\theta)=\tfrac12(f_\theta(x_i)-y_i)^2$, four groups defined
by class and the sign of the first standardized coordinate have counts
$24,8,9,23$. The implemented objectives are
\begin{align}
 L_{\rm train}(\theta,\lambda)
 &=\frac1{64}\sum_{i\in\mathcal T}\lambda_{g(i)}\ell_i(\theta),
 \label{eq:implemented-group-loss}\\
 Q_{\rm disc}(\lambda)
 &=\frac1{64}\sum_{i\in\mathcal V}
           \ell_i(\Theta_{\eta,T}(\theta_0;\lambda)).
 \label{eq:implemented-outer-loss}
\end{align}
The denominator is the sample count, not the sum of weights.
Gradients are with respect to $\lambda$, with base point $(1,1,1,1)$.
Proposals are $\lambda-\epsilon g/\|g\|_2$ without projection.
Nonpositive weights reject a proposal or stop a path. Every evaluation
reruns all $N=2000$, $T=1$ inner steps from the same initialization;
there is no warm start.

The width-4 population contains 19 final references, but its outer cohort
contains 15 because earlier gates excluded seeds 5 and 18 by separation,
9 by reference validation and 12 by the flow finite-difference gate.
Seed 16 also fails the final grazing audit. This 19-to-15 reduction is
not a new application of the final gates. The 15 retained references
agree with the final pipeline in event signature and in all compared
endpoint, initialization and group-sensitivity arrays to relative error
below $8.66\cdot10^{-13}$. All 660 one-step proposals and every five-step
path were rerun with those final references.

\begin{table}[!htbp]
\centering
\caption{All 20 width-4 seeds. Calibration checks the executed
training word at one radius, not a complete smooth neighborhood.}
\label{tab:outer-provenance}
\begin{tabular}{@{}p{.21\linewidth}ccc p{.26\linewidth}@{}}
\toprule
Seeds & Reference & Outer cohort & Calibration & Restriction\\
\midrule
2, 4, 7, 10, 11, 13, 15 & yes&yes&yes&None\\
0, 1, 3, 6, 8, 17, 19 &yes&yes&no&Training word changes\\
14 &yes&yes&no&Proposal itinerary changes\\
5,18 &yes&no&--&Earlier separation gate\\
9 &yes&no&--&Earlier reference gate\\
12 &yes&no&--&Earlier flow FD gate\\
16 &no&no&--&Grazing; earlier separation gate\\
\bottomrule
\end{tabular}
\end{table}
Every outer-cohort seed passes the reexecuted four-coordinate discrete
hypergradient finite differences, with maximum relative error
$5.79\cdot10^{-9}$. Four proposals (AD, regional limit, hybrid flow,
and corrected flow) use eight primary radii from $5\cdot10^{-5}$ to
$5\cdot10^{-3}$ and three boundary probes. Fourteen of 15 seeds retain
the common continuous and discrete event identities at all primary radii;
seed 14 fails this condition. The stricter calibration has seven seeds
at $\epsilon=5\cdot10^{-5}$. Its hashes cover training masks only, not
terminal validation masks or every point between evaluated endpoints.
Thus the full provenance is 57 references across widths, 19 at width 4,
15 outer-cohort seeds, 14 common-itinerary seeds, and seven calibration
seeds; these are different denominators and tests.

For the three smallest radii, hybrid-minus-AD median loss differences
are $9.96\cdot10^{-10}$, $1.99\cdot10^{-9}$ and
$3.99\cdot10^{-9}$, positive on all 14 seeds; signs mix at larger radii.
Base hybrid/discrete gradients differ in a component sign on 2/15 seeds,
an ordered top-two choice on 1/15, and no top-one choice.
The five-step audit recomputes each derivative and uses radius $0.005$.
Eleven of 15 seeds retain a common initial itinerary along all four
complete paths. Median final loss differences relative to AD are
$1.69\cdot10^{-9}$ for the regional path,
$-2.41\cdot10^{-7}$ for hybrid flow and
$-2.52\cdot10^{-7}$ for correction. Hybrid and corrected paths each
attain lower loss on 9/11 seeds. This negative optimization result is
retained in Figure~\ref{fig:single-objective-full}; a derivative identity
does not rank nonlinear multi-step optimizers.

\subsection{Supporting scalar, parameter, and resolution checks}
\label{app:scale-experiment}
The curvature-phase control fixes $T=.8$, $t_*=.3$ and event ratios
$r\in\{1/40,1/20,1/10,1/5,1/2,1,2,5,10,20,40\}$, with
$m=1/\min(1,r)$, $c=(r-1)m$ and $\theta_0=m(1-e^{.3})$.
Eight step sizes $.04,.02,.01,.005,.002,.001,.0005,.0002$ yield 88
nonresonant one-crossing rows. The objective is $q(z)=z^2/2$ and the
scalar correction inserts the exact event ratio. Minimum speed,
dimensionless resonance distance and branch radius are $1$, $0.150005$
and $4.05\cdot10^{-5}$. Maximum endpoint/outer AD discrepancies are
$3.82\cdot10^{-14}$/$4.11\cdot10^{-13}$; corresponding finite-difference
discrepancies are $3.66\cdot10^{-9}$/$3.37\cdot10^{-8}$.

The earlier scalar radius grid contains 78 rows and six record-low
near-resonance probes, in addition to the complete phase population in
E1. The probe $N=9801$, $\rho_\eta=8.74\cdot10^{-10}$,
$\epsilon=\sqrt{\rho_\eta\eta}=2.99\cdot10^{-7}$ has response $89.53$
but absolute objective difference only $5.35\cdot10^{-5}$.
It deliberately probes resonance, not typical impact.
Within-cell response agrees with AD to $6.15\cdot10^{-13}$; the
integral-plus-jump residual is at most $1.69\cdot10^{-9}$.
An independent Euler loop checks 18 endpoints and derivatives.
The earlier width-4 seed-0 coordinate scan retains all 24 mesh--radius
comparisons and 12 perturbed reference flows. Every response is negative.
Training and terminal validation branches can change separately;
no interval-certified branch radius is claimed.

The two-parameter control uses
$F_{(u,v),\lambda}(1)=(u,v,\operatorname{ReLU}(u-\lambda)+.5v)$,
target $(1,.2,1)$, $T=.8$, $\lambda=0$ and
$(u_0,v_0)=(1-e^{.3},0)$. Its one event is at $.3$, state
$(0,.1751180039)$, with speeds $1,1.9124409981$ and
$\beta=-.9124409981$. On the same eight step sizes, maximum PyTorch,
discrete finite-difference and flow finite-difference discrepancies are
$4.87\cdot10^{-14}$, $2.59\cdot10^{-8}$ and $1.62\cdot10^{-9}$.
Continuous and independently implemented discrete forward--adjoint
pairings agree within $1.39\cdot10^{-17}$ and $1.43\cdot10^{-14}$.
These full-Jacobian/adjoint checks do not increase network coverage.

The scalar smoothing control uses
$F(u)=.2+1.8[1+\tanh(u/2)]/2$ on $[-20,20]$, with 25 logarithmic
values each of $\eta,\tau$ from $10^{-5}$ to $10^{-1}$: 625 float64
cells. The transfer error is $|J-10|$ and the truncation floor is
$2.04\cdot10^{-7}$. This fixed profile supports the scalar sufficient
resolution condition, not a universal phase boundary.

\section{Prior work and the mathematical increment}
\label{app:prior-art-comparison}
\paragraph{Primary-text locations.}
For \citet{stewart2010optimal}, we inspected the authors' June 6, 2005
preprint ANL/MCS-P1258-0605: printed pp.~4--5, Section~2, (2.3)--(2.4),
and pp.~7--8, Lemma~3.1. These are preprint pages, not an inferred offset
into the journal version, pp.~653--695, whose subscription full text was
unavailable for version comparison. The publisher abstract confirms the
crossing-sensitivity obstruction. For \citet{nurkanovic2024finite}, we
inspected the published Theorem~19 and proof, pp.~1150--1152, and
Section~4.6/Figure~9, pp.~1152--1154.

Stewart--Anitescu's crossing dynamics is
$\dot x\in(1+\alpha)-\mathrm{Sgn}(x)$, $\alpha>0$. They explicitly
calculate the flow sensitivity $\alpha/(2+\alpha)$ and persistent errors
of differentiated discretizations. Their theta-method includes implicit
and partly implicit branches; their derivatives are not all the Euler
regional limit. Lemma~3.1 explicitly obtains a nonpositive normal field
jump from one-sided Lipschitzness. The obstruction and attenuation geometry
therefore have direct precedents.

\paragraph{Convex-potential deduction and what strong convexification changes.}
Our algebraic deduction from their field is
\[
 V_\alpha(x)=|x|-(1+\alpha)x
 =2\operatorname{ReLU}(x)-(2+\alpha)x,\qquad
 -\partial V_\alpha=(1+\alpha)-\mathrm{Sgn}(x).
\]
This potential is convex, but unbounded below for $\alpha>0$, with zero
regional Hessians. It is a deduction, not a convex-ReLU theorem stated by
the prior authors. Expanding our residuals gives exactly
\begin{equation}
 \mathcal L_{\alpha+2,-2}(x)=V_\alpha(x)+\tfrac12x^2
 +\tfrac12\operatorname{ReLU}(x)^2+C_\alpha,
 \qquad C_\alpha=\tfrac12(\alpha+2)^2+2.
 \label{eq:prior-convex-embedding}
\end{equation}
The added term is $C^1$, globally 1-strongly convex, and has zero derivative
at zero. It preserves the interface, jump magnitude $2$, one-sided speeds
$(\alpha+2,\alpha)$, and saltation ratio. It changes the regional fields
from constants to $\alpha+2-x$ and $\alpha-2x$, and Hessians from $(0,0)$
to $(1,2)$. The crossing clock changes from $-x_0/(\alpha+2)$ to
$\log((\alpha+2-x_0)/(\alpha+2))$; regional propagation is no longer
the identity. The new potential is coercive with unique minimizer
$\alpha/2$. Thus the full trajectories and discrete branches are changed.
Every reciprocal ratio $R>1$ already occurs as $R=(\alpha+2)/\alpha$;
our fixed-speed family also rescales the instance. Strong convexity and
the explicit conditioning calculation are extensions of the example,
not a new inconsistency mechanism.

\paragraph{Mathematical versus application distinctions.}
Fixed-time GD refinement is explicit Euler refinement of a discontinuous
field. Calling depth $T/\eta$ optimization time changes the learning
interpretation, not that numerical limit. Fixed-program AD has different
quantifiers: it identifies the returned derivative at each finite program
without uniformly controlling shrinking cells
\citep[Section~3]{lee2020correctness}. Fixed-point differentiation genuinely
changes the limiting object: a fixed map is iterated to equilibrium, rather
than an $\eta$-dependent map evolved for time $T$
\citep[Assumption~1 and Theorem~2, pp.~3 and~5]{bolte2022iterative}.
Architectural Euler refinement can share the same numerical structure;
substantive differences require specifying fields, parameter sharing,
regularity and derivative targets. Smooth high-order ODE-net corrections
\citep[Propositions~2.2--2.3]{xu2023correcting} concern different
discretization weights. Gao et al.'s smooth gradient theorem
\citep[Proposition~2, p.~5]{gao2025global} and ReLU experiment do not
identify this optimizer sensitivity, but terminology alone supplies no
priority. Event-aware ReLU neural modeling also has a direct precedent in
\citet{brandle2026continuous}.

\subsection{Finite perturbation analysis and the mathematical increment}
Classical perturbation analysis already propagates finite changes along a
sample path, including changes of subsequent event decisions
\citep{ho1983finitepa,ho1991perturbation,ho1992concepts,wardi2017pa}.
For example, the Lindley recursion in \citet[author-hosted preprint,
pp.~3084--3085, Eqs.~(1)--(3)]{wardi2017pa} gives exactly
\[
 \delta D_k=\delta Z_k+
 \max\{I_k,\delta D_{k-1}\}-\max\{I_k,0\},
 \qquad I_k=A_k-D_{k-1}.
\]
An upstream finite perturbation changes the downstream branch
when it crosses the corresponding idle-time threshold. Hybrid IPA likewise
propagates state and endogenous event-time derivatives recursively
\citep[pp.~645--647, Eqs.~(6), (7), (11)]{cassandras2010ipa}.
The ordered variational and saltation factors are classical
\citep[Sec.~7.1]{burden2016event}.

Encoding each Euler step as a clock event places the entire finite
training program within this broad perturbation-analysis framework.
We study a family of these programs in which the step size and
perturbation radius vanish together. We derive the
integer crossing indices from grid signs, including the smooth Euler
bias and the earlier crossings' effects on later boundaries, and prove
a uniform local endpoint expansion away from the resulting recursive
boundaries. For the contractive affine subclass, verified branch words
also support an explicit finite-grid remainder bound. The new estimates
quantify how the integer crossing steps and the endpoint response
depend on mesh refinement and proposal scale.

\begin{table}[t]
\centering
\caption{Closest prior results and the restricted mathematical increment.
Locators refer to the inspected versions; detailed access records and
the full comparison matrix accompany the evidence package.}
\label{tab:fpa-audit}\label{tab:novelty-comparison}
\small
\begin{tabular}{@{}>{\raggedright\arraybackslash}p{.13\linewidth}>{\raggedright\arraybackslash}p{.16\linewidth}>{\raggedright\arraybackslash}p{.14\linewidth}>{\raggedright\arraybackslash}p{.12\linewidth}>{\raggedright\arraybackslash}p{.15\linewidth}>{\raggedright\arraybackslash}p{.10\linewidth}@{}}
\toprule
Source & Object and locator & Already established & Limit considered & Additional proof here & Access\\
\midrule
Ho et al.; Ho--Cao & FPA, 1983; book Ch.~6, pp.~185--227 & Finite sample-path reconstruction; event-order changes & Finite perturbation of a model & Uniform refining-grid theorem not verified there & Abstract and contents only\\
Ho (1992) & GSMP and FPA, pp.~231--238 & Clock-event representation; cut-and-paste propagation & Finite or infinitesimal path changes & Bias-corrected crossing-index asymptotics & Full text\\
Wardi et al. (2017) & Lindley rule, preprint pp.~3084--3085 & Exact finite thresholds; downstream feedback & Finite queue perturbation & Uniform Euler-scale expansion & Full text\\
Cassandras et al. (2010) & Hybrid IPA, pp.~645--647 & Recursive endogenous/induced event sensitivities & Infinitesimal parameter change & Derive integer events and surviving mesh phases & Relevant pages\\
Burden et al. (2016) & Sec.~7.1, pp.~1250--1254 & Ordered saltation; linearized guards; flow selections & Directional flow derivative & Discontinuous, nonhomogeneous grid-scale response & Relevant sections\\
\citet{zhang2007mpa} & Thms.~1--3; Sec.~3 & Local critical regions and finite error analysis & Queue perturbation within a basis & Recursive grid cells and cross-cell response & Full text\\
Stewart--Anitescu; FESD & Crossing example; Thm.~19 & Discretization sensitivity failure; event-aligned recovery & Fixed-grid or adapted ODE discretization & Coupled finite-response law and verified comparison & Full relevant proofs\\
\bottomrule
\end{tabular}
\end{table}

The full 1983 article, the 1985 event-domain article, and the full 1991
FPA chapter were available here through their abstracts and
bibliographic records. We use those records for historical attribution
and give theorem-level comparisons for the inspected texts.

\subsection{Finite outer learning and the resolution question}
\citet{maclaurin2015reversible} differentiate a complete finite training
run (PDF pp.~1--2, Algorithms~1--2). Their Section~4/Figure~10 (PDF p.~7)
already shows that exact hypergradients can fail to describe a broader
objective trend under unstable learning dynamics. The warning that a
derivative is local is therefore not an additional contribution here.
Our transverse-event analysis gives a quantitative mesh-dependent
response law, including strongly convex examples, without requiring
that instability mechanism.

\citet{franceschi2018bilevel} explicitly define the finite-inner objective
in Eqs.~(5)--(10) (PDF pp.~3--5), while their convergence theorems
compare increasing inner iteration counts with an inner optimum.
This differs from refining the Euler step at fixed physical horizon.
\citet{finn2017maml} update an initialization through a finite adaptation
step (Eq.~(1), Algorithms~1--2, PDF pp.~2--4).
\citet{ren2018reweight} obtain validation-guided example weights through
a one-step training surrogate and subsequent rectification/normalization
(Eqs.~(4)--(9), PDF pp.~3--4).
These sources motivate finite outer actions; they do not establish the
current event geometry for their stochastic, momentum or online algorithms.
Continuous outer parameters can be incorporated under the augmented-state
conditions in Appendix~\ref{app:resolution-synthesis}. Discrete changes
of architecture or program length require a separate model.

The contribution boundary is thus precise. Smooth Euler bias, chronological
propagation, and converting an error bound into a sign test are standard
ingredients. The main independent argument derives the selected integers
and closes their uniform endpoint expansion together, with a phase-dependent
mesh threshold. The affine theorem supplies an explicit verified
remainder. The three-radius statement and scalar phase moments are
deductions used to interpret those results, not extra novelty claims.

\section{Additional response and outer-decision checks}
\label{app:rebuttal-checks}
We isolate the Euler bias, evaluate candidate choices after reweighting,
and describe the rounding correction under a phase law.

\subsection{Isolating smooth Euler bias}
The scalar E1 task has $a=4,c=-3,T=1$,
$t=\sqrt2/4$, $x_0=4(1-e^t)$ and $\epsilon=.3/N$.
All original $N=200,\ldots,10000$ remain in the comparison.
For displacement $z\in\{-.3,0,.3\}$, the pre-event regional
propagation gives $a_1=e^{-t}z+2t$, and $u_1=4$.
Theorem~\ref{thm:coupled-response}'s crossing prediction is therefore
\[
 \widehat K(z)=\left\lceil Nt-z/(4-x_0)-t/2\right\rceil.
\]
The ablation deletes only $-t/2$, the smooth Euler bias term.
Exact finite crossings are evaluated independently from the logarithmic
formula in Section~\ref{sec:scales}, and endpoints from its exact
regional recurrence. No phase-boundary cases are removed.

The full formula predicts all three indices correctly on 9,801/9,801
meshes; the ablation does so on 7,035/9,801. The signed negative-minus-positive
index shift is correct on 9,801 versus 7,736 meshes.
Using the corresponding first-order scalar response, mean absolute
response errors are $1.06951\times10^{-4}$ and $1.84317\times10^{-1}$;
median errors are $5.25180\times10^{-5}$ and $6.99268\times10^{-5}$.
The mean/median contrast exposes the ablation's missed branch jumps.
The independently recomputed finite responses match the frozen E1
values in this environment and retain its 1,034 reversals. Exact index
agreement here describes this finite population; the theorem still
requires its geometric and phase-margin conditions.

\subsection{A fixed outer-weight candidate task}
The inner model, initialization seeds, dataset, training/validation
split, group definitions and loss are the original network audit's.
For each of 20 seeds at widths $h=4,8,16$, set the four training-group
weights to $\lambda_0=(1,1,1,1)$. The outer objective is the unweighted
validation half-MSE after full-batch float64 GD to $T=1$.
Use $N\in\{500,2000\}$ and $\zeta\in\{.3,3\}$,
$\epsilon=\zeta/N$. The nine candidates, fixed before computing their
outcomes, are the base point followed by
$\lambda_0+\epsilon e_0,\lambda_0-\epsilon e_0,\ldots,
\lambda_0+\epsilon e_3,\lambda_0-\epsilon e_3$.
Each method minimizes its predicted objective change, breaking exact
ties by this candidate order. Independently rerun candidate GD losses
define the best available candidate and selection regret.

The branch calculation applies the constant-parameter augmentation
in Appendix~\ref{app:resolution-synthesis}. For each candidate put
$z_\lambda=N(\lambda-\lambda_0)$. Along each fixed base region,
with $A=D_\theta f$ and $B=D_\lambda f$, propagate
\[
 \dot E=AE+Bz_\lambda-\tfrac12Af,\qquad E(0)=0.
\]
At each base event use $k_j=\lceil Nt_j-n_j^\top E/u_j\rceil$ and
$E\leftarrow E-\Delta_j(k_j-Nt_j)$, separately for every candidate.
The objective-change predictor is
$\nabla q(x_T)^\top[E(z_\lambda)-E(0)]/N$.
AD differentiates the actual finite-GD run; flow sensitivities use
the regional parameter variational equation and saltation updates.

The branch predictor is assisted by the original root-audited base
reference times. Regional DOP853 solves use relative/absolute tolerances
$2\times10^{-10}/2\times10^{-12}$. This does not repeat the original
root-discovery audit. The original uncovered cases
$(h,s)=(4,16),(16,4),(16,15)$ remain uncovered: 57/60 cases and 228/240
decisions have a common reference. Reference event counts range from
1 to 138. A decision is admitted only if every candidate's proposed
indices satisfy $0<k_1<\cdots<k_m<N$. The order screen uses predicted indices. Complete agreement uses the
rerun sequences.

\begin{table}[t]
\centering
\begin{tabular}{lrrrrr}
\toprule
Width & Planned & Reference & Gate admitted & Word matched & Terminal matched\\
\midrule
4 & 80 & 76 & 50 & 68 & 74\\
8 & 80 & 80 & 36 & 57 & 73\\
16 & 80 & 72 & 9 & 41 & 70\\
All & 240 & 228 & 95 & 166 & 217\\
\bottomrule
\end{tabular}
\caption{Coverage of outer-weight decisions. Word and terminal columns
count decisions where all nine candidates match the reference-based
predicted training words or terminal validation masks, respectively,
among all reference decisions. These are separate post-evaluation
diagnostics. Of 95 gate-admitted decisions, 85 match all training words
and 10 have a mismatch. All 95 match the terminal masks. These counts
measure empirical agreement after screening.}
\label{tab:rebuttal-coverage}
\end{table}

\begin{table}[t]
\centering
\begin{tabular}{lrrrr}
\toprule
Method & Change MAE & Best candidate & Mean regret & Maximum regret\\
\midrule
AD & $1.14431\!\times\!10^{-6}$ & 94/95 & $5.59981\!\times\!10^{-9}$ & $5.31982\!\times\!10^{-7}$\\
Flow & $6.29722\!\times\!10^{-7}$ & 94/95 & $5.59981\!\times\!10^{-9}$ & $5.31982\!\times\!10^{-7}$\\
Branch & $1.28586\!\times\!10^{-7}$ & 95/95 & 0 & 0\\
\bottomrule
\end{tabular}
\caption{Identical 95 gate-admitted decisions for all methods. Change
MAE averages absolute errors over the eight nonbase candidate objective
changes and then over decisions. Regret is actual selected validation
loss minus the minimum of all nine rerun losses. Medians of regret are
zero for every method. The nonlinear experiment evaluates no finite
remainder bound. Both the 85 word-matched and 10 word-mismatched
admitted decisions are empirical predictions. Theorem~\ref{thm:contractive-main}
additionally requires contractive affine dynamics, complete word
verification, and a strict gap over its remainder bound.}
\label{tab:rebuttal-decisions}
\end{table}

We also evaluate the unrestricted branch predictions on every
reference decision, including those outside the order screen. On all 228
reference decisions, change MAE is $3.81045\times10^{-6}$ (AD),
$1.68218\times10^{-6}$ (flow) and $3.81580\times10^{-7}$ (Branch).
Their unrestricted best-candidate counts are 225, 225 and 227;
the screened decision policy leaves the remaining 133 reference
decisions unresolved.

The one corrected selection in the admitted population is
$h=8$, seed 10, $N=2000$, $\zeta=.3$: AD and flow select $+e_1$,
whereas Branch selects $+e_3$. Their actual losses are respectively
$.250270946713142$ and $.250270414731544$. All nine training words
and terminal masks match for this decision. The loss difference is
$5.32\times10^{-7}$. We identified the example after evaluating all
decisions.

The candidate loss implementation is checked against the original
scalar implementation at three widths. The fixed pilot $h=4$, seed 0
is independently rerun candidate-by-candidate, with maximum loss
discrepancy $5.56\times10^{-17}$, and its outer AD gradient is checked
by centered differences at radius $10^{-7}$, with maximum component
error $7.53\times10^{-10}$. Across the complete new sweep, base Euler
endpoints agree with frozen endpoints within $2.23\times10^{-16}$,
and flow parameter tangents agree with frozen tangents to relative
error at most $3.81\times10^{-12}$. These are numerical implementation
checks, not proofs of nonlinear word validity.

The fixed protocol, 60-case ledger, all candidate losses/predictions,
and scripts accompany the response. Repeated seeds across meshes
and radii are dependent, so counts are descriptive. The study measures choices in shallow networks on a controlled
dataset. Most choices agree, and order-screen coverage falls with
width. Evaluating repeated updates and larger networks would require
measuring their loss changes and the cost of constructing references
in addition to the recorded propagation cost.

\subsection{A statistical interpretation of the rounding recurrence}
\label{app:statistical-interface}
Writing $c_j=p_j-n_j^\top a_j/u_j$ and
$r_j=\lceil c_j\rceil-c_j\in(0,1)$ away from boundaries, Eq.~\eqref{eq:main-coupled-recursion}
is exactly $b_j=\Xi_j a_j-\Delta_j r_j$.
This exposes bounded rounding terms transported through the classical
saltation factors. Their joint law depends on earlier rounding through
$a_j$; independent uniform phases are not assumed. The scalar uniform
mesh-phase moments in Eq.~\eqref{eq:main-phase-classification} suggest
a distributional question, not an average-over-seeds guarantee.
A paired-response identity makes the required dependence explicit.
For the initialization theorem with smooth terminal $q$, fix the
reference path, $\zeta>0$ and direction $v$, and restrict to phases
where the base and both candidates have positive recursive margins.
Let $\delta r_j=r_j(\zeta v)-r_j(-\zeta v)$, and define the downstream
transport and scalar jump weights by
\[
 \mathcal T_m=P_m,\qquad
 \mathcal T_j=\mathcal T_{j+1}\Xi_{j+1}P_j\quad(j=m-1,\ldots,1),
 \qquad \gamma_j=\nabla q(x_T)^\top\mathcal T_j\Delta_j.
\]
For the first-order branch response
$\widehat D_{p,\zeta}(v)=\nabla q(x_T)^\top[e_p(\zeta v)-e_p(-\zeta v)]/(2\zeta)$,
\[
 \widehat D_{p,\zeta}(v)
 =g_0^\top v-\frac{1}{2\zeta}\sum_{j=1}^m\gamma_j\delta r_j.
\]
Indeed, subtracting the two candidate recursions cancels every smooth
bias term and gives
$\delta b_j=\Xi_jP_{j-1}\delta b_{j-1}-\Delta_j\delta r_j$,
with $\delta b_0=2\zeta v$. Unrolling it gives the displayed identity;
the homogeneous product is the classical flow derivative.
For $m=0$, the empty rounding correction is zero.
Euler bias has not disappeared from the calculation: it still affects
which rounding values are selected for each candidate.

For any chosen probability law supported on these admissible phases of
this fixed reference geometry, the finite sums give exactly
\begin{align*}
 \mathbb E\widehat D_{p,\zeta}(v)
 &=g_0^\top v-\frac{1}{2\zeta}\sum_j\gamma_j\mathbb E\delta r_j,\\
 \operatorname{Var}(\widehat D_{p,\zeta}(v))
 &=\frac{1}{4\zeta^2}\sum_{i,j}\gamma_i\gamma_j
       \operatorname{Cov}(\delta r_i,\delta r_j).
\end{align*}
Thus mean cancellation requires the weighted rounding contrasts to have
zero mean; it does not follow from boundedness. Event and candidate
correlations remain in the covariance sum. This is an algebraic
consequence of the finite-event response law, not a new stochastic
limit, an assumed phase distribution for real training, or a guarantee
for random data. A growing-event or mean-field limit would additionally
need control of these weighted sums, geometric degeneration and the
original endpoint remainder as the event count grows.

\FloatBarrier
\section{Additional distributions and supporting figures}
\label{app:additional-figures}
\begin{figure}[H]
\centering\includegraphics[width=\textwidth]{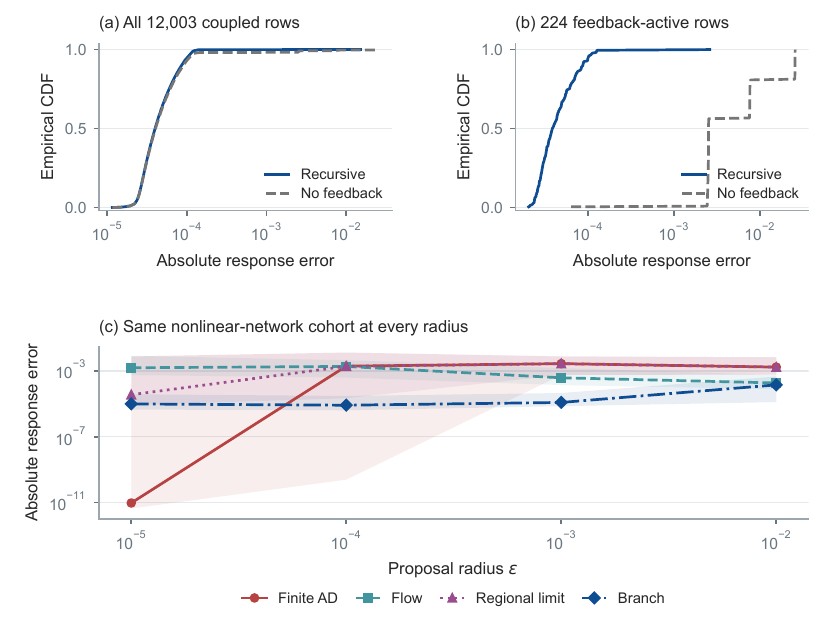}
\caption{\textbf{Full error distributions and the common network cohort.}
(a) The empirical error CDF includes all 12,003 coupled dense-grid
rows, including four failed words. (b) the 224-row subset is defined
by a difference in the predicted signed second-event crossing count,
without consulting endpoint objectives. It retains the one failed word
where the ablation is more accurate. (c) The same 19 reference seeds
and both meshes at every radius, now also showing the regional limit.
Bands are interquartile ranges across 38 seed--mesh observations per
radius. The smooth branch predictor is evaluated even when its proposed
word fails; it has no finite-grid network certificate. Mean, median,
tail errors and all failure counts are reported in the accompanying
tables rather than replacing the cohort with successful subsets.}
\label{fig:final-supplement}\label{fig:network-response}
\end{figure}

\begin{figure}[H]
 \centering
 \includegraphics[width=\textwidth]{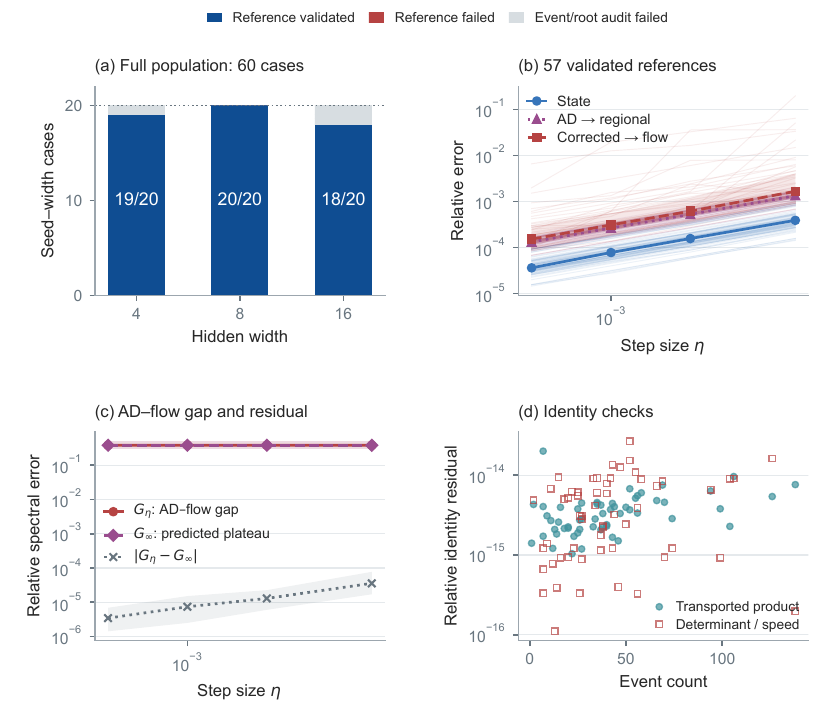}
 \caption{\textbf{Root-audited all-seed diagnostics.}
 (a) Exact counts retain all 20 prespecified seeds per width; no Wilson or other
 sampling interval is attached.  (b) Over the 57 reference-validated cases,
 state, AD-to-regional, and coarse-path corrected-to-flow errors decrease under mesh
 refinement; thin curves show individual cases and thick curves show medians
 with IQR bands.  (c) The relative finite-mesh gap \(G_\eta\) approaches the
 event-predicted plateau \(G_\infty\), while \(\Delta_G\) decreases.  (d) The
 transported-product and determinant/speed-ratio identities hold to numerical
 precision. Reference coverage is not full-Jacobian independent coverage.
 The 63/228 coarse itinerary mismatches remain in these pooled summaries;
 Appendix~\ref{app:derivative-coverage} separately gives the same 31 matched configurations
 across all meshes. This controlled grid is not a prevalence study.}
 \label{fig:width-geometry-full}
\end{figure}

\begin{figure}[H]
 \centering
 \includegraphics[width=\textwidth]{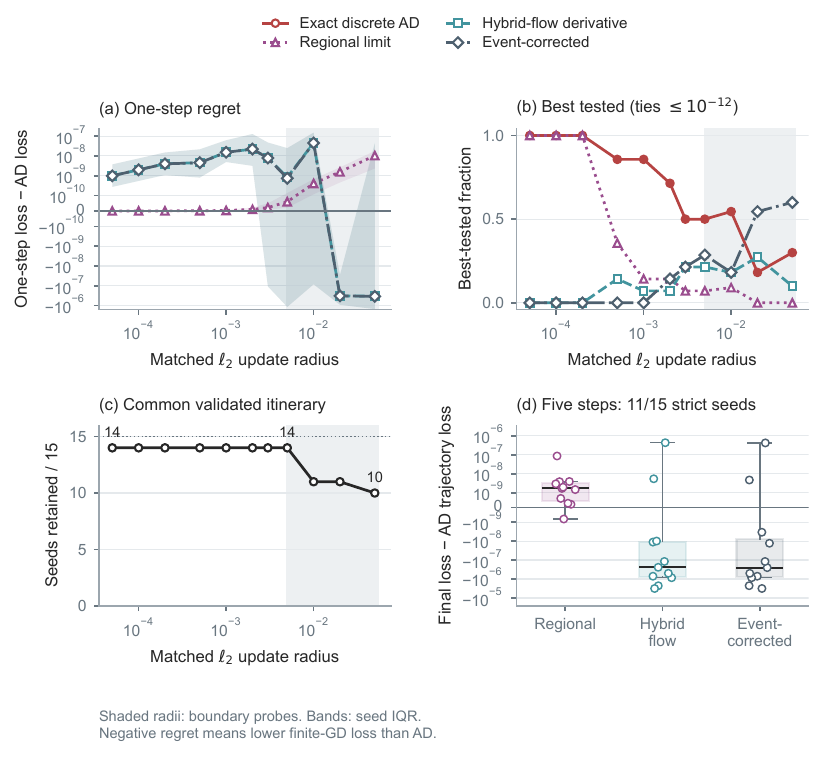}
 \caption{\textbf{All derivative proposals evaluated on the same finite-GD
 objective.}  Signed differences are measured against the exact discrete-AD
 proposal; negative values mean that a method happened to achieve a lower
 finite-radius loss.  The local exact-gradient advantage is measurable but
 small, and it does not imply a multi-step performance ordering.  Gray shading marks the boundary-probe radii; bands show the seed IQR.
Panel (c) reports the common-itinerary denominator, and panel (d) uses
the strict five-step subset. These are historical derivative proposals,
distinct from the finite-response Branch predictor.}
 \label{fig:single-objective-full}
\end{figure}

\FloatBarrier
\section{Consequences for finite proposals}
\label{app:finite-proposal-consequences}
The critical response theorem yields a direction-dependent training-cell
radius and a quantitative obstruction to representing all proposals by
one gradient. These consequences make the response object's use and
scope concrete; they are deductions from the established branch law.

\subsection{A radius from recursive phase margins}
\label{app:directional-radius}
At phase $p$, compute the base recursion at $z=0$. Write
$c_j^0=p_j-n_j^\top a_j(p,0)/u_j$,
$d_j=\operatorname{dist}(c_j^0,\mathbb Z)>0$,
$R_j=P_{j-1}\cdots P_0$, and
$\beta_j(v)=n_j^\top R_jv/u_j$. Define
\[
 Z_p(v)=\min_{j:\beta_j(v)\ne0}\frac{d_j}{|\beta_j(v)|},
 \qquad\min\varnothing=+\infty.
\]
The recursive margins $d_j$ include the base Euler bias and upstream
rounding. Raw distances of $t_j/\eta$ to integers do not suffice.

\begin{corollary}[A sufficient directional training-cell radius]
\label{cor:directional-radius}
Fix a unit direction $v$ and $0<\kappa<1$. If $Z_p(v)<\infty$, then
on $K=\{sv:|s|\le(1-\kappa)Z_p(v)\}$ every ceiling in the phase-$p$
recursion equals its base ceiling, and its recursive margin is at least
$\kappa\min_j d_j$. If $Z_p(v)=\infty$, the same assertion holds on
any fixed bounded interval, with unchanged margins.

Let $\rho_\eta^{\rm train}(v)$ be the largest symmetric initialization
radius preserving the complete training word. Under the critical
theorem's geometry, on a phase-convergent subsequence $p_N\to p$ with
the stated positive base margins,
\[
 \liminf\frac{\rho_\eta^{\rm train}(v)}{\eta}\ge Z_p(v).
\]
The statement also applies to the full training-and-validation word
radius if the terminal validation branch is fixed on a neighborhood
of the flow endpoint. Without that condition, terminal validity must
be checked separately.
\end{corollary}

\begin{proof}
Induct on the events. While the preceding ceilings are unchanged,
$a_j(p,sv)-a_j(p,0)=sR_jv$, because each preceding rounding offset
is constant. Thus its ceiling argument is $c_j^0-s\beta_j(v)$.
For $|s|\le(1-\kappa)Z_p(v)$ its variation has magnitude at most
$(1-\kappa)d_j$, so it stays in the same open integer cell with
margin at least $\kappa d_j$. This closes the induction. The same
strict inequalities persist for sufficiently close local lifts of $p_N$,
uniformly on this compact interval. Theorem~\ref{thm:coupled-response}
then gives the same actual discrete indices and training word for all
these candidates on sufficiently fine grids. Let $\kappa\downarrow0$.
When $Z_p(v)=\infty$, all $\beta_j(v)$ vanish, and the argument applies
to each bounded interval. Terminal branch stability follows from a
fixed endpoint neighborhood and the uniform endpoint expansion.
\end{proof}

This sufficient radius uses the limiting reference phase. To evaluate
the scale, transport the direction through regional variational
equations and compare it with the recursively corrected margins. A vanishing margin can
make this radius arbitrarily small relative to $\eta$.

\paragraph{A finite-grid check of the radius guidance.}
\phantomsection\label{app:finite-radius-check}
A finite-grid radius also needs a bound on the crossing-clock error.
We compare the leading-order and remainder-adjusted radii on all
9,801 E1 grids, $N=200,\ldots,10000$. The check uses the activation
steps and includes the boundary-near cases. For the scalar risk
$\mathcal L_{4,-3}$, set $T=1$, $t=\sqrt2/4$,
$x_0=4(1-e^t)$ and $B=4e^t$. Write
\[
 a_\eta=-\log(1-\eta),\quad
 c_N=t/\eta-t/2,\quad d_N=\operatorname{dist}(c_N,\mathbb Z),\quad
 \kappa_N=t/a_\eta,\quad \delta_N=|\kappa_N-c_N|.
\]
Here $c_N$ is the bias-corrected first-order clock and $\delta_N$
is its exact finite remainder, available in closed form in this scalar
model. The finite plug-in radius $.9\eta B d_N$ leaves the actual
training cell at $N=289$ and $9801$, retaining it on 9,799/9,801 grids.
These grid-dependent margins are not the fixed admissible phase limit
in Corollary~\ref{cor:directional-radius}.

A finite correction follows from a lower bound on the actual clock
margin. Let $\underline d_N=(d_N-\delta_N)_+$. If this is zero, retain an
unresolved outcome; otherwise choose
\[
 \epsilon_N^{\rm finite}
 =.9B\bigl(1-e^{-a_\eta\underline d_N}\bigr).
\]
Indeed, $\operatorname{dist}(\kappa_N,\mathbb Z)\ge\underline d_N$.
For any perturbation $|h|\le\epsilon<B$, the exact crossing clock is
$(t+\log(1-h/B))/a_\eta$, whose change has magnitude at most
$-\log(1-\epsilon/B)/a_\eta$.
At the displayed radius this is strictly less than $\underline d_N$,
so its ceiling, and hence the complete scalar training word, stays fixed
throughout the interval. The initial and endpoint crossing conditions
are also checked on the stated grids.

The protocol fixed the population and safety factor $.9$ before
computing the check. All 9,801 grids are admitted by the corrected rule,
and all preserve their training word. Exact cell radii and both
endpoint indices are retained for every grid. Every failed plug-in case
and the 40 smallest corrected margins are independently checked at
70-digit precision. The scalar check uses the exact clock to show
how the finite remainder affects the radius. A nonlinear extension
would require its own finite remainder bound.

\subsection{Why a single gradient has an error floor}
\label{app:linear-barrier}
The next example uses a strongly convex ReLU training risk, rather
than only a prescribed discontinuous vector field.

\begin{proposition}[A positive error floor for every linear predictor]
\label{prop:linear-barrier}
Let
\[
 L(\theta)=\tfrac12(\theta_1-4)^2
 +\tfrac12(\operatorname{ReLU}(\theta_1)+3)^2+\tfrac12\theta_2^2,
 \qquad q(y)=\tfrac12(y_1-1)^2,
\]
and initialize at $\theta_0=(4(1-e^{1/2}),0)$, with horizon $T=1$.
This risk is globally $1$-strongly convex. Put $B=4e^{1/2}$,
$\zeta=B/2$, $E=e^{-1}$, $J=e^{-3/2}$,
$q'=-(1+E)/2$, and
\[
 \mathcal V=\{e_1,e_2,(1/3,2\sqrt2/3)\}.
\]
On odd grids $\eta=1/N$, the finite central responses satisfy
\[
 \lim_{\substack{N\to\infty\\N\text{ odd}}}
 \inf_{h\in\mathbb R^2}\max_{v\in\mathcal V}
 |D_{\eta,\zeta\eta}(v)-h^\top v|
 =\frac{|q'cE/B|}{4+2\sqrt2}=:C_*>0,
 \qquad c=-3.
\]
Numerically, $C_*$ is approximately $0.0168$. The infimum allows a
different vector on every grid. The critical branch predictor, evaluated
separately on the three directions, has maximum error $O(\eta)$.
The base finite program is nonresonant on sufficiently fine odd grids,
so its ordinary AD derivative exists and is locally correct.
All these candidate points have a common positive recursive phase
margin of at least $1/12$ in the limiting recursion.
\end{proposition}

\begin{proof}
The two regional fields are
$(4-\theta_1,-\theta_2)$ and $(1-2\theta_1,-\theta_2)$.
The unique event occurs at $t=1/2$ and has normal speeds $4$ and $1$.
The loss is $1$-strongly convex because the first and third quadratic
terms have Hessian $I$, and the squared shifted ReLU term is convex.
The scalar grid phase from Section~\ref{sec:scales} satisfies
\[
 \left\{\frac{1/2}{-\log(1-1/N)}\right\}\longrightarrow\frac14
 \quad(N\text{ odd}),
\]
since $(1/2)/[-\log(1-1/N)]=N/2-1/4+O(1/N)$.
At the common radius $\zeta=B/2$, the negative-to-positive candidate
interval in direction $e_1$ contains one scaled boundary. The interval
in direction $(1/3,2\sqrt2/3)$ contains none, because its scalar radius
is $B/6$ and its limiting phase distance is $1/4>1/6$.
The $e_2$ displacement does not affect $q$.
The limiting responses are therefore
\[
 A=q'(J+cE/B),\qquad 0,\qquad C=q'J/3.
\]
The recursion gives margin $1/4$ for the base and the $e_1,e_2$
endpoints, and minimum $1/12$ for the mixed endpoints. The finite-set
version of Theorem~\ref{thm:coupled-response} consequently yields these
limits with $O(\eta)$ error on the odd subsequence.

For any $h$, if all three prediction errors are at most $\delta$, then
with $s=1/3$, $r=2\sqrt2/3$,
\[
 |C-sA|\le |C-sh_1-rh_2|+s|h_1-A|+r|h_2|
 \le(1+s+r)\delta.
\]
Thus $\delta\ge |q'cE/B|/(4+2\sqrt2)$.
This bound is sharp: because $C-sA<0$, take
$h_1=A-C_*$ and $h_2=-C_*$; each absolute error equals $C_*$.
The optimum of this finite-direction approximation problem changes by
at most the maximum change in its three target values. Its finite-grid
optimum therefore converges to $C_*$, while the critical law has
$O(\eta)$ error. The same lower-bound argument permits arbitrary
grid-dependent choices of $h$.
\end{proof}

The obstruction concerns finite proposals around a differentiable base
program, rather than absence of a local derivative. It also applies to a central
response computed from any smooth surrogate whose Hessian stays
uniformly bounded on the candidate neighborhood: its central response
differs from its gradient prediction by $O(\epsilon)$, which vanishes
on these grids. A quadratic surrogate gives exactly its gradient's
central prediction. Capturing the critical law requires the proposal's
branch information, or a surrogate whose variation resolves that scale.

The numerical check retains ten predetermined odd grids:
$N=101,\allowbreak201,\allowbreak501,\allowbreak1001,\allowbreak2001,
\allowbreak4001,\allowbreak8001,\allowbreak16001,\allowbreak32001,\allowbreak100001$.
Closed-form scalar recurrences are evaluated at 80-digit precision,
and selected endpoints are independently rerun by Euler loops.
At every mesh, a linear program checks the sharp finite-grid optimum
$|D(v)-D(e_1)/3|/(1+1/3+2\sqrt2/3)$, with $v=(1/3,2\sqrt2/3)$.
The accompanying figure reports maximum error across the same three
directions for AD, flow, the best unrestricted single vector, and the
critical response. This controlled example is not a prevalence claim.

\begin{figure}[H]
\centering
\includegraphics[width=\textwidth]{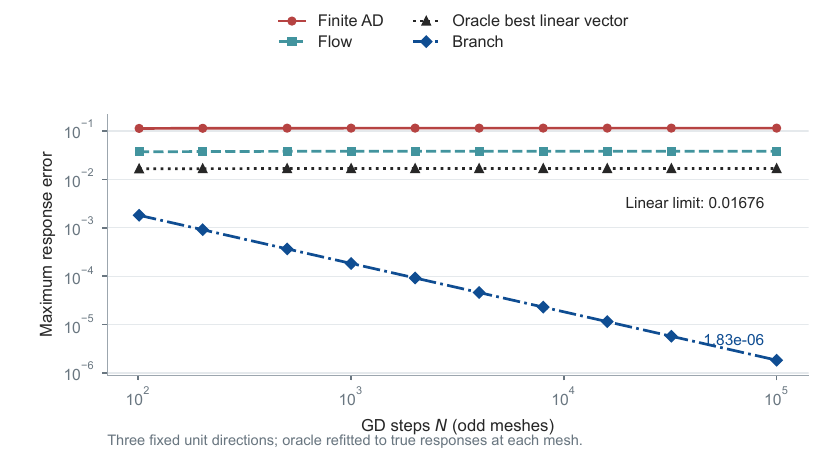}
\caption{Every single vector has a positive error floor on the three
fixed unit directions of Proposition~\ref{prop:linear-barrier}, while
the critical branch response converges. The loss is strongly convex;
one transverse event and a common positive phase margin suffice.
All predetermined odd meshes are retained. The dotted triangle curve is the
best linear predictor fitted separately at each mesh using the true
responses, an oracle comparison rather than an implementable method.
Maximum errors are measured in the same finite-response units.}
\label{fig:linear-barrier}
\end{figure}

\subsection{A zero-margin case that refinement alone does not repair}
\label{app:zero-margin}
In the same strongly convex risk, change only the critical radius to
$\zeta=B/4$ and keep odd $N$. The first-order ceiling argument for
the positive $e_1$ proposal is an integer:
$Nt-t/2-\zeta/B=(N-1)/2$.
Its zero recursive margin places it outside Theorem~\ref{thm:coupled-response}.
The exact logarithmic crossing clock instead expands as
\[
 \frac{\log((4-x_0-\eta\zeta)/4)}{-\log(1-\eta)}
 =\frac{N-1}{2}+\frac{5}{96}\eta+O(\eta^2).
\]
Consequently the actual positive proposal crosses one step later than
the uncertified first-order ceiling would indicate. Both actual
proposals retain the base index $(N+1)/2$ for sufficiently large odd $N$.
The true response converges to $q'J\simeq-.152608$,
whereas inserting the boundary value into the first-order ceiling
recurrence would give $q'(J+2cE/B)\simeq .076304$, the opposite sign.

The geometry remains separated and transverse. The error comes from
using a first-order integer selector at its own unresolved boundary.
This example does not contradict the theorem, which excludes that
use. It shows why a reliability rule must retain a phase condition,
a higher-order boundary analysis, or an unresolved outcome; smaller
steps alone do not justify a zero-margin selector.

\FloatBarrier
\section{Paired finite-grid resolution of the outer-weight task}
\label{app:paired-refinement}
The refinement protocol was fixed before computing the additional
candidate outcomes. It keeps the original 60 seed--width configurations,
dataset, nine candidates, two ratios $\zeta=.3,3$, and frozen base event
references. The first-round meshes $N=500,2000$ are preserved;
additional runs use $N=8000$ at all widths and $N=32000$ at width 16.
The same original three uncovered references remain uncovered.
All 57 reference cases complete the added runs, supplying 150 of
160 planned new decisions. No seed, radius, or mesh is selected after
observing its prediction error.

For each of the four weight-coordinate directions, compute the central
response $(Q_+-Q_-)/(2\epsilon)$, where $\epsilon=\zeta/N$.
AD and flow predictions are their coordinate derivatives; Branch uses
the difference of its two separate candidate predictions divided by
$2\epsilon$. The reported MAE averages four coordinate errors within
a decision, then averages over the same seed--ratio population at
each mesh. This normalization prevents smaller raw objective changes
from automatically appearing as better sensitivity predictions.

\begin{table}[H]
\centering
\begin{tabular}{rrrrrrrr}
\toprule
Width & $N$ & Reference & Order & Word & AD MAE & Flow MAE & Branch MAE\\
\midrule
4 & 500 & 38/40 & 18 & 31 & $2.49\!\times\!10^{-3}$ & $1.54\!\times\!10^{-3}$ & $1.35\!\times\!10^{-4}$\\
4 & 2000 & 38/40 & 32 & 37 & $2.32\!\times\!10^{-3}$ & $1.17\!\times\!10^{-3}$ & $6.21\!\times\!10^{-6}$\\
4 & 8000 & 38/40 & 34 & 38 & $2.71\!\times\!10^{-3}$ & $1.63\!\times\!10^{-3}$ & $1.34\!\times\!10^{-6}$\\
8 & 500 & 40/40 & 14 & 20 & $3.47\!\times\!10^{-3}$ & $1.70\!\times\!10^{-3}$ & $8.19\!\times\!10^{-5}$\\
8 & 2000 & 40/40 & 22 & 37 & $1.85\!\times\!10^{-3}$ & $1.49\!\times\!10^{-3}$ & $1.65\!\times\!10^{-5}$\\
8 & 8000 & 40/40 & 30 & 40 & $2.07\!\times\!10^{-3}$ & $1.47\!\times\!10^{-3}$ & $1.54\!\times\!10^{-6}$\\
16 & 500 & 36/40 & 0 & 16 & $7.26\!\times\!10^{-4}$ & $4.53\!\times\!10^{-4}$ & $3.25\!\times\!10^{-5}$\\
16 & 2000 & 36/40 & 9 & 25 & $6.71\!\times\!10^{-4}$ & $3.57\!\times\!10^{-4}$ & $6.55\!\times\!10^{-6}$\\
16 & 8000 & 36/40 & 25 & 31 & $6.84\!\times\!10^{-4}$ & $3.95\!\times\!10^{-4}$ & $1.40\!\times\!10^{-6}$\\
16 & 32000 & 36/40 & 31 & 35 & $6.65\!\times\!10^{-4}$ & $4.05\!\times\!10^{-4}$ & $3.06\!\times\!10^{-7}$\\
\bottomrule
\end{tabular}
\caption{Identical fixed seed--ratio populations at each mesh within
each width. Every mesh has 40 planned decisions. Order counts decisions
where all nine candidate index sequences are strictly increasing and
internal; Word counts decisions where all nine complete predicted
training words match the reruns. These are different tests. All
reference decisions, including order or word failures, enter each
method's radius-normalized response MAE. Counts are descriptive.}
\label{tab:paired-refinement}
\end{table}

At width 16, strict-order coverage increases from 0/36 at $N=500$
to 9/36, 25/36 and 31/36 at $N=2000,8000,32000$, respectively;
all-nine word agreement is 16/36, 25/36, 31/36 and 35/36.
The full planned denominator remains 40 at every mesh.
At the finest mesh, central-response MAE is
$6.65\times10^{-4}$ (AD), $4.05\times10^{-4}$ (flow), and
$3.06\times10^{-7}$ (Branch) on the same 36 decisions.
At width 4 and 8, all candidate words match on every reference
decision at $N=8000$; the strict-order gate is more conservative.

Agreement can change in either direction for a given decision. Between
$N=2000$ and 8000 at width 16, eight decisions change from mismatched
to matched words and two change in the reverse direction. Between
8000 and 32000, five recover and one loses agreement. Grid refinement
also changes event phases. Theorem~\ref{thm:coupled-response} uses
a common positive recursive margin along the grid sequence. The
evidence records all paired transitions, tails, and failed predictions.

Every new base Euler endpoint is independently rerun with the original
scalar core; its maximum discrepancy from the batched candidate code
is $2.23\times10^{-16}$. We reuse the numerical reference roots and
check the regional endpoints and parameter tangents. The study
measures resolution at fixed width and horizon as the proposal
radius shrinks.

\begin{figure}[H]
\centering
\includegraphics[width=\textwidth]{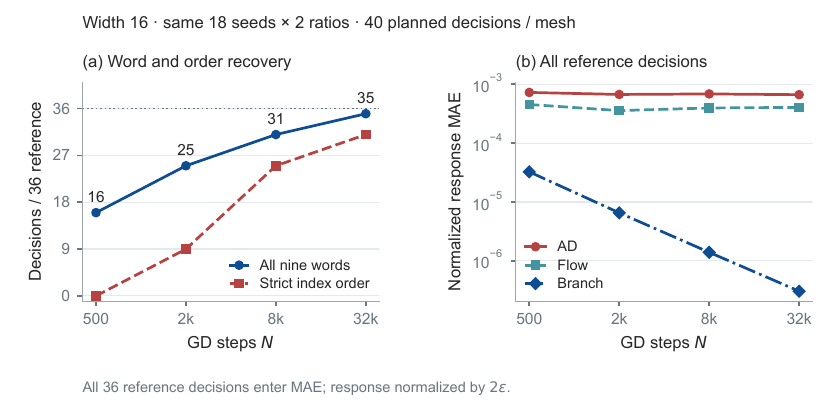}
\caption{Paired width-16 refinement on the same 18 reference-covered
seeds and two radius ratios, with four uncovered decisions retained
in the 40-decision planned denominator. Left: strict predicted order
and complete-word agreement recover differently. Right: all 36
reference decisions enter each method's central-response MAE,
normalized by $2\epsilon$; no failed words are removed. These are
correlated fixed-population summaries, without fitted rate or
sampling confidence intervals.}
\label{fig:paired-refinement}
\end{figure}

\FloatBarrier
\section{Spacing of activation events on the GD grid}
\label{app:event-crowding}
We use the same 57 validated paths from 60 planned seed--width cases
(Appendix~\ref{app:derivative-coverage}). For each $N$, we count
paths with two continuous events in a bin $[kT/N,(k+1)T/N)$ and paths
with an adjacent event gap below $T/N$. We also replay all 60 base
GD runs per mesh and count runs with multiple sample--unit mask
changes between any two consecutive executed masks. The terminal
mask is outside the executed sequence. For the 57 validated cases,
replayed endpoints and total flips, including terminal flips in the
total, agree with the stored GD results.

\begin{table}[H]
\centering
\caption{\textbf{Continuous event spacing and observed changes in GD masks.}
Entries count paths with at least one occurrence. The first two rows
per width use the validated references. The third uses all 20 GD runs
per width, including cases with no validated reference. Median reference
event counts are 23, 36, and 53.5 at widths 4, 8, and 16.}
\label{tab:event-crowding}
\begin{tabular}{@{}rrlrrrr@{}}
\toprule
Width & Paths & Criterion & $N=200$ & $N=500$ & $N=1000$ & $N=2000$ \\
\midrule
4 & 19 & Shared bin & 11 & 5 & 3 & 2 \\
  &    & Gap below one step & 12 & 11 & 5 & 3 \\
  & 20 & Multiple GD mask flips & 12 & 8 & 4 & 3 \\
8 & 20 & Shared bin & 18 & 10 & 9 & 7 \\
  &    & Gap below one step & 18 & 14 & 11 & 9 \\
  & 20 & Multiple GD mask flips & 17 & 11 & 9 & 7 \\
16 & 18 & Shared bin & 17 & 18 & 14 & 10 \\
   &    & Gap below one step & 18 & 18 & 17 & 14 \\
   & 20 & Multiple GD mask flips & 20 & 19 & 17 & 12 \\
\bottomrule
\end{tabular}
\end{table}

The continuous bin count at width 16 rises from 17 at $N=200$ to
18 at $N=500$, whose bins are on nonnested grids. On the same 18
validated width-16 cases, the replay counts 17 runs with multiple flips
at $N=500$ and 10 at $N=2000$. The corresponding continuous bin
counts are 18 and 10.

\subsection{Order screening and observed mask changes}
\label{app:crowding-coverage}
We join the 228 reference weight decisions to the replay by width,
seed, and mesh. Each of the 114 base runs contributes two decisions.
The screen tests whether all nine sets of predicted indices are
strictly increasing and interior. Complete agreement compares all nine
predicted mask sequences with the reruns.

\begin{table}[H]
\centering
\begin{tabular}{@{}lrrrrrr@{}}
\toprule
 & \multicolumn{3}{c}{With same-step base changes}
 & \multicolumn{3}{c}{Without same-step base changes}\\
\cmidrule(lr){2-4}\cmidrule(l){5-7}
Width & Decisions & Screen & Match & Decisions & Screen & Match\\
\midrule
4 & 18 & 0 & 14 & 58 & 50 & 54\\
8 & 36 & 0 & 21 & 44 & 36 & 36\\
16 & 54 & 0 & 28 & 18 & 9 & 13\\
All & 108 & 0 & 63 & 120 & 95 & 103\\
\bottomrule
\end{tabular}
\caption{Order screen passes and complete sequence matches for the
weight decisions, grouped by observed base mask changes. Both radii
from each base run enter the same group.}
\label{tab:crowding-coverage}
\end{table}

The screen excludes all 108 decisions with a same-step base change.
The unrestricted predictor matches all nine sequences in 63 of them.

\end{document}